%% file: main.tex
\documentclass{article}

\usepackage[numbers]{natbib}

     \usepackage[final]{neurips_2019}

\usepackage[utf8]{inputenc} 
\usepackage[T1]{fontenc}    
\usepackage{hyperref}       
\usepackage{url}            
\usepackage{booktabs}       
\usepackage{amsfonts}       
\usepackage{nicefrac}       
\usepackage{microtype}      

\usepackage{amsmath,amssymb,amsthm,graphics}
\usepackage[ruled,vlined]{algorithm2e}
\usepackage[utf8]{inputenc}
\newtheorem{Def}{Definition}[section]
\newtheorem{Thm}{Theorem}[section]
\newtheorem{Lem}{Lemma}[section]
\newtheorem{Cor}{Corollary}[section]
\newtheorem{Clm}{Claim}[section]
\newtheorem{Prp}{Proposition}[section]
\newtheorem{Asu}{Assumption}[section]
\newtheorem{Rem}{Remark}[section]

\DeclareMathOperator*{\argmin}{arg\,min}

\DeclareMathOperator*{\E}{\mathbb{E}}
\newcommand{\R}{\operatorname{\bf R}}
\newcommand{\U}{\operatorname{\bf U}}
\newcommand{\Q}{\mathcal Q}
\newcommand{\TAR}{\operatorname{\bf\bar R}}
\newcommand{\RUB}{\operatorname{\mathbf{\bar U}}}
\newcommand{\TASK}{\operatorname{TASK}}

\newcommand{\FTRL}{\operatorname{FTRL}}
\newcommand{\OMD}{\operatorname{OMD}}

\newcommand{\FTL}{\operatorname{FTL}}
\newcommand{\Conv}{\operatorname{Conv}}
\newcommand{\Proj}{\operatorname{Proj}}

\newcommand{\Breg}{\mathcal B}
\newcommand{\Diag}{\operatorname{Diag}}
\DeclareMathOperator{\Tr}{Tr}

\newcommand{\DYN}{\operatorname{INIT}}
\newcommand{\SIM}{\operatorname{SIM}}
\newcommand{\finit}{f^\textrm{init}}
\newcommand{\fsim}{f^\textrm{sim}}
\newcommand{\Rinit}{\U^\textrm{init}}
\newcommand{\Rsim}{\U^\textrm{sim}}
\usepackage{scalerel,mathtools}
\let\svsqrt\sqrt
\newsavebox\Nsqrt
\def\sr#1{\ThisStyle{%
		\savebox\Nsqrt{\scalebox{.5}[1]{$\SavedStyle\svsqrt{\phantom{\cramped{#1#1}}}$}}%
		\ooalign{\usebox{\Nsqrt}\cr\kern.2pt\usebox{\Nsqrt}\cr\hfil$\SavedStyle\cramped{#1}$}}}

\newcommand{\sd}{\scalebox{0.64}[1]{$\dots$}}

\newcommand{\Ephemeral}{}
\ifdefined\Ephemeral
\newcommand{\Eph}{Ephemeral\xspace}
\else
\newcommand{\Eph}{FMRL\xspace}
\fi

\usepackage{enumitem}
\usepackage{bm}
\def\*#1{\bm{#1}}
\usepackage{float}

\title{
	Adaptive Gradient-Based Meta-Learning
	Methods
}

\author{%
  Mikhail Khodak \\
  Carnegie Mellon University\\
  \texttt{khodak@cmu.edu} \\
  \And
  Maria-Florina Balcan \\
  Carnegie Mellon University\\
  \texttt{ninamf@cs.cmu.edu} \\
  \And
  Ameet Talwalkar \\
  Carnegie Mellon University\\
  \& Determined AI\\
  \texttt{talwalkar@cmu.edu} \\
}

\begin{document}

\maketitle

\begin{abstract}
	
	We build a theoretical framework for designing and understanding practical meta-learning methods that integrates sophisticated formalizations of task-similarity with the extensive literature on online convex optimization and sequential prediction algorithms.
	Our approach enables the task-similarity to be learned adaptively, provides sharper transfer-risk bounds in the setting of statistical learning-to-learn, and leads to straightforward derivations of average-case regret bounds for efficient algorithms in settings where the task-environment changes dynamically or the tasks share a certain geometric structure.
	We use our theory to modify several popular meta-learning algorithms and improve their meta-test-time performance on standard problems in few-shot learning and federated learning.
	
\end{abstract}

\input{intro.tex}

\input{related.tex}

\input{aruba.tex}

\input{adaptive.tex}

\input{statistical.tex}

\input{empirical.tex}

\bibliography{references}
\bibliographystyle{plainnat}

\newpage
\appendix

\input{background.tex}

\input{coupling.tex}

\input{dynamic.tex}

\input{geometry.tex}

\input{batch.tex}

\input{similarity.tex}

\input{details.tex}

\end{document}

%% file: intro.tex

\section{Introduction}\label{sec:intro}

{\em Meta-learning}, or {\em learning-to-learn} (LTL) \citep{thrun:98}, has recently re-emerged as an important direction for developing algorithms for multi-task learning, dynamic environments, and federated settings.
By using the data of numerous training tasks, meta-learning methods seek to perform well on new, potentially related test tasks without using many samples.
Successful modern approaches have also focused on exploiting the capabilities of deep neural networks, whether by learning multi-task embeddings passed to simple classifiers \citep{snell:17} or by neural control of optimization algorithms \citep{ravi:17}.

Because of its simplicity and flexibility, a common approach is {\em parameter-transfer}, where all tasks use the same class of $\Theta$-parameterized functions $f_\theta:\mathcal X\mapsto\mathcal Y$;
often a shared model $\phi\in\Theta$ is learned that is used to train within-task models.
In {\em gradient-based meta-learning} (GBML) \citep{finn:17}, $\phi$ is a meta-initialization for a gradient descent method over samples from a new task.
GBML is used in a variety of LTL domains such as vision \citep{li:17,nichol:18,kim:18}, federated learning \citep{chen:18}, and robotics \citep{duan:17,al-shedivat:18}. 
Its simplicity also raises many practical and theoretical questions about the task-relations it can exploit and the settings in which it can succeed.
Addressing these issues has naturally led several authors to online convex optimization (OCO) \citep{zinkevich:03}, either directly \citep{finn:19,khodak:19} or from online-to-batch conversion \citep{khodak:19,denevi:19}.
These efforts study how to find a meta-initialization, either by proving algorithmic learnability \citep{finn:19} or giving meta-test-time performance guarantees \citep{khodak:19,denevi:19}.

However, this recent line of work has so far considered a very restricted, if natural, notion of task-similarity -- closeness to a single fixed point in the parameter space.
We introduce a new theoretical framework, {\bf Average Regret-Upper-Bound Analysis (ARUBA)}, that enables the derivation of meta-learning algorithms that can provably take advantage of much more sophisticated structure.
ARUBA treats meta-learning as the online learning of a sequence of losses that each upper bounds the regret on a single task.
These bounds often have convenient functional forms that are (a) sufficiently nice, so that we can draw upon the existing OCO literature, and (b) strongly dependent on both the task-data and the meta-initialization, thus encoding task-similarity in a mathematically accessible way. 
Using ARUBA we introduce or dramatically improve upon GBML results in the following settings:
\begin{itemize}[leftmargin=*]
	\item{\bf Adapting to the Task-Similarity:}
	A major drawback of previous work is a reliance on knowing the task-similarity beforehand to set the learning rate \citep{finn:19} or regularization \citep{denevi:19}, or the use of a sub-optimal guess-and-tune approach using the doubling trick \citep{khodak:19}.
	ARUBA yields a simple gradient-based algorithm that eliminates the need to guess the similarity by learning it on-the-fly.\vspace{-1mm}
	\item{\bf Adapting to Dynamic Environments:}
	While previous theoretical work has largely considered a fixed initialization \citep{finn:19,khodak:19}, in many practical applications of GBML the optimal initialization varies over time due to a changing environment \citep{al-shedivat:18}.
	We show how ARUBA reduces the problem of meta-learning in dynamic environments to a dynamic regret-minimization problem, for which there exists a vast array of online algorithms with provable guarantees that can be directly applied.\vspace{-1mm}
	\item{\bf Adapting to the Inter-Task Geometry:}
	A recurring notion in LTL is that certain model weights, such as feature extractors, are shared, whereas others, such as classification layers, vary between tasks.
	By only learning a fixed initialization we must re-learn this structure on every task.
	Using ARUBA we provide a method that adapts to this structure and determines which directions in $\Theta$ need to be updated by learning a Mahalanobis-norm regularizer for online mirror descent (OMD). 
	We show how a variant of this can be used to meta-learn a per-coordinate learning-rate for certain GBML methods, such as MAML \citep{finn:17} and Reptile \citep{nichol:18}, as well as for FedAvg, a popular federated learning algorithm \citep{mcmahan:17}.
	This leads to improved meta-test-time performance on few-shot learning and a simple, tuning-free approach to effectively add user-personalization to FedAvg.\vspace{-1mm}
	\item{\bf Statistical Learning-to-Learn:}
	ARUBA allows us to leverage powerful results in online-to-batch conversion \citep{zhang:05,kakade:08b} to derive new bounds on the transfer risk when using GBML for statistical LTL \citep{baxter:00}, including fast rates in the number of tasks when the task-similarity is known and high-probability guarantees for a class of losses that includes linear regression.
	This improves upon the guarantees of \citet{khodak:19} and \citet{denevi:19} for similar or identical GBML methods.\vspace{-1mm}
\end{itemize}

%% file: related.tex

\subsection{Related Work}\label{subsec:related}
{\bf Theoretical LTL:}
The statistical analysis of LTL was formalized by \citet{baxter:00}.
Several works have built upon this theory for modern LTL, such as via a PAC-Bayesian perspective \citep{amit:18} or by learning the kernel for the ridge regression \citep{denevi:18a}.
However, much effort has also been devoted to the online setting, often through the framework of lifelong learning \citep{pentina:14,balcan:15,alquier:17}.
\citet{alquier:17} consider a many-task notion of regret similar to the one we study in order to learn a shared data representation, although our algorithms are much more practical.
Recently, \citet{bullins:19} developed an efficient online approach to learning a linear data embedding, but such a setting is distinct from GBML and more closely related to popular shared-representation methods such as ProtoNets \citep{snell:17}.
Nevertheless, our approach does strongly rely on online learning through the study of data-dependent regret-upper-bounds, which has a long history of use in deriving adaptive single-task methods \citep{mcmahan:10,duchi:11};
however, in meta-learning there is typically not enough data to adapt to without considering multi-task data.
Analyzing regret-upper-bounds was done implicitly by \citet{khodak:19}, but their approach is largely restricted to using Follow-the-Leader (FTL) as the meta-algorithm.
Similarly, \citet{finn:19} use FTL to show learnability of the MAML meta-initialization.
In contrast, the ARUBA framework can handle general classes of meta-algorithms, which leads not only to new and improved results in static, dynamic, and statistical settings but also to significantly more practical LTL methods.

{\bf GBML:}
GBML stems from the Model-Agnostic Meta-Learning (MAML) algorithm \citep{finn:17} and has been widely used in practice \citep{al-shedivat:18,nichol:18,jerfel:18}.
An expressivity result was shown for MAML by \citet{finn:18}, proving that the meta-learner can approximate any permutation-invariant learner given enough data and a specific neural architecture.
Under strong-convexity and smoothness assumptions and using a fixed learning rate, \citet{finn:19} show that the MAML meta-initialization is learnable, albeit via an impractical FTL method.
In contrast to these efforts, \citet{khodak:19} and \citet{denevi:19} focus on providing finite-sample meta-test-time performance guarantees in the convex setting, the former for the SGD-based Reptile algorithm of \citet{nichol:18} and the latter for a regularized variant.
Our work improves upon these analyses by considering the case when the learning rate, a proxy for the task-similarity, is not known beforehand as in \citet{finn:19} and \citet{denevi:19} but must be learned online;
\citet{khodak:19} do consider an unknown task-similarity but use a doubling-trick-based approach that considers the absolute deviation of the task-parameters from the meta-initialization and is thus average-case suboptimal and sensitive to outliers.
Furthermore, ARUBA can handle more sophisticated and dynamic notions of task-similarity and in certain settings can provide better statistical guarantees than those of \citet{khodak:19} and \citet{denevi:19}.

%% file: aruba.tex

\vspace{-2mm}
\section{Average Regret-Upper-Bound Analysis}\label{sec:aruba}
\vspace{-1.5mm}
Our main contribution is ARUBA, a framework for analyzing the learning of $\mathcal X$-parameterized learning algorithms via reduction to the online learning of a sequence of functions $\U_t:\mathcal X\mapsto\mathbb R$ upper-bounding their regret on task $t$.
We consider a meta-learner facing a sequence of online learning tasks $t=1,\dots,T$, each with $m_t$ loss functions $\ell_{t,i}:\Theta\mapsto\mathbb R$ over action-space $\Theta\subset\mathbb R^d$.
The learner has access to a set of learning algorithms parameterized by $x\in\mathcal X$ that can be used to determine the action $\theta_{t,i}\in\Theta$ on each round $i\in[m_t]$ of task $t$.
Thus on each task $t$ the meta-learner chooses $x_t\in\mathcal X$, runs the corresponding algorithm, and suffers regret $\R_t(x_t)=\sum_{i=1}^{m_t}\ell_{t,i}(\theta_{t,i})-\min_\theta\sum_{i=1}^{m_t}\ell_{t,i}(\theta)$.
We propose to analyze the meta-learner's performance by studying the online learning of a sequence of regret-upper-bounds $\U_t(x_t)\ge\R_t(x_t)$, specifically by bounding the {\bf average regret-upper-bound} $\RUB_T=\frac1T\sum_{t=1}^T\U_t(x_t)$.
The following two observations highlight why we care about this quantity:\vspace{-1mm}
\begin{enumerate}[leftmargin=*]
	\item{\bf Generality:}
	Many algorithms of interest in meta-learning have regret guarantees $\U_t(x)$ with nice, e.g. smooth and convex, functional forms that depend strongly on both their parameterizations $x\in\mathcal X$ and the task-data.
	This data-dependence lets us adaptively set the parameterization $x_t\in\mathcal X$.
	\item{\bf Consequences:}
	By definition of $\U_t$ we have that $\RUB_T$ bounds the {\bf task-averaged regret (TAR)} $\TAR_T=\frac1T\sum_{t=1}^T\R_t(x_t)$ \citep{khodak:19}.
	Thus if the average regret-upper-bound is small then the meta-learner will perform well on-average across tasks.
	In Section~\ref{sec:statistical} we further show that a low average regret-upper-bound will also lead to strong statistical guarantees in the batch setting.\vspace{-1mm}
\end{enumerate}
ARUBA's applicability depends only on finding a low-regret algorithm over the functions $\U_t$;
then by observation 2 we get a task-averaged regret bound where the first term vanishes as $T\to\infty$ while by observation 1 the second term can be made small due to the data-dependent task-similarity:\vspace{-1mm}
\begin{equation*}
\TAR_T\le\RUB_T\le o_T(1)+\min_x\frac1T\sum_{t=1}^T\U_t(x)\vspace{-1mm}
\end{equation*}
\paragraph{The Case of Online Gradient Descent:}
Suppose the meta-learner uses online gradient descent (OGD) as the within-task learning algorithm, as is done by Reptile \cite{nichol:18}.
OGD can be parameterized by an initialization $\phi\in\Theta$ and a learning rate $\eta>0$, so that $\mathcal X=\{(\phi,\eta):\phi\in\Theta,\eta>0\}$.
Using the notation $v_{a:b}=\sum_{i=a}^bv_i$ and $\nabla_{t,j}=\nabla\ell_{t,j}(\theta_{t,j})$, at each round $i$ of task $t$ OGD plays $\theta_{t,i}=\argmin_{\theta\in\Theta}\frac12\|\theta-\phi\|_2^2+\eta\langle\nabla_{t,1:i-1},\theta\rangle$.
The regret of this procedure when run on $m$ convex $G$-Lipschitz losses has a well-known upper-bound \cite[Theorem~2.11]{shalev-shwartz:11}\vspace{-1mm}
\begin{equation}\label{eq:omdregret}
\U_t(x)
=\U_t(\phi,\eta)
=\frac1{2\eta}\|\theta_t^\ast-\phi\|_2^2+\eta G^2m
\ge\sum_{i=1}^m\ell_{t,i}(\theta_t)-\ell_{t,i}(\theta_t^\ast)
=\R_t(x)\vspace{-1mm}
\end{equation}
which is convex in the learning rate $\eta$ and the initialization $\phi$.
Note the strong data dependence via $\theta_t^\ast\in\argmin_\theta\sum_{i=1}^{m_t}\ell_{t,i}(\theta)$, the optimal action in hindsight.
To apply ARUBA, first note that if $\bar\theta^\ast=\frac1T\theta_{1:T}^\ast$ is the mean of the optimal actions $\theta_t^\ast$ on each task and $V^2=\frac1T\sum_{t=1}^T\|\theta_t^\ast-\bar\theta^\ast\|_2^2$ is their empirical variance, then $\min_{\phi,\eta}\frac1T\sum_{t=1}^T\U_t(\phi,\eta)=\mathcal O(GV\sqrt m)$.
Thus by running a low-regret algorithm on the regret-upper-bounds $\U_t$ the meta-learner will suffer task-averaged regret at most $o_T(1)+\mathcal O(GV\sqrt m)$, which can be much better than the single-task regret $\mathcal O(GD\sqrt m)$, where $D$ is the $\ell_2$-radius of $\Theta$, if $V\ll D$, i.e. if the optimal actions $\theta_t^\ast$ are close together.
See Theorem~\ref{thm:similarity} for the result yielded by ARUBA in this simple setting.

\vspace{-2mm}
\section{Adapting to Similar Tasks and Dynamic Environments}\label{sec:simdyn}
\vspace{-1.5mm}

We now demonstrate the effectiveness of ARUBA for analyzing GBML by using it to prove a general bound for a class of algorithms that can adapt to both {\em task-similarity}, i.e. when the optimal actions $\theta_t^\ast$ for each task are close to some good initialization, and to {\em changing environments}, i.e. when this initialization changes over time.
The task-similarity will be measured using the {\bf Bregman divergence} $\Breg_R(\theta||\phi)=R(\theta)-R(\phi)-\langle\nabla R(\phi),\theta-\phi\rangle$ of a 1-strongly-convex function $R:\Theta\mapsto\mathbb R$ \citep{bregman:67}, a generalized notion of distance.
Note that for $R(\cdot)=\frac12\|\cdot\|_2^2$ we have $\Breg_R(\theta||\phi)=\frac12\|\theta-\phi\|_2^2$.
A changing environment will be studied by analyzing {\bf dynamic regret}, which for a sequence of actions $\{\phi_t\}_t\subset\Theta$ taken by some online algorithm over a sequence of loss functions $\{f_t:\Theta\mapsto\mathbb R\}_t$ is defined w.r.t. a reference sequence $\Psi=\{\psi_t\}_t\subset\Theta$ as $\R_T(\Psi)=\sum_{t=1}^Tf_t(\phi_t)-f_t(\psi_t)$.
Dynamic regret measures the performance of an online algorithm taking actions $\phi_t$ relative to a potentially time-varying comparator taking actions $\psi_t$.
Note that when we fix $\psi_t=\psi^\ast\in\argmin_{\psi\in\Theta}\sum_{t=1}^Tf_t(\psi)$ we recover the standard {\bf static regret}, in which the comparator always uses the same action.

Putting these together, we seek to define variants of Algorithm~\ref{alg:general} for which as $T\to\infty$ the average regret scales with $V_\Psi$, where $V_\Psi^2=\frac1T\sum_{t=1}^T\Breg_R(\theta_t^\ast||\psi_t)$, without knowing this quantity in advance.
Note for fixed $\psi_t=\bar\theta^\ast=\frac1T\theta_{1:T}^\ast$ this measures the empirical standard deviation of the optimal task-actions $\theta_t^\ast$.
Thus achieving our goal implies that average performance improves with task-similarity.

\begin{algorithm}[!t]
	\DontPrintSemicolon
	Set meta-initialization $\phi_1\in\Theta$ and learning rate $\eta_1>0$.\\
	\For{{\em task} $t\in[T]$}{
		\For{{\em round} $i\in[m_t]$}{
			$\theta_{t,i}\gets\argmin_{\theta\in\Theta}\Breg_R(\theta||\phi_t)+\eta_t\langle\nabla_{t,1:i-1},\theta\rangle$\tcp*{online mirror descent step}
			Suffer loss $\ell_{t,i}(\theta_{t,i})$
		}
		Update $\phi_{t+1},\eta_{t+1}$ \tcp*{meta-update of OMD initialization and learning rate}
	\vspace{-2mm}
	}
	\caption{\label{alg:general}
		Generic online algorithm for gradient-based parameter-transfer meta-learning.
		To run OGD within-task set $R(\cdot)=\frac12\|\cdot\|_2^2$.
		To run FTRL within-task substitute $\ell_{t,j}(\theta)$ for $\langle\nabla_{t,j},\theta\rangle$.
	}
\end{algorithm}

On each task $t$ Algorithm~\ref{alg:general} runs online mirror descent with regularizer $\frac1{\eta_t}\Breg_R(\cdot||\phi_t)$ for initialization $\phi_t\in\Theta$ and learning rate $\eta_t>0$.
It is well-known that OMD and the related Follow-the-Regularized-Leader (FTRL), for which our results also hold, generalize many important online methods, e.g. OGD and multiplicative weights \citep{hazan:15}.
For $m_t$ convex losses with mean squared Lipschitz constant $G_t^2$ they also share a convenient, data-dependent regret-upper-bound for any $\theta_t^\ast\in\Theta$ \citep[Theorem~2.15]{shalev-shwartz:11}:
\begin{equation}\label{eq:regret}
\R_t\le\U_t(\phi_t,\eta_t)=\frac1{\eta_t}\Breg_R(\theta_t^\ast||\phi_t)+\eta_tG_t^2m_t
\end{equation}
All that remains is to come up with update rules for the meta-initialization $\phi_t\in\Theta$ and the learning rate $\eta_t>0$ in Algorithm~\ref{alg:general} so that the average over $T$ of these upper-bounds $\U_t(\phi_t,\eta_t)$ is small.
While this can be viewed as a single online learning problem to determine actions $x_t=(\phi_t,\eta_t)\in\Theta\times(0,\infty)$, it is easier to decouple $\phi$ and $\eta$ by first defining  two
 function sequences $\{f_t^\textrm{init}\}_t$ and $\{\fsim_t\}_t$: 
 \begin{equation}\label{eq:auxregret}
f_t^\textrm{init}(\phi)=\Breg_R(\theta_t^\ast||\phi)G_t\sqrt{m_t}
\qquad\qquad\qquad
\fsim_t(v)=\left(\frac{\Breg_R(\theta_t^\ast||\phi_t)}v+v\right)G_t\sqrt{m_t}
\end{equation}
We show in Theorem~\ref{thm:general} that to get an adaptive algorithm it suffices to specify two OCO algorithms, $\DYN$ and $\SIM$, such that the actions $\phi_t=\DYN(t)$ achieve good (dynamic) regret over $f_t^\textrm{init}$ and the actions $v_t=\SIM(t)$ achieve low (static) regret over $\fsim_t$;
these actions then determine the update rules of $\phi_t$ and $\eta_t=v_t/(G_t\sqrt{m_t})$.
We will specialize Theorem~\ref{thm:general} to derive algorithms that provably adapt to task similarity (Theorem~\ref{thm:similarity}) and to dynamic environments (Theorem~\ref{thm:dynamic}).

To understand the formulation of $f_t^\textrm{init}$ and $\fsim_t$, first note that $\fsim_t(v)=\U_t(\phi_t,v/(G_t\sqrt{m_t}))$, so the online algorithm $\SIM$ over $\fsim_t$ corresponds to an online algorithm over the regret-upper-bounds $\U_t$ when the sequence of initializations $\phi_t$ is chosen adversarially.
Once we have shown that $\SIM$ is low-regret we can compare its losses $\fsim_t(v_t)$ to those of an arbitrary fixed $v>0$;
this is the first line in the proof of Theorem~\ref{thm:general} (below).
For fixed $v$, each $\finit_t(\phi_t)$ is an affine transformation of $\fsim_t(v)$, 
so the algorithm $\DYN$ with low dynamic regret over $\finit_t$ corresponds to an algorithm with low dynamic regret over the regret-upper-bounds $\U_t$ when $\eta_t=v/(G_t\sqrt{m_t})~\forall~t$.
Thus once we have shown a dynamic regret guarantee for $\DYN$ we can compare its losses $\finit_t(\phi_t)$ to those of an arbitrary comparator sequence $\{\psi_t\}_t\subset\Theta$;
this is the second line in the proof of Theorem~\ref{thm:general}.

\begin{Thm}\label{thm:general}
	Assume $\Theta\subset\mathbb R^d$ is convex, each task $t\in[T]$ is a sequence of $m_t$ convex losses $\ell_{t,i}:\Theta\mapsto\mathbb R$ with mean squared Lipschitz constant $G_t^2$, and $R:\Theta\mapsto\mathbb R$ is 1-strongly-convex.\vspace{-2mm}
	\begin{itemize}[leftmargin=*]
		\item Let $\DYN$ be an algorithm whose dynamic regret over functions $\{\finit_t\}_t$ w.r.t. any reference sequence $\Psi=\{\psi_t\}_{t=1}^T\subset\Theta$ is upper-bounded by $\Rinit_T(\Psi)$.\vspace{-2mm}
		\item Let $\SIM$ be an algorithm whose static regret over functions $\{\fsim_t\}_t$ w.r.t. any $v>0$ is upper-bounded by a non-increasing function $\Rsim_T(v)$ of $v$.\vspace{-2mm}
	\end{itemize}
	If Algorithm~\ref{alg:general} sets $\phi_t=\DYN(t)$ and $\eta_t=\frac{\SIM(t)}{G_t\sqrt{m_t}}$ then for $V_\Psi^2=\frac{\sum_{t=1}^T\Breg_R(\theta_t^\ast||\psi_t)G_t\sqrt{m_t}}{\sum_{t=1}^TG_t\sqrt{m_t}}$ it will achieve average regret\vspace{-1mm}
	$$\TAR_T
	\le\RUB_T
	\le\frac{\Rsim_T(V_\Psi)}T+\frac1T\min\left\{\frac{\Rinit_T(\Psi)}{V_\Psi},2\sqrt{\Rinit_T(\Psi)\sum_{t=1}^TG_t\sqrt{m_t}}\right\}+\frac{2V_\Psi}T\sum_{t=1}^TG_t\sqrt{m_t}$$
\end{Thm}

\begin{proof}
	For $\sigma_t=G_t\sqrt{m_t}$ we have by the regret bound on OMD/FTRL \eqref{eq:regret} that\vspace{-1mm}
\begin{align*}
	\RUB_TT
	=\sum_{t=1}^T\left(\frac{\Breg_R(\theta_t^\ast||\phi_t)}{v_t}+v_t\right)\sigma_t
	&\le\min_{v>0}\Rsim_T(v)+\sum_{t=1}^T\left(\frac{\Breg_R(\theta_t^\ast||\phi_t)}v+v\right)\sigma_t\\
	&\le\min_{v>0}\Rsim_T(v)+\frac{\Rinit_T(\Psi)}v+\sum_{t=1}^T\left(\frac{\Breg_R(\theta_t^\ast||\psi_t)}v+v\right)\sigma_t\\
	&\le\Rsim_T(V_\Psi)+\min\left\{\frac{\Rinit_T(\Psi)}{V_\Psi},2\sr{\Rinit_T(\Psi)\sigma_{1:T}}\right\}+2V_\Psi\sigma_{1:T}
\end{align*}
	where the last line follows by substituting $v=\max\left\{V_\Psi,\sqrt{\Rinit_T(\Psi)/\sigma_{1:T}}\right\}$.\vspace{-2mm}
\end{proof}

\vspace{-2mm}
\paragraph{Similar Tasks in Static Environments:}
By Theorem~\ref{thm:general}, if we can specify algorithms $\DYN$ and $\SIM$ with sublinear regret over $\finit_t$ and $\fsim_t$ \eqref{eq:auxregret}, respectively, then the average regret will converge to $\mathcal O(V_\Psi\sqrt m)$ as desired.
We first show an approach in the case when the optimal actions $\theta_t^\ast$ are close to a fixed point in $\Theta$, i.e. for fixed $\psi_t=\bar\theta^\ast=\frac1T\theta_{1:T}^\ast$.
Henceforth we assume the Lipschitz constant $G$ and number of rounds $m$ are the same across tasks;
detailed statements are in the supplement.

Note that if $R(\cdot)=\frac12\|\cdot\|_2^2$ then $\{\finit_t\}_t$ are quadratic functions, so playing $\phi_{t+1}=\frac1t\theta_{1:t}^\ast$ has logarithmic regret \citep[Corollary~2.2]{shalev-shwartz:11}.
We use a novel strongly convex coupling argument to show that this holds for any such sequence of Bregman divergences, {\em even for nonconvex} $\Breg_R(\theta_t^\ast||\cdot)$.
The second sequence $\{\fsim_t\}_t$ is harder because it is not smooth near 0 and not strongly convex if $\theta_t^\ast=\phi_t$.
We study a regularized sequence $\tilde f_t^\textrm{sim}(v)=\fsim_t(v)+\varepsilon^2/v$ for $\varepsilon\ge0$.
Assuming a bound of $D^2$ on the Bregman divergence and setting $\varepsilon=1/\sqrt[4]T$, we achieve $\tilde{\mathcal O}(\sqrt T)$ regret on the original sequence by running exponentially-weighted online-optimization (EWOO) \citep{hazan:07} on the regularized sequence:\vspace{-1mm}
\begin{equation}\label{eq:ewoo}
v_t=\frac{\int_0^{\sqrt{D^2+\varepsilon^2}}v\exp(-\gamma\sum_{s<t}\tilde f_s^\textrm{sim}(v))dv}{\int_0^{\sqrt{D^2+\varepsilon^2}}\exp(-\gamma\sum_{s<t}\tilde f_s^\textrm{sim}(v))dv}\qquad\textrm{for}\qquad\gamma=\frac2{DG\sqrt m}\min\left\{\frac{\varepsilon^2}{D^2},1\right\}
\end{equation}
Note that while EWOO is inefficient in high dimensions, we require only single-dimensional integrals.
In the supplement we also show that simply setting $v_{t+1}^2=\varepsilon^2t+\sum_{s\le t}\Breg_R(\theta_s^\ast||\phi_t)$ has only a slightly worse regret of $\tilde{\mathcal O}(T^{3/5})$.
These guarantees suffice to show the following:

\begin{Thm}\label{thm:similarity}
	Under the assumptions of Theorem~\ref{thm:general} and boundedness of $\Breg_R$ over $\Theta$, if $\DYN$ plays $\phi_{t+1}=\frac1t\theta_{1:t}^\ast$ and $\SIM$ uses $\varepsilon$-EWOO \eqref{eq:ewoo} with $\varepsilon=1/\sqrt[4]T$ then Algorithm~\ref{alg:general} achieves average regret\vspace{-2mm}
	\begin{equation*}
	\TAR_T\le\RUB_T=\tilde{\mathcal O}\left(\min\left\{\frac{1+\frac1V}{\sqrt T},\frac1{\sqrt[4]T}\right\}+V\right)\sqrt m\qquad\textrm{for}\qquad V^2=\min_{\phi\in\Theta}\frac1T\sum_{t=1}^T\Breg_R(\theta_t^\ast||\phi)
	\end{equation*}
	\vspace{-2mm}
\end{Thm}

\vspace{-3mm}
Observe that if $V$, the average deviation of $\theta_t^\ast$, is $\Omega_T(1)$ then the bound becomes $\mathcal O(V\sqrt m)$ at rate $\tilde{\mathcal O}(1/\sqrt T)$, while if $V=o_T(1)$ the bound tends to zero.
Theorem~\ref{thm:general} can be compared to the main result of \citet{khodak:19}, who set the learning rate via a doubling trick.
We improve upon their result in two aspects.
First, their asymptotic regret is $\mathcal O(D^\ast\sqrt m)$, where $D^\ast$ is the maximum distance between {\em any two optimal actions}.
Note that $V$ is always at most $D^\ast$, and indeed may be much smaller in the presence of outliers.
Second, our result is more general, as we do not need convex $\Breg_R(\theta_t^\ast||\cdot)$.

\begin{Rem}
	We assume an oracle giving a unique $\theta^\ast\in\argmin_{\theta\in\Theta}\sum_{\ell\in S}\ell(\theta)$ for any 
	finite loss sequence $\mathcal S$, which may be inefficient or undesirable.
	One can instead use the last or average iterate of within-task OMD/FTRL for the meta-update; in the supplement we show that this incurs an additional $o(\sqrt m)$ regret term under a quadratic growth assumption that holds in many practical settings \citep{khodak:19}.	
\end{Rem}

\vspace{-3mm}
\paragraph{Related Tasks in Changing Environments:}
In many settings we have a changing environment and so it is natural to study dynamic regret.
This has been widely analyzed by the online learning community \citep{cesa-bianchi:12,jadbabaie:15}, often by showing a dynamic regret bound consisting of a sublinear term plus a bound on the variation in the action or function space.
Using Theorem~\ref{thm:general} we can show dynamic guarantees for GBML via reduction to such bounds.
We provide an example in the Euclidean geometry using the popular path-length-bound $P_\Psi=\sum_{t=2}^T\|\psi_t-\psi_{t-1}\|_2$ for reference actions $\Psi=\{\psi_t\}_{t=1}^T$ \citep{zinkevich:03}.
We use a result showing that OGD with learning rate $\eta\le1/\beta$ over $\alpha$-strongly-convex, $\beta$-strongly-smooth, and $L$-Lipschitz functions has a bound of $\mathcal O(L(1+P_\Psi))$ on its dynamic regret \citep[Corollary~1]{mokhtari:16}.
Observe that in the case of $R(\cdot)=\frac12\|\cdot\|_2^2$ the sequence $\finit_t$ in Theorem~\ref{thm:general} consists of $DG\sqrt m$-Lipschitz quadratic functions.
Thus using Theorem~\ref{thm:general} we achieve the following:
\begin{Thm}\label{thm:dynamic}
	Under Theorem~\ref{thm:general} assumptions, bounded $\Theta$, and $R(\cdot)=\frac12\|\cdot\|_2^2$, if $\DYN$ is OGD with learning rate $\frac1{G\sqrt m}$ and $\SIM$ uses $\varepsilon$-EWOO \eqref{eq:ewoo} with $\varepsilon=1/\sqrt[4]T$ then by using OGD within-task Algorithm~\ref{alg:general} will achieve for any fixed comparator sequence $\Psi=\{\psi_t\}_{t\in[T]}\subset\Theta$ the average regret\vspace{-1mm}
	\begin{equation*}
	\TAR_T\le\RUB_T=\tilde{\mathcal O}\left(\min\left\{\frac{1+\frac1{V_\Psi}}{\sqrt T},\frac1{\sqrt[4]T}\right\}+\min\left\{\frac{1+P_\Psi}{V_\Psi T},\sqrt{\frac{1+P_\Psi}T}\right\}+V_\Psi\right)\sqrt m
	\end{equation*}
	for $V_\Psi^2=\frac1{2T}\sum_{t=1}^T\|\theta_t^\ast-\psi_t\|_2^2$ and $P_\Psi=\sum_{t=2}^T\|\psi_t-\psi_{t-1}\|_2$.
\end{Thm}

This bound controls the average regret across tasks using the deviation $V_\Phi$ of the optimal task parameters $\theta_t^\ast$ from some reference sequence $\Phi$, which is assumed to vary slowly or sparsely so that the path length $P_\Phi$ is small.
Figure~\ref{fig:similarity} illustrates when such a guarantee improves over Theorem~\ref{thm:similarity}.
Note also that Theorem~\ref{thm:dynamic} specifies OGD as the meta-update algorithm $\DYN$, so under the approximation that each task $t$'s last iterate is close to $\theta_t^\ast$ this suggests that simple GBML methods such as Reptile \citep{nichol:18} or FedAvg \citep{mcmahan:17} are adaptive.
The generality of ARUBA also allows for the incorporation of other dynamic regret bounds \citep{hall:16,zhang:17} and other non-static notions of regret \citep{hazan:09}.

\begin{figure}[!t]
	\vspace{1mm}
	\begin{minipage}{0.49\linewidth}
		\hspace{6mm}
		\includegraphics[width=0.285\linewidth]{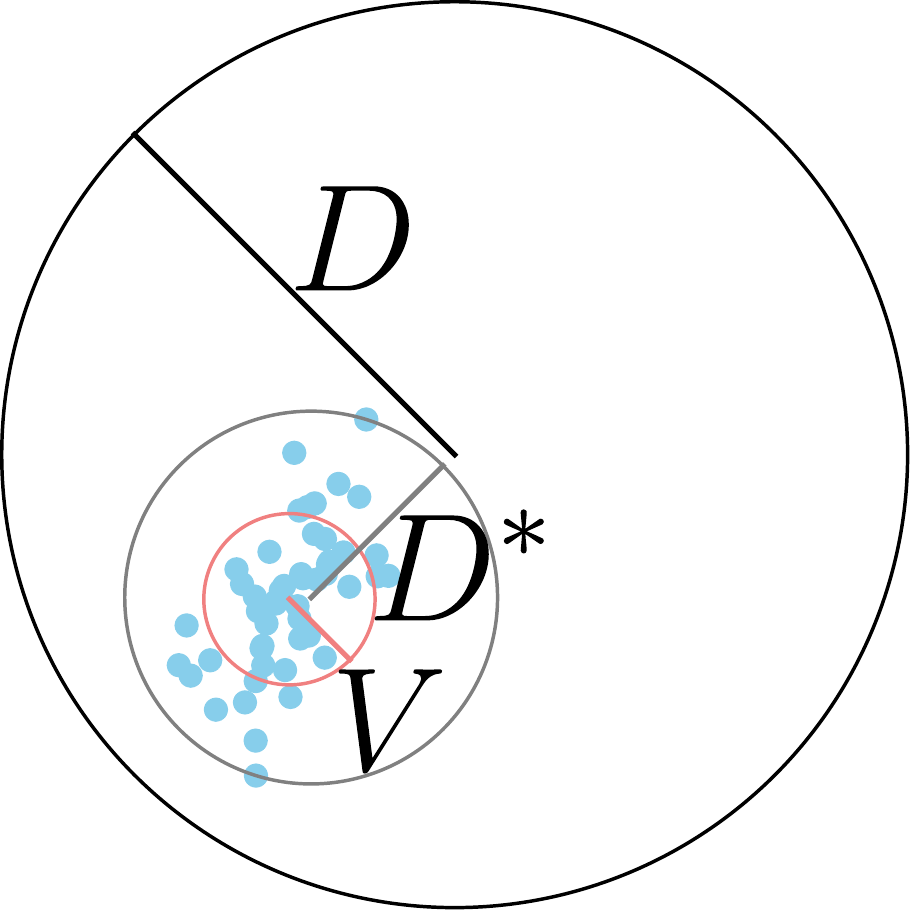}
		\hfill
		\includegraphics[width=0.285\linewidth]{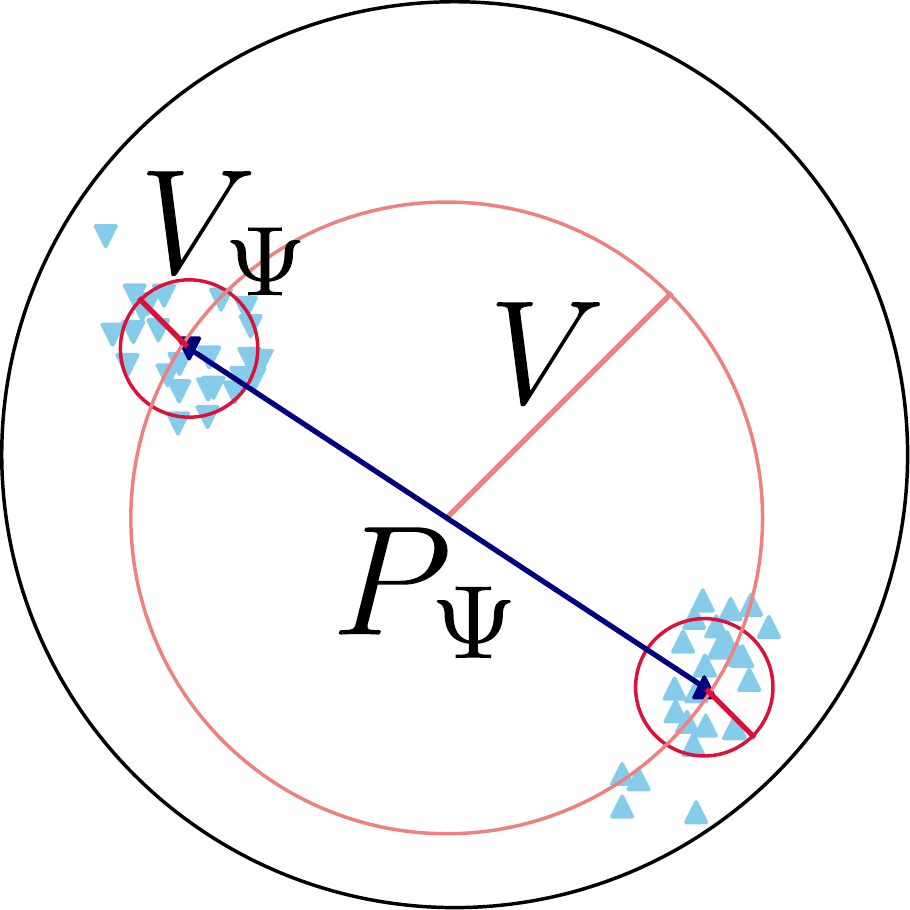}
		\hspace{6mm}
		\caption{\label{fig:similarity}
			Left - Theorem~\ref{thm:similarity} improves upon \cite[Theorem~2.1]{khodak:19} via its dependence on the average deviation $V$ rather than the maximal deviation $D^\ast$ of the optimal task-parameters $\theta_t^\ast$ (light~blue).
			Right - a case where Theorem~\ref{thm:dynamic} yields a strong task-similarity-based guarantee via a dynamic comparator $\Psi$ despite the deviation $V$ being large.
		}
	\end{minipage}
	\hfill
	\begin{minipage}{0.49\linewidth}
		\centering
		\includegraphics[width=0.88\linewidth]{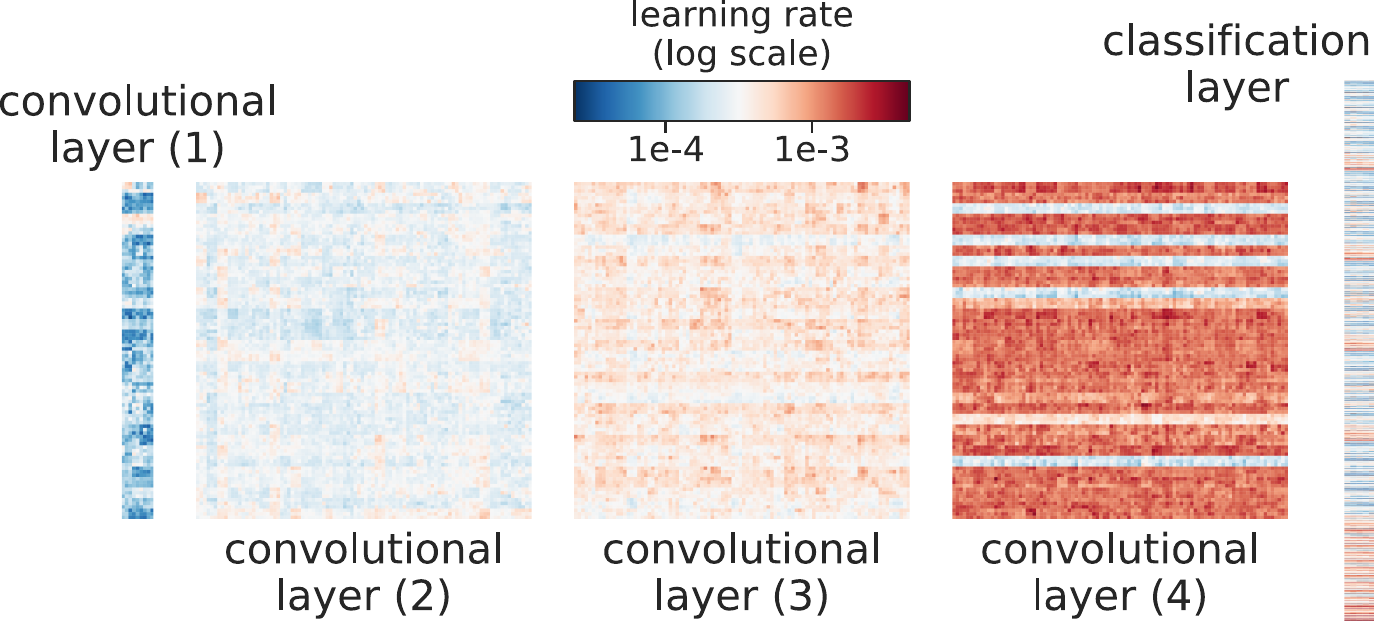}
		\caption{\label{fig:heatmap}
			Learning rate variation across layers of a convolutional net trained on Mini-ImageNet using Algorithm~\ref{alg:aruba}.
			Following intuition outlined in Section~\ref{sec:empirical}, shared feature extractors are not updated much if at all compared to higher layers.
		}
	\end{minipage}
\end{figure}


%% file: adaptive.tex

\vspace{-2mm}
\section{Adapting to the Inter-Task Geometry}\label{sec:adaptive}
\vspace{-1.5mm}

Previously we gave improved guarantees for learning OMD under a simple notion of task-similarity: closeness of the optimal actions $\theta_t^\ast$.
We now turn to new algorithms that can adapt to a more sophisticated task-similarity structure.
Specifically, we study a class of learning algorithms parameterized by an initialization $\phi\in\Theta$ and a symmetric positive-definite matrix $H\in\mathcal M\subset\mathbb R^{d\times d}$ which plays\vspace{-1mm}
\begin{equation}\label{eq:matrix}
\theta_{t,i}=\argmin_{\theta\in\Theta}\frac12\|\theta-\phi\|_{H^{-1}}^2+\langle\nabla_{t,1:i-1},\theta\rangle
\end{equation}
This corresponds $\theta_{t,i+1}=\theta_{t,i}-H\nabla_{t,i}$, so if the optimal actions $\theta_t^\ast$ vary strongly in certain directions, a matrix emphasizing those directions improves within-task performance.
By strong-convexity of $\frac12\|\theta-\phi\|_{H^{-1}}^2$ w.r.t. $\|\cdot\|_{H^{-1}}$, the regret-upper-bound is $\U_t(\phi,H)=\frac12\|\theta_t^\ast-\phi\|_{H^{-1}}^2+\sum_{i=1}^m\|\nabla_{t,i}\|_H^2$  \citep[Theorem~2.15]{shalev-shwartz:11}.
We first study the diagonal case, i.e. learning a per-coordinate learning rate $\eta\in\mathbb R^d$ to get iteration $\theta_{t,i+1}=\theta_{t,i}-\eta_t\odot\nabla_{t,i}$.
We propose to set $\eta_t$ at each task $t$ as follows:
\begin{equation}\label{eq:percoord}
\eta_t=\sr{\frac{\sum_{s<t}\varepsilon_s^2+\frac12(\theta_s^\ast-\phi_s)^2}{\sum_{s<t}\zeta_s^2+\sum_{i=1}^{m_s}\nabla_{s,i}^2}}~~~\textrm{for}~~\varepsilon_t^2=\frac{\varepsilon^2}{(t+1)^p},~\zeta_t^2=\frac{\zeta^2}{(t+1)^p}~\forall~t\ge0,~\textrm{where}~~\varepsilon,\zeta,p>0
\end{equation}
Observe the similarity between this update AdaGrad \citep{duchi:11}, which is also inversely related to the sum of the element-wise squares of all gradients seen so far.
Our method adds multi-task information by setting the numerator to depend on the sum of squared distances between the initializations $\phi_t$ set by the algorithm and that task's optimal action $\theta_t^\ast$.
This algorithm has the following guarantee:
\begin{Thm}\label{thm:diagonal}
	Let $\Theta$ be a bounded convex subset of $\mathbb R^d$, let $\mathcal D\subset\mathbb R^{d\times d}$ be the set of positive definite diagonal matrices, and let each task $t\in[T]$ consist of a sequence of $m$ convex Lipschitz loss functions $\ell_{t,i}:\Theta\mapsto\mathbb R$.
	Suppose for each task $t$ we run the iteration in Equation~\ref{eq:matrix} setting $\phi=\frac1{t-1}\theta_{1:t-1}^\ast$ and setting $H=\Diag(\eta_t)$ via Equation~\ref{eq:percoord} for $\varepsilon=1,\zeta=\sqrt m$, and $p=\frac25$. Then we achieve \vspace{-1mm}
	\begin{equation*}
	\TAR_T
	\le\RUB_T
	=\min_{\begin{smallmatrix}\phi\in\Theta\\H\in\mathcal D\end{smallmatrix}}\tilde{\mathcal O}\left(\sum_{j=1}^d\min\left\{\frac{\frac1{H_{jj}}+H_{jj}}{T^\frac25},\frac1{\sqrt[5]T}\right\}\right)\sr m+\frac1T\sum_{t=1}^T\frac{\|\theta_t^\ast-\phi\|_{H^{-1}}^2}2+\sum_{i=1}^m\|\nabla_{t,i}\|_H^2 \vspace{-2mm}
	\end{equation*}
\end{Thm}

As $T\to\infty$ the average regret converges to the minimum over $\phi,H$ of the last two terms, which corresponds to running OMD with the optimal initialization and per-coordinate learning rate on every task.
The rate of convergence of $T^{-2/5}$ is slightly slower than the usual $1/\sqrt T$ achieved in the previous section;
this is due to the algorithm's adaptivity to within-task gradients, whereas previously we simply assumed a known Lipschitz bound $G_t$ when setting $\eta_t$.
This adaptivity makes the algorithm much more practical, leading to a method for adaptively learning a within-task learning rate using multi-task information;
this is outlined in Algorithm~\ref{alg:aruba} and shown to significantly improve GBML performance in Section~\ref{sec:empirical}.
Note also the per-coordinate separation of the left term, which shows that the algorithm converges more quickly on non-degenerate coordinates.
The per-coordinate specification of $\eta_t$ \eqref{eq:percoord} can be further generalized to learning a full-matrix adaptive regularizer, for which we show guarantees in Theorem~\ref{thm:fullmatrix}.
However, the rate is much slower, and without further assumptions such methods will have $\Omega(d^2)$ computation and memory requirements.

\begin{Thm}\label{thm:fullmatrix}
	Let $\Theta$ be a bounded convex subset of $\mathbb R^d$ and let each task $t\in[T]$ consist of a sequence of $m$ convex Lipschitz loss functions $\ell_{t,i}:\Theta\mapsto\mathbb R$.
	Suppose for each task $t$ we run the iteration in Equation~\ref{eq:matrix} with $\phi=\frac1{t-1}\theta_{1:t-1}^\ast$ and $H$ the unique positive definite solution of $B_t^2=HG_t^2H$ for\vspace{-2.5mm}
	\begin{equation*}
	B_t^2=t\varepsilon^2I_d+\frac12\sum_{s<t}(\theta_s^\ast-\phi_s)(\theta_s^\ast-\phi_s)^T\qquad\textrm{and}\qquad G_t^2=t\zeta^2I_d+\sum_{s<t}\sum_{i=1}^m\nabla_{s,i}\nabla_{s,i}^T\vspace{-2mm}
	\end{equation*}
	for $\varepsilon=1/\sqrt[8]T$ and $\zeta=\sqrt m/\sqrt[8]T$.
	Then for $\lambda_j$ corresponding to the $j$th largest eigenvalue we have\vspace{-2.5mm}
	\begin{equation*}
	\TAR_T
	\le\RUB_T
	=\tilde{\mathcal O}\left(\frac1{\sqrt[8]T}\right)\sqrt m+\min_{\begin{smallmatrix}\phi\in\Theta\\H\succ0\end{smallmatrix}}\frac{2\lambda_1^2(H)}{\lambda_d(H)}\frac{1+\log T}T+\sum_{t=1}^T\frac{\|\theta_t^\ast-\phi^\ast\|_{H^{-1}}^2}2+\sum_{i=1}^m\|\nabla_{t,i}\|_H^2
	\end{equation*}
\end{Thm}

%% file: statistical.tex

\vspace{-2mm}
\section{Fast Rates and High Probability Bounds for Statistical Learning-to-Learn}\label{sec:statistical}
\vspace{-1.5mm}

Batch-setting transfer risk bounds have been an important motivation for studying LTL via online learning \citep{alquier:17,khodak:19,denevi:19}.
If the regret-upper-bounds are convex, which is true for most practical variants of OMD/FTRL, ARUBA yields several new results in the classical distribution over task-distributions setup of \citet{baxter:00}.
In Theorem~\ref{thm:statistical} we present bounds on the risk $\ell_{\mathcal P}(\bar\theta)$ of the parameter $\bar\theta$ obtained by running OMD/FTRL on i.i.d. samples from a new task distribution $\mathcal P$ and averaging the iterates.

\begin{Thm}\label{thm:statistical}
	Assume $\Theta,\mathcal X$ are convex Euclidean subsets.
	Let convex losses $\ell_{t,i}:\Theta\mapsto[0,1]$ be drawn i.i.d. $\mathcal P_t\sim\Q,\{\ell_{t,i}\}_i\sim\mathcal P_t^m$ for distribution $\Q$ over tasks.
	Suppose they are passed to an algorithm with average regret upper-bound $\RUB_T$ that at each $t$ picks $x_t\in\mathcal X$ to initialize a within-task method with convex regret upper-bound $\U_t:\mathcal X\mapsto[0,B\sqrt m]$, for $B\ge0$.
	If the within-task algorithm is initialized by $\bar x=\frac1Tx_{1:T}$ and it takes actions $\theta_1,\dots,\theta_m$ on $m$ i.i.d. losses from new task $\mathcal P\sim\Q$ then $\bar\theta=\frac1m\theta_{1:m}$ satisfies the following transfer risk bounds for any $\theta^\ast\in\Theta$ (all w.p. $1-\delta$):\vspace{-2mm}
	\begin{enumerate}[leftmargin=*]
		\item{\bf general case:} $\quad\E_{\mathcal P\sim\Q}\E_{\mathcal P^m}\ell_{\mathcal P}(\bar\theta)\le\E_{\mathcal P\sim\Q}\ell_{\mathcal P}(\theta^\ast)+\mathcal L_T\quad$ for $\quad\mathcal L_T=\frac\RUB m+B\sqrt{\frac8{mT}\log\frac1\delta}$.\vspace{-1mm}
		\item{\bf $\rho$-self-bounded losses $\ell$:} \quad if $\exists~\rho>0$ s.t. $\rho\E_{\ell\sim\mathcal P}\Delta\ell(\theta)\ge\E_{\ell\sim\mathcal P}(\Delta\ell(\theta)-\E_{\ell\sim\mathcal P}\Delta\ell(\theta))^2$ for all distributions $\mathcal P\sim\Q$, where $\Delta\ell(\theta)=\ell(\theta)-\ell(\theta^\ast)$ for any $\theta^\ast\in\argmin_{\theta\in\Theta}\ell_{\mathcal P}(\theta)$, then for $\mathcal L_T$ as above we have
		$\quad\E_{\mathcal P\sim\Q}\ell_{\mathcal P}(\bar\theta)
		\le\E_{\mathcal P\sim\Q}\ell_{\mathcal P}(\theta^\ast)+\mathcal L_T+\sqrt{\frac{2\rho\mathcal L_T}m\log\frac2\delta}+\frac{3\rho+2}m\log\frac2\delta$.\vspace{-1mm}
		\item{\bf $\alpha$-strongly-convex, $G$-Lipschitz regret-upper-bounds $\U_t$:}\quad in parts 1 and 2 above we can substitute $~~\mathcal L_T
		=\frac{\RUB+\min_x\E_{\mathcal P\sim\Q}\U(x)}m+\frac{4G}T\sqrt{\frac\RUB{\alpha m}\log\frac{8\log T}\delta}
		+\frac{\max\{16G^2,6\alpha B\sqrt m\}}{\alpha mT}\log\frac{8\log T}\delta$.\vspace{-2mm}
	\end{enumerate}
\end{Thm}

In the {\bf general case}, Theorem~\ref{thm:statistical} provides bounds on the excess transfer risk decreasing with $\RUB/m$ and $1/\sqrt{mT}$.
Thus if $\RUB$ improves with task-similarity so will the transfer risk as $T\to\infty$.
Note that the second term is $1/\sqrt{mT}$ rather than $1/\sqrt T$ as in most-analyses \citep{khodak:19,denevi:19};
this is because regret is $m$-bounded but the OMD regret-upper-bound is $\mathcal O(\sqrt m)$-bounded.
The results also demonstrate ARUBA's ability to utilize specialized results from the online-to-batch conversion literature.
This is witnessed by the guarantee for {\bf self-bounded losses}, a class which \citet{zhang:05} shows includes linear regression;
we use a result by the same author to obtain high-probability bounds, whereas previous GBML bounds are in-expectation \citep{khodak:19,denevi:19}.
We also apply a result due to \citet{kakade:08b} for the case of {\bf strongly-convex regret-upper-bounds}, enabling fast rates in the number of tasks $T$.
The strongly-convex case is especially relevant for GBML since it holds for OGD with fixed learning rate.

\newpage
We present two consequences of these results for the algorithms from Section~\ref{sec:simdyn} when run on i.i.d. data.
To measure task-similarity we use the {\bf variance} $V_\Q^2=\min_{\phi\in\Theta}\E_{\mathcal P\sim\Q}\E_{\mathcal P^m}\|\theta^\ast-\phi\|_2^2$ of the empirical risk minimizer $\theta^\ast$ of an $m$-sample task drawn from $\Q$.
If $V_\Q$ is known we can use strong-convexity of the regret-upper-bounds to obtain a fast rate for learning the initialization, as shown in the first part of Corollary~\ref{cor:statistical}.
The result can be loosely compared to \citet{denevi:19}, who provide a similar asymptotic improvement but with a slower rate of $\mathcal O(1/\sqrt T)$ in the second term.
However, their task-similarity measures the deviation of the true, not empirical, risk-minimizers, so the results are not directly comparable.
Corollary~\ref{cor:statistical} also gives a guarantee for when we do {\em not} know $V_\Q$ and must learn the learning rate $\eta$ in addition to the initialization;
here we match the rate of \citet{denevi:19}, who do not learn $\eta$, up to some additional fast $o(1/\sqrt m)$ terms.

\begin{Cor}\label{cor:statistical}
	In the setting of Theorems~\ref{thm:similarity} \&~\ref{thm:statistical}, if $\delta\le1/e$ and Algorithm~\ref{alg:general} uses within-task OGD with initialization $\phi_{t+1}=\frac1t\theta_{1:t}^\ast$ and step-size $\eta_t=\frac{V_\Q+1/\sqrt T}{G\sqrt m}$ for $V_\Q$ as above, then w.p. $1-\delta$\vspace{-2mm}
	\begin{equation*}
	\E_{\mathcal P\sim\Q}\E_{\mathcal P^m}\ell_{\mathcal P}(\bar\theta)\le\E_{\mathcal P\sim\Q}\ell_{\mathcal P}(\theta^\ast)+\tilde{\mathcal O}\left(\frac{V_\Q}{\sqrt m}+\left(\frac1{\sqrt{mT}}+\frac1T\right)\log\frac1\delta\right)\vspace{-2mm}
	\end{equation*}
	If $\eta_t$ is set adaptively using $\varepsilon$-EWOO as in Theorem~\ref{thm:similarity} for $\varepsilon=1/\sqrt[4]{mT}+1/\sqrt m$ then w.p. $1-\delta$\vspace{-2mm}
	\begin{equation*}
	\E_{\mathcal P\sim\Q}\E_{\mathcal P^m}\ell_{\mathcal P}(\bar\theta)\le\E_{\mathcal P\sim\Q}\ell_{\mathcal P}(\theta^\ast)+\tilde{\mathcal O}\left(\frac{V_\Q}{\sqrt m}+\min\left\{\frac{\frac1{\sqrt m}+\frac1{\sqrt T}}{V_\Q m},\frac1{\sqrt[4]{m^3T}}+\frac1m\right\}+\sqrt{\frac1T\log\frac1\delta}\right)
	\end{equation*}
\end{Cor}

%% file: empirical.tex

\vspace{-2mm}
\section{Empirical Results: Adaptive Methods for Few-Shot \& Federated Learning}\label{sec:empirical}
\vspace{-2mm}


\begin{figure}[!t]
\begin{minipage}{0.59\linewidth}
\begin{algorithm}[H]
	\DontPrintSemicolon
	\KwIn{$T$ tasks, update method for meta-initialization, within-task descent method, settings $\varepsilon,\zeta,p>0$}
	Initialize $b_1\gets\varepsilon^21_d$, $g_1\gets\zeta^21_d$\\
	\For{{\em task} $t=1,2,\dots,T$}{
		Set $\phi_t$ according to update method, $\eta_t\gets\sqrt{b_t/g_t}$\\
		Run descent method from $\phi_t$ with learning rate $\eta_t$:\\
		\qquad observe gradients $\nabla_{t,1},\dots,\nabla_{t,m_t}$\\
		\qquad obtain within-task parameter $\hat\theta_t$\\
		$b_{t+1}\gets b_t+\frac{\varepsilon^21_d}{(t+1)^p}+\frac12(\phi_t-\hat\theta_t)^2$\\ $g_{t+1}\gets g_t+\frac{\zeta^21_d}{(t+1)^p}+\sum_{i=1}^{m_t}\nabla_{t,i}^2$
	}
	\KwResult{initialization $\phi_T$, learning rate $\eta_T=\sqrt{b_T/g_T}$}
	\caption{\label{alg:aruba}
		ARUBA: an approach for modifying a generic batch GBML method to learn a per-coordinate learning rate.
		Two specialized variants provided below.
	}
\end{algorithm}
{\bf ARUBA++:} starting with $\eta_{T,1}=\eta_T$ and $g_{T,1}=g_T$, adaptively reset the learning rate by setting $\hat g_{T,i+1}\gets\hat g_{T,i}+c\nabla_i^2$ for some $c>0$ and then updating $\eta_{T,i+1}\gets\sqrt{b_T/g_{T,i+1}}$.
{\bf Isotropic:} $b_t$ and $g_t$ are scalars tracking the sum of squared distances and sum of squared gradient norms, respectively.
\end{minipage}
\hfill
\begin{minipage}{0.39\linewidth}
	\centering
	\includegraphics[width=0.95\linewidth]{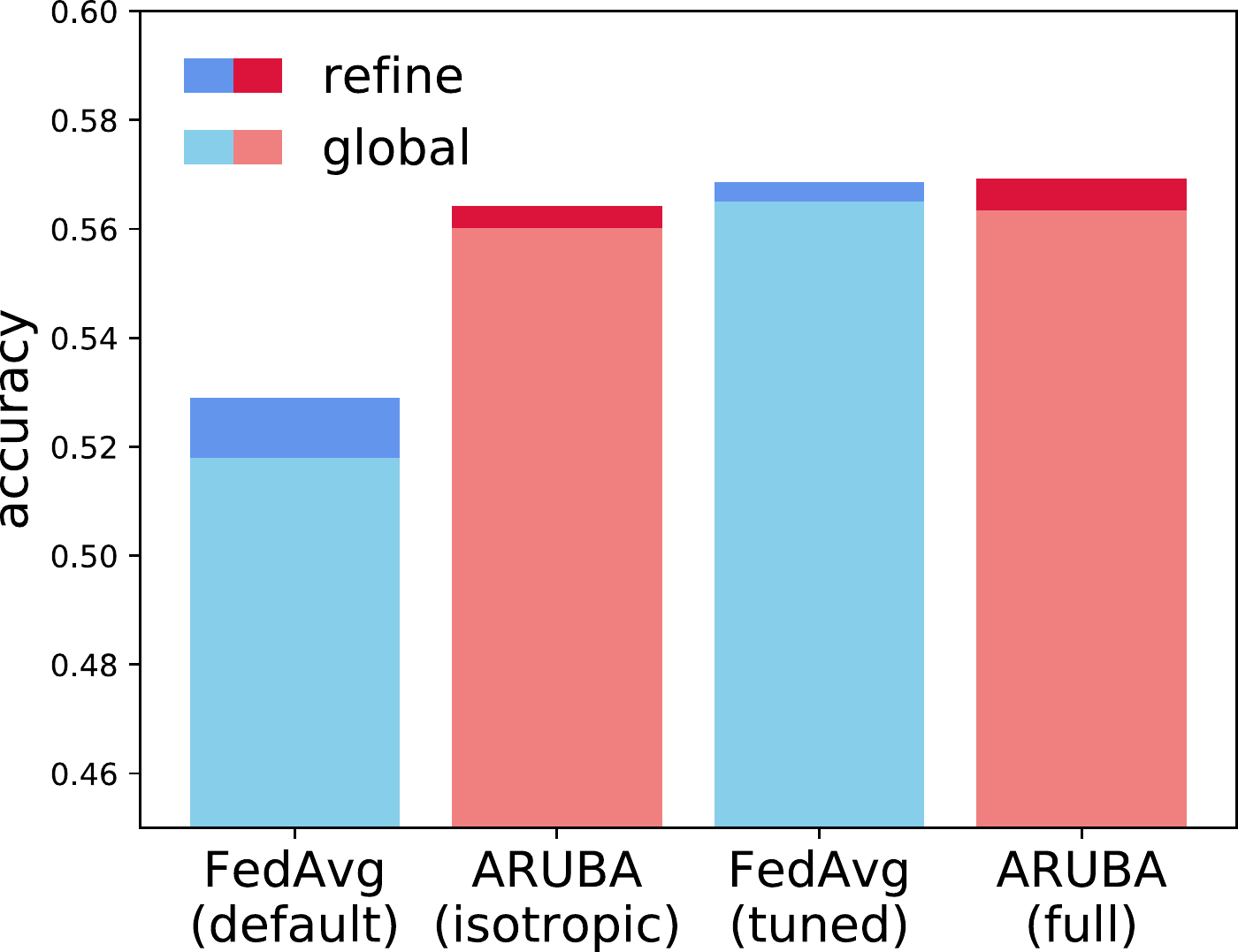}
	\vspace{-0.5mm}
	\caption{\label{fig:fedavg}
		Next-character prediction performance for recurrent networks trained on the Shakespeare dataset \citep{caldas:18} using FedAvg \citep{mcmahan:17} and its modifications by Algorithm~\ref{alg:aruba}.
		Note that the two ARUBA methods require no learning rate tuning when personalizing the model (refine), unlike {\em both} FedAvg methods;
		this is a critical improvement in federated settings.
		Furthermore, isotropic ARUBA has negligible overhead by only communicating scalars.
	}
\end{minipage}
\end{figure}

A generic GBML method does the following at iteration $t$: 
(1) initialize a descent method at $\phi_t$;
(2) take gradient steps with learning rate $\eta$ to get task-parameter $\hat\theta_t$;
(3) update meta-initialization to $\phi_{t+1}$.
Motivated by Section~\ref{sec:adaptive}, in Algorithm~\ref{alg:aruba} we outline a generic way of replacing $\eta$ by a per-coordinate rate learned on-the-fly.
This entails keeping track of two quantities: 
(1) $b_t\in\mathbb R^d$, a per-coordinate sum over $s<t$ of the squared distances from the initialization $\phi_s$ to within-task parameter $\hat\theta_s$;
(2) $g_t\in\mathbb R^d$, a per-coordinate sum of the squared gradients seen so far.
At task $t$ we set $\eta$ to be the element-wise square root of $b_t/g_t$, allowing multi-task information to inform the trajectory.
For example, if along coordinate $j$ the $\hat\theta_{t,j}$ is usually not far from initialization then $b_j$ will be small and thus so will $\eta_j$;
then if on a new task we get a high noisy gradient along coordinate $j$ the performance will be less adversely affected because it will be down-weighted by the learning rate.
Single-task algorithms such as AdaGrad \cite{duchi:11} and Adam \cite{kingma:15} also work by reducing the learning rate along frequent directions.
However, in meta-learning some coordinates may be frequently updated during meta-training because good task-weights vary strongly from the best initialization along them, and thus their gradients should not be downweighted;
ARUBA encodes this intuition in the numerator using the distance-traveled per-task along each direction, which increases the learning rate along high-variance directions.
We show in Figure~\ref{fig:heatmap} that this is realized in practice, as ARUBA assigns a faster rate to deeper layers than to lower-level feature extractors, following standard intuition in parameter-transfer meta-learning.
As described in Algorithm~\ref{alg:aruba}, we also consider two variants:
ARUBA++, which updates the meta-learned learning-rate at meta-test-time in a manner similar to AdaGrad, and Isotropic ARUBA, which only tracks scalar quantities and is thus useful for communication-constrained settings.

\vspace{-3mm}
\paragraph{Few-Shot Classification:}
We first examine if Algorithm~\ref{alg:aruba} can improve performance on Omniglot \citep{lake:11} and Mini-ImageNet \citep{ravi:17}, two standard few-shot learning benchmarks, when used to modify Reptile, a simple meta-learning method \citep{nichol:18}. 
In its serial form Reptile is roughly the algorithm we study in Section~\ref{sec:simdyn} when OGD is used within-task and $\eta$ is fixed.
Thus we can set Reptile+ARUBA to be Algorithm~\ref{alg:aruba} with $\hat\theta_t$ the last iterate of OGD and the meta-update a weighted sum of $\hat\theta_t$ and $\phi_t$.
In practice, however, Reptile uses Adam \citep{kingma:15} to exploit multi-task gradient information.
As shown in Table~\ref{tbl:meta}, ARUBA matches or exceeds this baseline on Mini-ImageNet, although on Omniglot it requires the additional within-task updating of ARUBA++ to show improvement.

\vspace{-1mm}
It is less clear how ARUBA can be applied to MAML \citep{finn:17}, as by only taking one step the distance traveled will be proportional to the gradient, so $\eta$ will stay fixed.
We also do not find that ARUBA improves multi-step MAML -- perhaps not surprising as it is further removed from our theory due to its use of held-out data.
In Table~\ref{tbl:meta} we compare to Meta-SGD \citep{li:17}, which does learn a per-coordinate learning rate for MAML by automatic differentiation.
This requires more computation but does lead to consistent improvement.
As with the original Reptile, our modification performs better on Mini-ImageNet but worse on Omniglot compared to MAML and its modification Meta-SGD.

\vspace{-3mm}
\paragraph{Federated Learning:}
A main goal in this setting is to use data on heterogeneous nodes to learn a global model without much communication; leveraging this to get a personalized model is an auxiliary goal \citep{smith:17}, with a common application being next-character prediction on mobile devices.
A popular method is FedAvg \citep{mcmahan:17}, where at each communication round $r$ the server sends a global model $\phi_r$ to a batch of nodes, which then run local OGD;
the server then sets $\phi_{r+1}$ to the average of the returned models.
This can be seen as a GBML method with each node a task, making it easy to apply ARUBA:
each node simply sends its accumulated squared gradients to the server together with its model.
The server can use this information and the squared difference between $\phi_r$ and $\phi_{r+1}$ to compute a learning rate $\eta_{r+1}$ via Algorithm~\ref{alg:aruba} and send it to each node in the next round.
We use FedAvg with ARUBA to train a character LSTM \citep{hochreiter:97} on the Shakespeare dataset, a standard benchmark of a thousand users with varying amounts of non-i.i.d. data \citep{mcmahan:17,caldas:18}.
Figure~\ref{fig:fedavg} shows that ARUBA significantly improves over non-tuned FedAvg and matches the performance of FedAvg with a tuned learning rate schedule.
Unlike both baselines we also do not require step-size tuning when refining the global model for personalization.
This reduced need for hyperparameter optimization is crucial in federated settings, where the number of user-data accesses are extremely limited.


\begin{table}[!t]
	\centering
	\footnotesize
	\begin{tabular}{cccccc}  
		\toprule
		&& \multicolumn{2}{c}{20-way Omniglot} & \multicolumn{2}{c}{5-way Mini-ImageNet} \vspace{-0.4mm}\\
		&& 1-shot & 5-shot & 1-shot & 5-shot \vspace{-1mm}\\
		\midrule
		& 1st-Order MAML \cite{finn:17} & $89.4\pm0.5$ & $\underline{97.9}\pm0.1$ & $48.07\pm1.75$ & $63.15\pm0.91$ \\
		1st & Reptile \citep{nichol:18} w. Adam \cite{kingma:15} & $89.43\pm0.14$ & $97.12\pm0.32$ & $49.97\pm0.32$ & $\underline{\bf 65.99}\pm0.58$ \\
		Order & Reptile w. ARUBA & $86.67\pm0.17$ & $96.61\pm0.13$ & $\underline{\bf 50.73}\pm0.32$ & $65.69\pm0.61$ \\
		& Reptile w. ARUBA++ & $\underline{89.66}\pm0.3$ & $97.49\pm0.28$ & $50.35\pm0.74$ & $65.89\pm0.34$ \vspace{-0.8mm}\\
		\midrule
		2nd & 2nd-Order MAML & $95.8\pm0.3$ & $98.9\pm0.2$ & $48.7\pm1.84$ & $63.11\pm0.92$ \\
		Order & Meta-SGD \cite{li:17} & $\underline{\bf 95.93}\pm0.38$ & $\underline{\bf 98.97}\pm0.19$ & $\underline{50.47}\pm1.87$ & $\underline{64.03}\pm0.94$ \vspace{-0.8mm}\\
		\bottomrule
	\end{tabular}
	\vspace{1mm}
	\caption{\label{tbl:meta}
		Meta-test-time performance of GBML algorithms on few-shot classification benchmarks.
		1st-order and 2nd-order results obtained from \citet{nichol:18} and \citet{li:17}, respectively.
	}
	\vspace{-3mm}
\end{table}

\vspace{-2mm}
\section{Conclusion}
\vspace{-1.5mm}

In this paper we introduced ARUBA, a framework for analyzing GBML that is both flexible and consequential, yielding new guarantees for adaptive, dynamic, and statistical LTL via online learning.
As a result we devised a novel per-coordinate learning rate applicable to generic GBML procedures, improving their training and meta-test-time performance on few-shot and federated learning.
We see great potential for applying ARUBA to derive many other new LTL methods in a similar manner.

\vspace{-2mm}
\section*{Acknowledgments}
\vspace{-1.5mm}

We thank Jeremy Cohen, Travis Dick, Nikunj Saunshi, Dravyansh Sharma, Ellen Vitercik, and our three anonymous reviewers for helpful feedback.
This work was supported in part by DARPA FA875017C0141, National Science Foundation grants CCF-1535967, CCF-1910321, IIS-1618714, IIS-1705121, IIS-1838017, and IIS-1901403, a Microsoft Research Faculty Fellowship, a Bloomberg Data Science research grant, an Amazon Research Award, an Amazon Web Services Award, an Okawa Grant, a Google Faculty Award, a JP Morgan AI Research Faculty Award, and a Carnegie Bosch Institute Research Award. Any opinions, findings and conclusions, or recommendations expressed in this material are those of the authors and do not necessarily reflect the views of DARPA, the National Science Foundation, or any other funding agency.

%% file: background.tex

\section{Background and Results for Online Convex Optimization}\label{app:background}

Throughout the appendix we assume all subsets are convex and in a finite-dimensional real vector space with inner product $\langle\cdot,\cdot\rangle$ unless explicitly stated.
Let $\|\cdot\|_\ast$ be the dual norm of $\|\cdot\|$ and note that the dual norm of $\|\cdot\|_2$ is itself.
For sequences of scalars $\sigma_1,\dots,\sigma_T\in\mathbb{R}$ we will use the notation $\sigma_{1:t}$ to refer to the sum of the first $t$ of them.
In the online learning setting, we will use the shorthand $\nabla_t$ to denote the subgradient of $\ell_t:\Theta\mapsto\mathbb R$ evaluated at action $\theta_t\in\Theta$.
We will use $\Conv(S)$ to refer to the convex hull of a set of points $S$ and $\Proj_S(\cdot)$ to be the projection to any convex subset $S$.

\subsection{Convex Functions}\label{subsec:functions}

We first state the related definitions of {\em strong convexity} and {\em strong smoothness}:
\begin{Def}\label{def:convex}
	An everywhere sub-differentiable function $f:S\mapsto\mathbb R$ is {\bf $\alpha$-strongly-convex} w.r.t. norm $\|\cdot\|$ if
	$$f(y)\ge f(x)+\langle\nabla f(x),y-x\rangle+\frac\alpha2\|y-x\|^2~\forall~x,y\in S$$
\end{Def}
\begin{Def}\label{def:smooth}
	An everywhere sub-differentiable function $f:S\mapsto\mathbb{R}$ is {\bf $\beta$-strongly-smooth} w.r.t. norm $\|\cdot\|$ if
	$$f(y)\le f(x)+\langle\nabla f(x),y-x\rangle+\frac\beta2\|y-x\|^2~\forall~x,y\in S$$
\end{Def}

Finally, we will also consider functions that are exp-concave \citep{hazan:07}:
\begin{Def}\label{def:expconcave}
	An everywhere sub-differentiable function $f:S\mapsto\mathbb R$ is {\bf $\gamma$-exp-concave} if $\exp(-\gamma f(x))$ is concave.
	For $S\subset\mathbb R$ we have that $\frac{\partial_{xx}f(x)}{(\partial_xf(x))^2}\ge\gamma~\forall~x\in S\implies f$ is $\gamma$-exp-concave.
\end{Def}

We now turn to the {\em Bregman divergence} and a discussion of several useful properties \cite{bregman:67,banerjee:05}:
\begin{Def}\label{def:bregman:app}
	Let $f:S\mapsto\mathbb{R}$ be an everywhere sub-differentiable strictly convex function.
	Its {\bf Bregman divergence} is defined as
	$$\Breg_f(x||y)=f(x)-f(y)-\langle\nabla f(y),x-y\rangle$$
	The definition directly implies that $\Breg_f(\cdot||y)$ preserves the (strong or strict) convexity of $f$ for any fixed  $y\in S$.
	Strict convexity further implies $\Breg_f(x||y)\ge0~\forall~x,y\in S$, with equality iff $x=y$.
	Finally, if $f$ is $\alpha$-strongly-convex, or $\beta$-strongly-smooth, w.r.t. $\|\cdot\|$ then Definitions~\ref{def:convex} and~\ref{def:smooth} imply $\Breg_f(x||y)\ge\frac\alpha2\|x-y\|^2$ or $\Breg_f(x||y)\le\frac\beta2\|x-y\|^2$, respectively.
\end{Def}
\begin{Clm}\label{clm:mean}
	Let $f:S\mapsto\mathbb{R}$ be a strictly convex function on $S$, $\alpha_1,\dots,\alpha_n\in\mathbb{R}$ be a sequence satisfying $\alpha_{1:n}>0$, and $x_1,\dots,x_n\in S$.
	Then 
	$$\bar x=\frac{1}{\alpha_{1:n}}\sum_{i=1}^n\alpha_ix_i=\argmin_{y\in S}\sum_{i=1}^n\alpha_i\Breg_f(x_i||y)$$
\end{Clm}
\begin{proof}
	$\forall~y\in S$ we have
	\begin{align*}
		\sum_{i=1}^n&\alpha_i\left(\Breg_f(x_i||y)-\Breg_f(x_i||\bar x)\right)\\
		&=\sum_{i=1}^n\alpha_i\left(f(x_i)-f(y)-\langle\nabla f(y),x_i-y\rangle-f(x_i)+f(\bar x)+\langle\nabla f(\bar x),x_i-\bar x\rangle\right)\\
		&=\left(f(\bar x)-f(y)+\langle\nabla f(y),y\rangle\right)\alpha_{1:n}+\sum_{i=1}^n\alpha_i\left(-\langle\nabla f(\bar x),\bar x\rangle+\langle\nabla f(\bar x)-\nabla f(y),x_i\rangle\right)\\
		&=\left(f(\bar x)-f(y)-\langle\nabla f(y),\bar x-y\rangle\right)\alpha_{1:n}\\
		&=\alpha_{1:n}\Breg_f(\bar x||y)
	\end{align*}
	By Definition~\ref{def:bregman:app} the last expression has a unique minimum at $y=\bar x$.
\end{proof}

\newpage
\subsection{Online Algorithms}\label{subsec:algorithms}

Here we provide a review of the online algorithms we use.
Recall that in this setting our goal is minimizing regret:
\begin{Def}\label{def:regret}
	The {\bf regret} of an agent playing actions $\{\theta_t\in\Theta\}_{t\in[T]}$ on a sequence of loss functions $\{\ell_t:\Theta\mapsto\mathbb R\}_{t\in[T]}$ is
	$$\R_T=\sum_{t=1}^T\ell_t(\theta_t)-\min_{\theta\in\Theta}\sum_{t=1}^T\ell_t(\theta)$$
\end{Def}
Within-task our focus is on two closely related meta-algorithms, Follow-the-Regularized-Leader (FTRL) and (linearized lazy) Online Mirror Descent (OMD).
\begin{Def}\label{def:ftrl}
	Given a strictly convex function $R:\Theta\mapsto\mathbb R$, starting point $\phi\in\Theta$, fixed learning rate $\eta>0$, and a sequence of functions $\{\ell_t:\Theta\mapsto\mathbb R\}_{t\ge1}$, {\bf Follow-the-Regularized Leader ($\FTRL_{\phi,\eta}^{(R)}$)} plays
	$$\theta_t=\argmin_{\theta\in\Theta}\Breg_R(\theta||\phi)+\eta\sum_{s<t}\ell_s(\theta)$$
\end{Def}
\begin{Def}\label{def:omd}
	Given a strictly convex function $R:\Theta\mapsto\mathbb R$, starting point $\phi\in\Theta$, fixed learning rate $\eta>0$, and a sequence of functions $\{\ell_t:\Theta\mapsto\mathbb R\}_{t\ge1}$, {\bf lazy linearized Online Mirror Descent ($\OMD_{\phi,\eta}^{(R)}$)} plays
	$$\theta_t=\argmin_{\theta\in\Theta}\Breg_R(\theta||\phi)+\eta\sum_{s<t}\langle\nabla_s,\theta\rangle$$
\end{Def}

These formulations make the connection between the two algorithms -- their equivalence in the linear case $\ell_s(\cdot)=\langle\nabla_s,\cdot\rangle$ -- very explicit.
There exists a more standard formulation of OMD that is used to highlight its generalization of OGD -- the case of $R(\cdot)=\frac12\|\cdot\|_2^2$ -- and the fact that the update is carried out in the dual space induced by $R$ \citep[Section~5.3]{hazan:15}.
However, we will only need the following regret bound satisfied by both \citep[Theorems~2.11 and~2.15]{shalev-shwartz:11}
\begin{Thm}\label{thm:ftrlomd}
	Let $\{\ell_t:\Theta\mapsto\mathbb R\}_{t\in[T]}$ be a sequence of convex functions that are $G_t$-Lipschitz w.r.t. $\|\cdot\|$ and let $R:S\mapsto\mathbb R$ be 1-strongly-convex.
	Then the regret of both $\FTRL_{\eta,\phi}^{(R)}$ and $\OMD_{\eta,\phi}^{(R)}$ is bounded by
	$$\R_T\le\frac{\Breg_R(\theta^\ast||\phi)}\eta+\eta G^2T$$
	for all $\theta^\ast\in\Theta$ and $G^2\ge\frac1T\sum_{t=1}^TG_t^2$.
\end{Thm}

We next review the online algorithms we use for the meta-update.
The main requirement here is logarithmic regret guarantees for the case of strongly convex loss functions, which is satisfied by two well-known algorithms:
\begin{Def}\label{def:ftl}
	Given a sequence of strictly convex functions $\{\ell_t:\Theta\mapsto\mathbb R\}_{t\ge1}$, {\bf Follow-the-Leader (FTL)} plays arbitrary $\theta_1\in\Theta$ and for $t>1$ plays
	$$\theta_t=\argmin_{\theta\in\Theta}\sum_{s<t}\ell_s(\theta)$$
\end{Def}
\begin{Def}\label{def:aogd}
	Given a sequence of functions $\{\ell_t:\Theta\mapsto\mathbb R\}_{t\ge1}$ that are $\alpha_t$-strongly-convex w.r.t. $\|\cdot\|_2$, {\bf Adaptive OGD (AOGD)} plays arbitrary $\theta_1\in\Theta$ and for $t>1$ plays
	$$\theta_{t+1}=\Proj_\Theta\left(\theta_t-\frac1{\alpha_{1:t}}\nabla f(\theta_t)\right)$$
\end{Def}

\citet[Theorem~2]{kakade:08} and \citet[Theorem~2.1]{bartlett:08} provide for FTL and AOGD, respectively, the following regret bound:
\begin{Thm}\label{thm:ftlaogd}
	Let $\{\ell_t:\Theta\mapsto\mathbb R\}_{t\in[T]}$ be a sequence of convex functions that are $G_t$-Lipschitz and $\alpha_t$-strongly-convex w.r.t. $\|\cdot\|$.
	Then the regret of both FTL and AOGD is bounded by
	$$\R_T\le\frac12\sum_{t=1}^T\frac{G_t^2}{\alpha_{1:t}}$$
\end{Thm}

Finally, we state the EWOO algorithm due to \citet{hazan:07}.
While difficult to run in high-dimensions, we will be running this method in single dimensions, when computing it requires only one integral.

\begin{Def}\label{def:ewoo}
	Given a sequence of $\gamma$-exp-concave functions $\{\ell_t:\Theta\mapsto\mathbb R\}$, {\bf Exponentially Weighted Online Optimization (EWOO)} plays
	$$\theta_t=\frac{\int_\Theta \theta\exp(-\gamma\sum_{s<t}\ell_s(\theta))d\theta}{\int_\Theta\exp(-\gamma\sum_{s<t}\ell_s(\theta))d\theta}$$
\end{Def}

%
%

\citet[Theorem~7]{hazan:07} provide the following guarantee for EWOO, which is notable for its lack of explicit dependence on the Lipschitz constant.

\begin{Thm}\label{thm:ewoo}
	Let $\{\ell_t:\Theta\mapsto\mathbb R\}$ be a sequence of $\gamma$-exp-concave functions.
	Then the regret of EWOO is bounded by
	$$\R_T\le\frac d\gamma(1+\log(T+1))$$
\end{Thm}

\subsection{Online-to-Batch Conversion}\label{subsec:batch}

Finally, as we are also interested in distributional meta-learning, we discuss some techniques for converting regret guarantees into generalization bounds, which are usually named {\em online-to-batch conversions}.
We first state some standard results.
\begin{Prp}\label{prp:o2bexp}
	If a sequence of bounded convex loss functions $\{\ell_t:\Theta\mapsto\mathbb R\}_{t\in[T]}$ drawn i.i.d. from some distribution $\mathcal D$ is given to an online algorithm with regret bound $\R_T$ that generates a sequence of actions $\{\theta_t\in\Theta\}_{t\in[T]}$ then
	$$\E_{\mathcal D^T}\E_{\ell\sim\mathcal D}\ell(\bar\theta)\le\E_{\ell\sim\mathcal D}\ell(\theta^\ast)+\frac{\R_T}T$$
	for $\bar\theta=\frac1T\theta_{1:T}$ and any $\theta^\ast\in\Theta$.
\end{Prp}
\begin{proof}
	Applying Jensen's inequality yields
	\begin{align*}
	\E_{\mathcal D^T}\E_{\ell\sim\mathcal D}\ell(\bar\theta)
	&\le\frac1T\E_{\mathcal D^T}\sum_{t=1}^T\E_{\ell_t'\sim\mathcal D}\ell_t'(\theta_t)\\
	&=\frac1T\E_{\{\ell_t\}\sim\mathcal D^T}\left(\sum_{t=1}^T\E_{\ell_t'\sim\mathcal D}\ell_t'(\theta_t)-\ell_t(\theta_t)\right)+\frac1T\E_{\{\ell_t\}\sim\mathcal D^T}\left(\sum_{t=1}^T\ell_t(\theta_t)\right)\\
	&\le\frac1T\sum_{t=1}^T\E_{\{\ell_s\}_{s<t}\sim\mathcal D^{t-1}}\left(\E_{\ell_t'\sim\mathcal D}\ell_t'(\theta_t)-\E_{\ell_t\sim\mathcal D}\ell_t(\theta_t)\right)+\frac{\R_T}T+\frac1T\sum_{t=1}^T\E_{\ell\sim\mathcal D}\ell(\theta^\ast)\\
	&=\frac{\R_T}T+\E_{\ell\sim\mathcal D}\ell(\theta^\ast)
	\end{align*}
	where we used the fact that $\theta_t$ only depends on $\ell_1,\dots,\ell_{t-1}$.
\end{proof}

For nonnegative bounded losses we have the following fact \citep[Proposition~1]{cesa-bianchi:04}:
\begin{Prp}\label{prp:o2b}
	If a sequence of loss functions $\{\ell_t:\Theta\mapsto[0,1]\}_{t\in[T]}$ drawn i.i.d. from some distribution $\mathcal D$ is given to an online algorithm that generates a sequence of actions $\{\theta_t\in\Theta\}_{t\in[T]}$ then
	$$\frac1T\sum_{t=1}^T\E_{\ell\sim\mathcal D}\ell(\theta_t)\le\frac1T\sum_{t=1}^T\ell_t(\theta_t)+\sqrt{\frac2T\log\frac1\delta}\qquad\textrm{w.p. }1-\delta$$
	$$\frac1T\sum_{t=1}^T\E_{\ell\sim\mathcal D}\ell(\theta_t)\ge\frac1T\sum_{t=1}^T\ell_t(\theta_t)-\sqrt{\frac2T\log\frac1\delta}\qquad\textrm{w.p. }1-\delta$$
\end{Prp}
Note that \citet{cesa-bianchi:04} only prove the first inequality;
the second follows via the same argument but applying the symmetric version of the Azuma-Hoeffding inequality \cite{azuma:67}.
The inequalities above can be easily used to derive the following competitive bounds:
\begin{Cor}\label{cor:o2bwhp}
	If a sequence of loss functions $\{\ell_t:\Theta\mapsto[0,1]\}_{t\in[T]}$ drawn i.i.d. from some distribution $\mathcal D$ is given to an online algorithm with regret bound $\R_T$ that generates a sequence of actions $\{\theta_t\in\Theta\}_{t\in[T]}$ then
	$$\E_{t\sim\mathcal U[T]}\E_{\ell\sim\mathcal D}\ell(\theta_t)\le\E_{\ell\sim\mathcal D}\ell(\theta^\ast)+\frac{\R_T}T+\sqrt{\frac8T\log\frac1\delta}\qquad\textrm{w.p. }1-\delta$$
	for any $\theta^\ast\in\Theta$.
	If the losses are also convex then for $\bar\theta=\frac1T\theta_{1:T}$ we have
	$$\E_{\ell\sim\mathcal D}\ell(\bar\theta)\le\E_{\ell\sim\mathcal D}\ell(\theta^\ast)+\frac{\R_T}T+\sqrt{\frac8T\log\frac1\delta}\qquad\textrm{w.p. }1-\delta$$	
\end{Cor}
\begin{proof}
	By Proposition~\ref{prp:o2b} we have
	$$\frac1T\sum_{t=1}^T\E_{\ell\sim\mathcal D}\ell(\theta_t)
	\le\frac1T\sum_{t=1}^T\ell_t(\theta^\ast)+\frac{\R_T}T+\sqrt{\frac2T\log\frac1\delta}
	\le\E_{\ell\sim\mathcal D}\ell(\theta^\ast)+\frac{\R_T}T+\sqrt{\frac8T\log\frac1\delta}$$
	Apply linearity of expectations to get the first inequality and Jensen's inequality to get the second.
\end{proof}

We now discuss some stronger guarantees for certain classes of loss functions.
The first, due to \citet[Theorem~2]{kakade:08b}, yields faster rates for strongly convex losses:
\begin{Thm}\label{thm:o2bsc}
	Let $\mathcal D$ be some distribution over loss functions $\ell:\Theta\mapsto[0,B]$ for some $B>0$ that are  $G$-Lipschitz w.r.t. $\|\cdot\|$ for some $G>0$ and $\alpha$-strongly-convex w.r.t $\|\cdot\|$ for some $\alpha>0$.
	If a sequence of loss functions $\{\ell_t\}_{t\in[T]}$ is drawn i.i.d. from $\mathcal D$ and given to an online algorithm with regret bound $\R_T$ that generates a sequence of actions $\{\theta_t\in\Theta\}_{t\in[T]}$ then w.p. $1-\delta$ we have for $\bar\theta=\frac1T\theta_{1:T}$ and any $\theta^\ast\in\Theta$ that
	$$\E_{\ell\sim\mathcal D}\ell(\bar\theta)\le\E_{\ell\sim\mathcal D}\ell(\theta^\ast)+\frac{\R_T}T+\frac{4G}T\sqrt{\frac{\R_T}\alpha\log\frac{4\log T}\delta}+\frac{\max\{16G^2,6\alpha B\}}{\alpha T}\log\frac{4\log T}\delta$$
\end{Thm}

We can also obtain a data-dependent bound using a result of \citet{zhang:05} under a self-bounding property.
\citet[Proposition~2]{cesa-bianchi:05} show a similar but less general result.

\begin{Def}
	A distribution $\mathcal D$ over $\ell:\Theta\mapsto\mathbb R$ has {\bf $\rho$-self-bounding} losses if $\forall~\theta\in\Theta$ we have
	$$\rho\E_{\ell\sim\mathcal D}\ell(\theta)\ge\E_{\ell\sim\mathcal D}(\ell(\theta)-\E_{\ell\sim\mathcal D}\ell(\theta))^2$$
\end{Def}

\begin{Thm}\label{thm:o2bsb}
	Let $\mathcal D$ be some distribution over $\rho$-self-bounding convex loss functions $\ell:\Theta\mapsto[-1,1]$ for some $\rho>0$.
	If a sequence of loss functions $\{\ell_t\}_{t\in[T]}$ is drawn i.i.d. from $\mathcal D$ and given to an online algorithm with regret bound $\R_T$ that generates a sequence of actions $\{\theta_t\in\Theta\}_{t\in[T]}$ then w.p. $1-\delta$ we have
	$$\E_{\ell\sim\mathcal D}\ell(\bar\theta)\le \bar L_T+\sqrt{\frac{2\rho\max\{0,\bar L_T\}}T\log\frac1\delta}+\frac{3\rho+2}T\log\frac1\delta$$
	where $\bar\theta=\frac1T\theta_{1:T}$ and $\bar L_T=\frac1T\sum_{t=1}^T\ell_t(\theta_t)$ is the average loss suffered by the agent.
\end{Thm}
\begin{proof}
	Apply Jensen's inequality and \citet[Theorem~4]{zhang:05}.
\end{proof}

Note that nonnegative 1-bounded convex losses satisfy the conditions of Theorem~\ref{thm:o2bsb} with $\rho=1$.
However, we are interested in a different result that can yield a data-dependent competitive bound:
\begin{Cor}\label{cor:o2bsb}
	Let $\mathcal D$ be some distribution over convex loss functions $\ell:\Theta\mapsto\mathbb[0,1]$ such that the functions $\ell(\theta)-\ell(\theta^\ast)$ are $\rho$-self-bounded for some $\theta^\ast\in\argmin_{\theta\in\Theta}\E_{\ell\sim\mathcal D}\ell(\theta)$.
	If a sequence of loss functions $\{\ell_t\}_{t\in[T]}$ is drawn i.i.d. from $\mathcal D$ and given to an online algorithm with regret bound $\R_T$ that generates a sequence of actions $\{\theta_t\in\Theta\}_{t\in[T]}$ then w.p. $1-\delta$ we have
	$$
	\E_{\ell\sim\mathcal D}\ell(\bar\theta)
	\le\E_{\ell\sim\mathcal D}\ell(\theta^\ast)+\frac{\R_T}T+\frac1T\sqrt{2\rho\R_T\log\frac1\delta}+\frac{3\rho+2}T\log\frac1\delta
	$$
	where $\bar\theta=\frac1T\theta_{1:T}$ and $\mathcal E^\ast=\argmin_{\theta\in\Theta}\E\ell(\theta)$.
\end{Cor}
\begin{proof}
	Apply Theorem~\ref{thm:o2bsb} over the sequence of functions $\{\ell_t(\theta)-\ell_t(\theta^\ast)\}_{t\in[T]}$ and by definition of regret substitute $\bar L_T=\frac1T\sum_{t=1}^T\ell_t(\theta)-\ell_t(\theta^\ast)\le\frac{\R_T}T$.
\end{proof}
\citet[Lemma~7]{zhang:05} shows that the conditions are satisfied for $\rho=4$ by least-squares regression.

\newpage
\subsection{Dynamic Regret Guarantees}\label{subsec:dynamic}

Here we review several results for optimizing dynamic regret.
We first define this quantity:

\begin{Def}\label{def:dynamic}
	The {\bf dynamic regret} of an agent playing actions $\{\theta_t\in\Theta\}_{t\in[T]}$ on a sequence of loss functions $\{\ell_t:\Theta\mapsto\mathbb R\}$ w.r.t. a sequence of reference parameters $\Psi=\{\psi_t\}_{t\in[T]}$ is
	$$\R_T(\Psi)=\sum_{t=1}^T\ell_t(\theta_t)-\sum_{t=1}^T\ell_t(\psi_t)$$
\end{Def}

\citet[Corollary~1]{mokhtari:16} show the following guarantee for OGD over strongly convex functions:

\begin{Thm}
	Let $\{\ell_t:\Theta\mapsto\mathbb R\}_{t\in[T]}$ be a sequence of $\alpha$-strongly-convex, $\beta$-strongly-smooth, and $G$-Lipschitz functions w.r.t. $\|\cdot\|_2$. 
	Then OGD with step-size $\eta\le\frac1\beta$ achieves dynamic regret
	$$\R_T(\Psi)\le\frac{GD}{1-\rho}\left(1+\sum_{t=2}^T\|\psi_t-\psi_{t-1}\|_2\right)$$
	w.r.t. reference sequence $\Psi=\{\psi_t\}_{t\in[T]}$ for $\rho=\sqrt{1-\frac{h\alpha}\eta}$ for any $h\in(0,1]$ and $D$ the $\ell_2$-diameter of $\Theta$.
\end{Thm}

%% file: coupling.tex

\newpage
\section{Strongly Convex Coupling}\label{app:coupling}

Our first result is a simple trick that we believe may be of independent interest.
It allows us to bound the regret of FTL on any (possibly non-convex) sequence of Lipschitz functions so long as the actions played are identical to those played on a different strongly-convex sequence of Lipschitz functions.
The result is formalized in Theorem~\ref{thm:coupling}.

\subsection{Derivation}

We start with some standard facts about convex functions.
\begin{Clm}\label{clm:cdual}
	Let $f:S\mapsto\mathbb R$ be an everywhere sub-differentiable convex function.
	Then for any norm $\|\cdot\|$ we have
	$$f(x)-f(y)\le\|\nabla f(x)\|_\ast\|x-y\|~\forall~x,y\in S$$
\end{Clm}

\begin{Clm}\label{clm:scmin}
	Let $f:S\mapsto\mathbb R$ be $\alpha$-strongly-convex w.r.t. $\|\cdot\|$ with minimum $x^\ast\in\argmin_{x\in S}f(x)$.
	Then $x^\ast$ is unique and for all $x\in S$ we have
	$$f(x)\ge f(x^\ast)+\frac\alpha2\|x-x^\ast\|^2$$
\end{Clm}

Next we state some technical results, starting with the well-known be-the-leader lemma \citep[Lemma~2.1]{shalev-shwartz:11}.

\begin{Lem}\label{lem:btl}
	Let $\theta_1,\dots,\theta_{T+1}\in\Theta$ be the sequence of actions of FTL on the function sequence $\{\ell_t:\Theta\mapsto\mathbb R\}_{t\in[T]}$.
	Then
	$$\sum_{t=1}^T\ell_t(\theta_t)-\ell_t(\theta^\ast)
	\le\sum_{t=1}^T\ell_t(\theta_t)-\ell_t(\theta_{t+1})$$
	for all $\theta^\ast\in\Theta$.
\end{Lem}
The final result depends on a stability argument for FTL on strongly-convex functions adapted from \citet{saha:12}:
\begin{Lem}\label{lem:stability}
	Let $\{\ell_t:\Theta\mapsto\mathbb R\}_{t\in[T]}$ be a sequence of functions that are $\alpha_t$-strongly-convex w.r.t. $\|\cdot\|$ and let $\theta_1,\dots,\theta_{T+1}\in\Theta$ be the corresponding sequence of actions of FTL.
	Then
	$$\|\theta_t-\theta_{t+1}\|\le\frac{2\|\nabla_t\|_\ast}{\alpha_t+2\alpha_{1:t-1}}$$
	for all $t\in[T]$.
\end{Lem}
\begin{proof}
	The proof slightly generalizes an argument in \citet[Theorem~6]{saha:12}.
	For each $t\in[T]$ we have by Claim~\ref{clm:scmin} and the $\alpha_{1:t}$-strong-convexity of $\sum_{s=1}^t\ell_s(\cdot)$ that
	$$\sum_{s=1}^t\ell_s(\theta_t)
	\ge\sum_{s=1}^t\ell_s(\theta_{t+1})+\frac{\alpha_{1:t}}2\|\theta_t-\theta_{t+1}\|^2$$
	We similarly have
	$$\sum_{s=1}^{t-1}\ell_s(\theta_{t+1})
	\ge\sum_{s=1}^{t-1}\ell_s(\theta_t)+\frac{\alpha_{1:t-1}}2\|\theta_{t+1}-\theta_t\|^2$$
	Adding these two inequalities and applying Claim~\ref{clm:cdual} yields
	$$\left(\frac{\alpha_t}2+\alpha_{1:t-1}\right)\|\theta_t-\theta_{t+1}\|^2
	\le\ell_t(\theta_t)-\ell_t(\theta_{t+1})
	\le\|\nabla_t\|_\ast\|\theta_t-\theta_{t+1}\|$$
	Dividing by $\|\theta_t-\theta_{t+1}\|$ yields the result.
\end{proof}

\newpage
\begin{Thm}\label{thm:coupling}
	Let $\{\ell_t:\Theta\mapsto\mathbb R\}_{t\in[T]}$ be a sequence of functions that are $G_t$-Lipschitz in $\|\cdot\|_A$ and let $\theta_1,\dots,\theta_{T+1}$ be the sequence of actions produced by FTL.
	Let $\{\ell_t':\Theta\mapsto\mathbb R\}_{t\in[T]}$ be a sequence of functions on which FTL also plays $\theta_1,\dots,\theta_{T+1}$ but which are $G_t'$-Lipschitz and $\alpha_t$-strongly-convex in $\|\cdot\|_B$.
	Then
	$$\sum_{t=1}^T\ell_t(\theta_t)-\ell_t(\theta^\ast)
	\le2C\sum_{t=1}^T\frac{G_tG_t'}{\alpha_t+2\alpha_{1:t-1}}$$
	for all $\theta^\ast\in\Theta$ and some constant $C$ s.t. $\|\theta\|_A\le C\|\theta\|_B~\forall~\theta\in\Theta$.
	If the functions $\ell_t$ are also convex then we have
	$$\sum_{t=1}^T\ell_t(\theta_t)-\ell_t(\theta^\ast)\le2C\sum_{t=1}^T\frac{\|\nabla_t\|_{A,\ast}\|\nabla_t'\|_{B,\ast}}{\alpha_t+2\alpha_{1:t-1}}$$
	or all $\theta^\ast\in\Theta$
\end{Thm}
\begin{proof}
	By Lemma~\ref{lem:stability},
	$$\|\theta_t-\theta_{t+1}\|_A\le C\|\theta_t-\theta_{t+1}\|_B
	\le\frac{2CG_t'}{\alpha_t+2\alpha_{1:t-1}}$$
	for all $t\in[T]$.
	Then by Lemma~\ref{lem:btl} and the $G_t$-Lipschitzness of $\ell_t$ we have for all $\theta^\ast\in\Theta$ that
	$$\sum_{t=1}^T\ell_t(\theta_t)-\ell_t(\theta^\ast)
	\le\sum_{t=1}^T\ell_t(\theta_t)-\ell(\theta_{t+1})
	\le\sum_{t=1}^TG_t\|\theta_t-\theta_{t+1}\|_A
	\le2C\sum_{t=1}^T\frac{G_tG_t'}{\alpha_t+2\alpha_{1:t-1}}$$
	In the convex case we instead apply Claim~\ref{clm:cdual} and Lemma~\ref{lem:stability} to get
	$$\sum_{t=1}^T\ell_t(\theta_t)-\ell_t(\theta^\ast)
	\le\sum_{t=1}^T\ell_t(\theta_t)-\ell(\theta_{t+1})
	\le\sum_{t=1}^T\|\nabla_t\|_{A,\ast}\|\theta_t-\theta_{t+1}\|_A
	\le2C\sum_{t=1}^T\frac{\|\nabla_t\|_{A,\ast}\|\nabla_t'\|_{B,\ast}}{\alpha_t+2\alpha_{1:t-1}}$$
\end{proof}

\subsection{Applications}

We now show two applications of strongly convex coupling.
The first shows logarithmic regret for FTL run on a sequence of Bregman regularizers.
Note that these functions are nonconvex in general.
\begin{Prp}\label{prp:bregman}
	Let $R:\Theta\mapsto\mathbb R$ be 1-strongly-convex w.r.t. $\|\cdot\|$ and consider any $\theta_1,\dots,\theta_T\in\Theta$.
	Then when run on the loss sequence $\alpha_1\Breg_R(\theta_1||\cdot),\dots,\alpha_T\Breg_R(\theta_T||\cdot)$ for any positive scalars $\alpha_1,\dots,\alpha_T\in\mathbb R_+$, FTL obtains regret
	$$\R_T\le2CD\sum_{t=1}^T\frac{\alpha_t^2G_t}{\alpha_t+2\alpha_{1:t-1}}$$
	for $C$ s.t. $\|\theta\|\le C\|\theta\|_2~\forall~\theta\in\Theta$, $D=\max_{\theta,\phi\in\Theta}\|\theta-\phi\|_2$ the $\ell_2$-diameter of $\Theta$, and $G_t$ the Lipschitz constant of $\Breg_R(\theta_t||\cdot)$ over $\Theta$ w.r.t. $\|\cdot\|$.
	Note that for $\|\cdot\|=\|\cdot\|_2$ we have $C=1$ and $G_t\le D~\forall~t\in[T]$.
\end{Prp}
\begin{proof}
	Note that $\alpha_t\Breg_R(\theta_t||\cdot)$ is $\alpha_tG_t$-Lipschitz w.r.t. $\|\cdot\|$.
	Let $R'(\cdot)=\frac12\|\cdot\|_2^2$, so  $\Breg_{R'}(\theta_t||\phi)=\frac12\|\theta_t-\phi\|_2^2~\forall~\phi\in\Theta,t\in[T]$.
	The function $\alpha_t\Breg_{R'}(\theta_t||\cdot)$ is thus $\alpha_t$-strongly-convex and $D$-Lipschitz w.r.t. $\|\cdot\|_2$.
	Now by Claim~\ref{clm:mean} FTL run on this new sequence plays the same actions as FTL run on the original sequence.
	Applying Theorem~\ref{thm:coupling} yields the result.
\end{proof}

\newpage
In the next application we use coupling to give a $\tilde{\mathcal O}(T^\frac35)$-regret algorithm for a sequence of non-Lipschitz convex functions.

\begin{Prp}\label{prp:ftl}
	Let $\{\ell_t:\mathbb R_+\mapsto\mathbb R\}_{t\ge1}$ be a sequence of functions of form $\ell_t(x)=\left(\frac{B_t^2}x+x\right)\alpha_t$ for any positive scalars $\alpha_1,\dots,\alpha_T\in\mathbb R_+$ and adversarially chosen $B_t\in[0,D]$.
	Then the $\varepsilon$-FTL algorithm, which for $\varepsilon>0$ uses the actions of FTL run on the functions $\tilde\ell_t(x)=\left(\frac{B_t^2+\varepsilon^2}x+x\right)\alpha_t$ over the domain $[\varepsilon,\sqrt{D^2+\varepsilon^2}]$ to determine $x_t$, achieves regret
	$$\R_T\le\min\left\{\frac{\varepsilon^2}{x^\ast},\varepsilon\right\}\alpha_{1:T}+2D\max\left\{\frac{D^3}{\varepsilon^3},1\right\}\sum_{t=1}^T\frac{\alpha_t^2}{\alpha_t+2\alpha_{1:t-1}}$$
	for all $x^\ast>0$.
\end{Prp}
\begin{proof}
	Define $\tilde B_t^2=B_t^2+\varepsilon^2$ and note that FTL run on the functions $\tilde\ell_t'(x)=\left(\frac{x^2}2-\tilde B_t^2\log x\right)\alpha_t$ plays the exact same actions $x_t^2=\frac{\sum_{s<t}\alpha_s\tilde B_s^2}{\alpha_{1:t-1}}$ as FTL run on $\tilde\ell_t$.
	We have that
	$$|\partial_x\tilde\ell_t|
	=\alpha_t\left|1-\frac{\tilde B_t^2}{x^2}\right|
	\le\frac{\alpha_tD^2}{\varepsilon^2}$$
	$$|\partial_x\tilde\ell_t'|
	=\alpha_t\left|x-\frac{\tilde B_t^2}x\right|
	\le\alpha_t\max\left\{D,\frac{D^2}\varepsilon\right\}
	\qquad
	\partial_{xx}\tilde\ell_t'
	=\alpha_t\left(1+\frac{\tilde B_t^2}{x^2}\right)
	\ge\alpha_t$$
	so the functions $\tilde\ell_t$ are $\frac{\alpha_tD^2}{\varepsilon^2}$-Lipschitz while the functions $\tilde\ell_t'$ are $\alpha_tD\max\left\{\frac D\varepsilon,1\right\}$-Lipschitz and $\alpha_t$-strongly-convex.
	Therefore by Theorem~\ref{thm:coupling} we have that
	$$\sum_{t=1}^T\tilde\ell_t(x_t)-\tilde\ell_t(x^\ast)
	\le2D\max\left\{\frac{D^3}{\varepsilon^3},1\right\}\sum_{t=1}^T\frac{\alpha_t^2}{\alpha_t+2\alpha_{1:t-1}}$$
	for any $x^\ast\in[\varepsilon,\sqrt{D^2+\varepsilon^2}]$.
	Since $\sum_{t=1}^T\tilde\ell_t$ is minimized on $[\varepsilon,\sqrt{D^2+\varepsilon^2}]$, the above also holds for all $x^\ast>0$.
	Therefore we have that
	\begin{align*}
	\sum_{t=1}^T\ell_t(x_t)
	&\le\sum_{t=1}^T\left(\frac{B_t^2+\varepsilon^2}{x_t}+x_t\right)\alpha_t\\
	&=\sum_{t=1}^T\tilde\ell_t(x_t)\\
	&\le\min_{x^\ast>0}2D\max\left\{\frac{D^3}{\varepsilon^3},1\right\}\sum_{t=1}^T\frac{\alpha_t^2}{\alpha_t+2\alpha_{1:t-1}}+\sum_{t=1}^T\tilde\ell_t(x^\ast)\\
	&=\min_{x^\ast>0}2D\max\left\{\frac{D^3}{\varepsilon^3},1\right\}\sum_{t=1}^T\frac{\alpha_t^2}{\alpha_t+2\alpha_{1:t-1}}+\sum_{t=1}^T\left(\frac{B_t^2+\varepsilon^2}{x^\ast}+x^\ast\right)\alpha_t\\
	&=\min_{x^\ast>0}\frac{\varepsilon^2}{x^\ast}\alpha_{1:T}+2D\max\left\{\frac{D^3}{\varepsilon^3},1\right\}\sum_{t=1}^T\frac{\alpha_t^2}{\alpha_t+2\alpha_{1:t-1}}+\sum_{t=1}^T\ell_t(x^\ast)
	\end{align*}
	Note that substituting $x^\ast=\sqrt{\frac{\sum_{t=1}^T\alpha_t\tilde B_t^2}{\alpha_{1:T}}}$ into the second-to-last line yields
	$$\min_{x^\ast>0}\sum_{t=1}^T\left(\frac{B_t^2+\varepsilon^2}{x^\ast}+x^\ast\right)\alpha_t
	\le2\sqrt{\alpha_{1:T}\sum_{t=1}^T\alpha_t\tilde B_t^2}
	\le2\varepsilon\alpha_{1:T}+\min_{x^\ast>0}\sum_{t=1}^T\ell_t(x^\ast)$$
	completing the proof.
\end{proof}

%% file: dynamic.tex

\newpage
\section{Adaptive and Dynamic Guarantees}\label{sec:app:adaptive}
Throughout Appendices~\ref{sec:app:adaptive},~\ref{sec:app:geo}, and~\ref{sec:app:o2b} we assume that $\argmin_{\theta\in\Theta}\sum_{\ell\in\mathcal S}\ell(\theta)$ returns a unique minimizer of the sum of the loss functions in the sequence $\mathcal S$.
Formally, this can be defined to be the one minimizing an appropriate Bregman divergence $\Breg_R(\cdot|\phi_R)$ from some fixed $\phi_R\in\Theta$, e.g. the origin in Euclidean space or the uniform distribution over the simplex, which is unique by strong-convexity of $\Breg_R(\cdot|\phi_R)$ and convexity of the set of optimizers of a convex function.

\begin{Thm}\label{thm:app:general}
	Let each task $t\in[T]$ consist of a sequence of $m_t$ convex loss functions $\ell_{t,i}:\Theta\mapsto\mathbb R$ that are $G_{t,i}$-Lipschitz w.r.t. $\|\cdot\|$.
	For $G_t^2=G_{1:m_t}^2/m_t$ and $R:\Theta\mapsto\mathbb R$ a 1-strongly-convex function w.r.t. $\|\cdot\|$ define the following online algorithms:
	\begin{enumerate}
		\item $\DYN$: a method that has dynamic regret $\Rinit_T(\Psi)=\sum_{t=1}^T\finit_t(\phi_t)-\finit_t(\psi_t)$ w.r.t. reference actions $\Psi=\{\psi_t\}_{t=1}^T\subset\Theta$ over the sequence $\finit_t(\cdot)=\Breg_R(\theta_t^\ast||\cdot)G_t\sqrt{m_t}$ .
		\item $\SIM$: a method that has (static) regret $\Rsim_T(x)$ decreasing in $x>0$ over the sequence of functions $\fsim_t(x)=\left(\frac{\Breg_R(\theta_t^\ast||\phi_t)}x+x\right)G_t\sqrt{m_t}$.
	\end{enumerate}
	Then if Algorithm~\ref{alg:general} sets $\phi_t=\DYN(t)$ and $\eta_t=\frac{\SIM(t)}{G_t\sqrt{m_t}}$ it will achieve 
	$$\TAR_T
	\le\RUB_T
	\le\frac{\Rsim_T(V_\Psi)}T+\frac1T\min\left\{\frac{\Rinit_T(\Psi)}{V_\Psi},2\sqrt{\Rinit_T(\Psi)\sum_{t=1}^TG_t\sqrt{m_t}}\right\}+\frac{2V_\Psi}T\sum_{t=1}^TG_t\sqrt{m_t}$$
	for $V_\Psi^2=\frac1{\sum_{t=1}^TG_t\sqrt{m_t}}\sum_{t=1}^T\Breg_R(\theta_t^\ast||\psi_t)G_t\sqrt{m_t}$.
\end{Thm}

\begin{proof}
	Letting $x_t=\SIM(t)$ be the output of $\SIM$ at time $t$, defining $\sigma_t=G_t\sqrt{m_t}$ and $\sigma_{1:T}=\sum_{t=1}^T\sigma_t$, and substituting into the regret-upper-bound of OMD/FTRL \eqref{eq:regret}, we have that
	\begin{align*}
	\RUB_TT
	\hspace{-0.5mm}=\hspace{-0.5mm}\sum_{t=1}^T\left(\frac{\Breg_R(\theta_t^\ast||\phi_t)}{x_t}+x_t\right)\sigma_t
	&\le\min_{x>0}\Rsim_T(x)+\sum_{t=1}^T\left(\frac{\Breg_R(\theta_t^\ast||\phi_t)}x+x\right)\sigma_t\\
	&\le\min_{x>0}\Rsim_T(x)+\frac{\Rinit_T(\Psi)}x+\sum_{t=1}^T\left(\frac{\Breg_R(\theta_t^\ast||\psi_t)}x+x\right)\sigma_t\\
	&\le\Rsim_T(V_\Psi)\hspace{-0.25mm}+\hspace{-0.25mm}\min\left\{\frac{\Rinit_T(\Psi)}{V_\Psi},2\sr{\Rinit_T(\Psi)\sigma_{1:T}}\right\}\hspace{-0.25mm}+\hspace{-0.25mm}2V_\Psi\sigma_{1:T}
	\end{align*}
	where the last line follows by substituting $x=\max\left\{V_\Psi,\sqrt{\frac{\Rinit_T(\Psi)}{\sigma_{1:T}}}\right\}$.
\end{proof}

\begin{Cor}\label{cor:epsftl}
	Under the assumptions of Theorem~\ref{thm:app:general} and boundedness of $\Breg_R$ over $\Theta$, if $\DYN$ uses FTL, or AOGD in the case of $R(\cdot)=\frac12\|\cdot|_2^2$, and $\SIM$ uses $\varepsilon$-FTL as defined in Proposition~\ref{prp:ftl}, then Algorithm~\ref{alg:general} achieves
	$$\RUB_TT
	\le\min\left\{\frac{\varepsilon^2}V,\varepsilon\right\}\sigma_{1:T}+2D\max\left\{\frac{D^3}{\varepsilon^3},1\right\}\sum_{t=1}^T\frac{\sigma_t^2}{\sigma_{1:t}}+\sqrt{8CD\sigma_{1:T}\sum_{t=1}^T\frac{\sigma_t^2}{\sigma_{1:t}}}+2V\sigma_{1:T}$$
	for $V^2=\min_{\phi\in\Theta}\sum_{t=1}^T\sigma_t\Breg_R(\theta_t^\ast||\phi)$ and constant $C$ the product of the constant $C$ from Proposition~\ref{prp:bregman} and the bound on the gradient of the Bregman divergence.
	Assuming $\sigma_t=G\sqrt m~\forall~t$ and substituting $\varepsilon=\frac1{\sqrt[5]T}$ yields
	$$\TAR_T\le\RUB_T=\tilde{\mathcal O}\left(\min\left\{\frac1{VT^\frac25}+\frac1{\sqrt T},\frac1{\sqrt[5]T}\right\}+V\right)\sqrt m$$
\end{Cor}
\begin{proof}
	Substitute Propositions~\ref{prp:bregman} and~\ref{prp:ftl} into Theorem~\ref{thm:app:general}.
\end{proof}

\newpage
\begin{Prp}\label{prp:expconcave}
	Let $\{\ell_t:\mathbb R_+\mapsto\mathbb R\}_{t\ge1}$ be a sequence of functions of form $\ell_t(x)=\left(\frac{B_t^2}x+x\right)\alpha_t$ for any positive scalars $\alpha_1,\dots,\alpha_T\in\mathbb R_+$ and adversarially chosen $B_t\in[0,D]$.
	Then the losses $\tilde\ell_t(x)=\left(\frac{B_t^2+\varepsilon^2}x+x\right)\alpha_t$ over the domain $[\varepsilon,\sqrt{D^2+\varepsilon^2}]$ are $\frac{\alpha_tD^2}{\varepsilon^2}$-Lipschitz and $\frac2{\alpha_tD}\min\left\{\frac{\varepsilon^2}{D^2},1\right\}$-exp-concave.
\end{Prp}
\begin{proof}
	Lipschitzness follows by taking derivatives as in Proposition~\ref{prp:ftl}.
	Define $\tilde B_t^2=B_t^2+\varepsilon^2$.
	We then have
	$$\partial_x\tilde\ell_t=\alpha_t\left(1-\frac{\tilde B_t^2}{x^2}\right)\qquad\qquad\partial_{xx}\tilde\ell_t=\frac{2\alpha_t\tilde B_t^2}{x^3}$$
	The $\gamma$-exp-concavity of the functions $\tilde\ell_t$ can be determined by finding the largest $\gamma$ satisfying
	$$\gamma
	\le\frac{\partial_{xx}\tilde\ell_t}{(\partial_x\tilde\ell_t)^2}
	=\frac{2\tilde B_t^2x}{\alpha_t(\tilde B_t^2-x^2)^2}$$
	for all $x\in[\varepsilon,\sqrt{D^2+\varepsilon^2}]$ and all $t\in[T]$.
	We first minimize jointly over choice of $x,\tilde B_t\in[\varepsilon,\sqrt{D^2+\varepsilon^2}]$.
	The derivatives of the objective w.r.t. $x$ and $\tilde B_t$, respectively, are
	$$\frac{2\tilde B_t^2(\tilde B_t^2+3x^2)}{(\tilde B_t^2-x^2)^3}\qquad\qquad-\frac{4\tilde B_tx(\tilde B_t^2+x^2)}{(\tilde B_t^2-x^2)^3}$$
	Note that the objective approaches $\infty$ as the coordinates approach the line $x=\tilde B_t$.
	For $x<\tilde B_t$ the derivative w.r.t. $x$ is always positive while the derivative w.r.t. $\tilde B_t$ is always negative.
	Since we have the constraints $x\ge\varepsilon$ and $\tilde B_t^2\le D^2+\varepsilon^2$, the optimum over $x<\tilde B_t$ is thus attained at $x=\varepsilon$ and $\tilde B_t^2=D^2+\varepsilon^2$.
	Substituting into the original objective yields
	$$\frac{2(D^2+\varepsilon^2)\varepsilon}{\alpha_tD^4}\ge\frac{2\varepsilon}{\alpha_tD^2}$$
	For $x>\tilde B_t$ the derivative w.r.t. $x$ is always negative while the derivative w.r.t. $\tilde B_t$ is always positive.
	Since we have the constraints $x\le\sqrt{D^2+\varepsilon^2}$ and $\tilde B_t^2\ge\varepsilon^2$, the optimum over $x>\tilde B_t$ is thus attained at $x=\sqrt{D^2+\varepsilon^2}$ and $\tilde B_t^2=\varepsilon^2$.
	Substituting into the original objective yields
	$$\frac{2\varepsilon^2\sqrt{D^2+\varepsilon^2}}{\alpha_tD^4}\ge\frac{2\varepsilon^2}{\alpha_tD^3}$$
	Thus we have that the functions $\tilde\ell_t$ are $\frac2{\alpha_tD}\min\left\{\frac{\varepsilon^2}{D^2},1\right\}$-exp-concave.
\end{proof}

\begin{Cor}\label{cor:ewoo}
	Let $\{\ell_t:\mathbb R_+\mapsto\mathbb R\}_{t\ge1}$ be a sequence of functions of form $\ell_t(x)=\left(\frac{B_t^2}x+x\right)\alpha_t$ for any positive scalars $\alpha_1,\dots,\alpha_T\in\mathbb R_+$ and adversarially chosen $B_t\in[0,D]$.
	Then the $\varepsilon$-EWOO algorithm, which for $\varepsilon>0$ uses the actions of EWOO run on the functions $\tilde\ell_t(x)=\left(\frac{B_t^2+\varepsilon^2}x+x\right)\alpha_t$ over the domain $[\varepsilon,\sqrt{D^2+\varepsilon^2}]$ to determine $x_t$, achieves regret
	$$\R_T\le\min_{x^\ast>0}\left\{\frac{\varepsilon^2}{x^\ast},\varepsilon\right\}\alpha_{1:T}+\frac{D\alpha_{\max}}2\max\left\{\frac{D^2}{\varepsilon^2},1\right\}(1+\log(T+1))$$
	for all $x^\ast>0$.
\end{Cor}
\begin{proof}
	Since $\sum_{t=1}^T\tilde\ell_t$ is minimized on $[\varepsilon,\sqrt{D^2+\varepsilon^2}]$, we apply Theorem~\ref{thm:ewoo} and follow a similar argument to that concluding Proposition~\ref{prp:ftl} to get
	\begin{align*}
	\sum_{t=1}^T\ell_t(x_t)
	&\le\frac{D\alpha_{\max}}2\max\left\{\frac{D^2}{\varepsilon^2},1\right\}(1+\log(T+1))+\sum_{t=1}^T\tilde\ell_t(x^\ast)\\
	&=\min_{x^\ast>0}\left\{\frac{\varepsilon^2}{x^\ast},\varepsilon\right\}\alpha_{1:T}+\frac{D\alpha_{\max}}2\max\left\{\frac{D^2}{\varepsilon^2},1\right\}(1+\log(T+1))+\sum_{t=1}^T\ell_t(x^\ast)
	\end{align*}
\end{proof}

\begin{Cor}\label{cor:epsewoo}
	Under the assumptions of Theorem~\ref{thm:app:general} and boundedness of $\Breg_R$ over $\Theta$, if $\DYN$ uses FTL, or AOGD in the case of $R(\cdot)=\frac12\|\cdot\|_2^2$, and $\SIM$ uses $\varepsilon$-EWOO as defined in Proposition~\ref{cor:ewoo}, then Algorithm~\ref{alg:general} achieves
	$$\RUB_TT
	\le\min\left\{\frac{\varepsilon^2}V,\varepsilon\right\}\sigma_{1:T}+\frac{D\sigma_{\max}}2\max\left\{\frac{D^2}{\varepsilon^2},1\right\}(1+\log(T+1))+\sr{8CD\sigma_{1:T}\sum_{t=1}^T\frac{\sigma_t^2}{\sigma_{1:t}}}+2V\sigma_{1:T}$$
	for $V^2=\min_{\phi\in\Theta}\sum_{t=1}^T\sigma_t\Breg_R(\theta_t^\ast||\phi)$ and constant $C$ the product of the constant $C$ from Proposition~\ref{prp:bregman} and the bound on the gradient of the Bregman divergence.
	Assuming $\sigma_t=G\sqrt m~\forall~t$ and substituting $\varepsilon=\frac1{\sqrt[4]T}$ yields
	$$\TAR_T\le\RUB_T=\tilde{\mathcal O}\left(\min\left\{\frac{1+\frac1V}{\sqrt T},\frac1{\sqrt[4]T}\right\}+V\right)\sqrt m$$
\end{Cor}
\begin{proof}
	Substitute Proposition~\ref{prp:bregman} and Corollary~\ref{cor:ewoo} into Theorem~\ref{thm:app:general}.
\end{proof}

\begin{Cor}\label{cor:dynamic}
	Under the assumptions of Theorem~\ref{thm:general} and boundedness of $\Theta$, if $\DYN$ is OGD with learning rate $\frac1{\sigma_{\max}}$ and $\SIM$ uses $\varepsilon$-EWOO as defined in Proposition~\ref{cor:ewoo} then Algorithm~\ref{alg:general} achieves
	\begin{align*}
	\RUB_TT
	&\le\min\left\{\frac{\varepsilon^2}{V_\Psi},\varepsilon\right\}\sigma_{1:T}+\frac{D\sigma_{\max}}2\max\left\{\frac{D^2}{\varepsilon^2},1\right\}(1+\log(T+1))\\
	&\quad+2D\min\left\{\frac{D\sigma_{\max}}{V_\Psi}(1+P_\Psi),\sqrt{2\sigma_{\max}\sigma_{1:T}(1+P_\Psi)}\right\}+2V_\Psi\sigma_{1:T}
	\end{align*}
	for $P_T(\Psi)=\sum_{t=2}^T\|\psi_t-\psi_{t-1}\|_2$.
	Assuming $\sigma_t=G\sqrt m~\forall~t$ and substituting $\varepsilon=\frac1{\sqrt[4]T}$ yields
	$$\TAR_T\le\RUB_T=\tilde{\mathcal O}\left(\min\left\{\frac{1+\frac1{V_\Psi}}{\sqrt T},\frac1{\sqrt[4]T}\right\}+\min\left\{\frac{1+P_\Psi}{V_\Psi T},\sqrt{\frac{1+P_\Psi}T}\right\}+V_\Psi\right)\sqrt m$$
\end{Cor}
\begin{proof}
	Substitute Theorem~\ref{thm:dynamic} and Corollary~\ref{cor:ewoo} into Theorem~\ref{thm:app:general}.
\end{proof}

%% file: geometry.tex

\newpage
\section{Adapting to the Inter-Task Geometry}\label{sec:app:geo}

For clarity, vectors and matrices in this section will be {\bf bolded}, although scalar regret quantities will continue to be as well.
For any two vectors $\*x,\*y\in\mathbb R^d$, $\*x\odot\*y$ will denote element-wise multiplication, $\frac{\*x}{\*y}$ will denote element-wise division, $\*x^p$ will denote raising each element of $\*x$ to the power $p$, and $\max\{\*x,\*y\}$ and $\min\{\*x,\*y\}$ will denote element-wise maximum and minimum, respectively.
For any nonnegative $\*a\in\mathbb R^d$ we will use the notation $\|\*\cdot\|_{\*a}=\langle\sqrt{\*a},\*\cdot\rangle$;
note that if all elements of $\*a$ are positive then $\|\*\cdot\|_{\*a}$ is a norm on $\mathbb R^d$ with dual norm $\|\*\cdot\|_{\*a^{-1}}$.

\begin{Clm}
	For $t\ge1$ and $p\in(0,1)$ we have
	$$\sum_{s=0}^{t-1}\frac1{(s+1)^p}\ge\sum_{s=1}^t\frac1{(s+1)^p}\ge\underline c_pt^{1-p}\qquad\textrm{and}\qquad\sum_{s=1}^t\frac1{s^p}\le\overline c_pt^{1-p}$$
	for $\underbar c_p=\frac{1-\left(\frac23\right)^{1-p}}{1-p}$ and $\overline c_p=\frac1{1-p}$.
\end{Clm}
\begin{proof}
	$$\sum_{s=0}^{t-1}\frac1{(s+1)^p}
	\ge\sum_{s=1}^t\frac1{(s+1)^p}
	\ge\int_1^{t+1}\frac{ds}{(s+1)^p}
	=\frac{(t+2)^{1-p}-2^{1-p}}{1-p}
	\ge\underline c_p(t+2)^{1-p}
	\ge\underline c_pt^{1-p}$$
	$$\sum_{s=1}^t\frac1{s^p}
	\le1+\int_1^t\frac{ds}{s^p}
	=1+\frac{t^{1-p}-1}{1-p}
	\le\overline c_pt^{1-p}$$
\end{proof}

\begin{Clm}
	For any $\*x\in\mathbb R^d$ we have $\|\*x^2\|_2^2\le\|\*x\|_2^4$.
\end{Clm}
\begin{proof}
	$$\|\*x^2\|_2^2
	=\sum_{j=1}^dx_j^4
	\le\left(\sum_{j=1}^dx_j^2\right)^2
	=\|\*x\|_2^4$$
\end{proof}

We now review some facts from matrix analysis.
Throughout this section we will use matrices in $\mathbb R^{d\times d}$;
we denote the subset of symmetric matrices by $\mathbb S^d$, the subset of symmetric PSD matrices by $\mathbb S_+^d$, and the subset of symmetric positive-definite matrices by $\mathbb S_{++}^d$.
Note that every symmetric matrix $\*A\in\mathbb S^d$ has diagonalization $\*A=\*V\*\Lambda\*V^{-1}$ for diagonal matrix $\*\Lambda\in\mathbb S^d$ containing the eigenvalues of $\*A$ along the diagonal and a matrix $\*V\in\mathbb R^{d\times d}$ of orthogonal eigenvectors.
For such matrices we will use $\lambda_j(\*A)$ to denote the $j$th largest eigenvalue of $\*A$ and for any function $f:[\lambda_d(\*A),\lambda_1(\*A)]\mapsto\mathbb R$ we will use the notation
$$f(\*A)=\*V\begin{pmatrix}f(\*\Lambda_{11})&&\\&\ddots&\\&&f(\*\Lambda_{dd})\end{pmatrix}\*V^{-1}$$
We will denote the spectral norm by $\|\cdot\|_2$ and the Frobenius norm by $\|\cdot\|_F$.
\begin{Clm}\label{clm:logdetgrad}\citep[Section~A.4.1]{boyd:04}
	$f(\*X)=\log\det\*X$ has gradient $\nabla_{\*X}f=\*X^{-1}$ over $\mathbb S_{++}^d$
\end{Clm}
\begin{Clm}\label{clm:logdetsc}\citep[Theorem~3.1]{moridomi:18}
	The function $f(\*X)=-\log\det\*X$ is $\frac1{\sigma^2}$-strongly-convex w.r.t. $\|\cdot\|_2$ over the set of symmetric positive-definite matrices with spectral norm bounded by $\sigma$.
\end{Clm}
\newpage
\begin{Def}\label{def:opconvex}
	A function $f:(0,\infty)\mapsto\mathbb R$ is {\bf operator convex} if $\forall~\*X,\*Y\in\mathbb S_{++}^d$ and any $t\in[0,1]$ we have
	$$f(t\*X+(1-t)\*Y)\preceq tf(\*X)+(1-t)f(\*Y)$$
\end{Def}
\begin{Clm}\label{clm:opconvex}
	If $\*A\in\mathbb S_+^d$ and $f:(0,\infty)\mapsto\mathbb R$ is operator convex then $\Tr(\*Af(\*X))$ is convex on $\mathbb S_{++}^d$.
\end{Clm}
\begin{proof}
	Consider any $\*X,\*Y\in\mathbb S_{++}^d$ and any $t\in[0,1]$.
	By the operator convexity of $f$, positive semi-definiteness of $\*A$, and linearity of the trace functional we have that
	\begin{align*}
	0
	&\preceq\Tr(\*A(tf(\*X)+(1-t)f(\*Y)-f(t\*X+(1-t)\*Y)))\\
	&=t\Tr(\*A(f(\*X)))+(1-t)\Tr(\*Af(\*Y))-\Tr(\*A(f(t\*X+(1-t)\*Y)))
	\end{align*}
\end{proof}
\begin{Cor}\label{cor:opconvex}
	If $\*A\in\mathbb S_+^d$ then $\Tr(\*A\*X^{-1})$ and $\Tr(\*A\*X)$ are convex over $\mathbb S_{++}^d$.
\end{Cor}
\begin{proof}
	By the L\"{o}wner-Heinz theorem \cite{davis:63}, $x^{-1},x,$ and $x^2$ are operator convex.
	The result follows by applying Claim~\ref{clm:opconvex}.
\end{proof}

\begin{Cor}\label{cor:trconvex}\citep[Corollary~1.1]{lieb:73}
	If $\*A,\*B\in\mathbb S_+^d$ then $\Tr(\*A\*X\*B\*X)$ is convex over $\mathbb S_+^d$.
\end{Cor}

\newpage
\begin{Prp}\label{prp:diagftl}
	Let $\{\ell_t:\mathbb R_+\mapsto\mathbb R\}_{t\ge1}$ be of form $\ell_t(\*x)=\left\|\frac{\*b_t^2}{\*x}+\*g_t^2\odot\*x\right\|_1$ for adversarially chosen $\*b_t,\*g_t$ satisfying $\|\*b_t\|_2\le D,\|\*g_t\|_2\le G$.
	Then the $(\varepsilon,\zeta,p)$-FTL algorithm, which for $\varepsilon,\zeta>0$ and $p\in(0,\frac23)$ uses the actions of FTL run on the functions $\tilde\ell_t(\*x)=\left\|\frac{\*b_t^2+\varepsilon_t^2\*1_d}{\*x}+(\*g_t^2+\zeta_t^2\*1_d)\odot\*x\right\|_1$, where $\varepsilon_t^2=\varepsilon^2(t+1)^{-p},\zeta_t^2=\zeta^2(t+1)^{-p}$ for $t\ge0$ and $\*b_0=\*g_0=\*0_d$, to determine $\*x_t$, has regret 
	\begin{align*}
	\R_T
	&\le C_p\sum_{j=1}^d\min\left\{\left(\frac{\varepsilon^2}{\*x_j^\ast}+\zeta^2\*x_j^\ast\right)T^{1-p},\sqrt{\zeta^2\*b_{j,1:T}^2+\varepsilon^2\*g_{j,1:T}^2}T^\frac{1-p}2+2\varepsilon\zeta T^{1-p}\right\}\\
	&\qquad+C_p\left(\frac{D+\varepsilon}{\zeta^3}G^4+\frac{G+\zeta}{\varepsilon^3}D^4\right)T^{\frac32p}+C_p(D\zeta+G\varepsilon+\varepsilon\zeta)d
	\end{align*}
	for any $\*x>0$ and some constant $C_p$ depending only on $p$.
\end{Prp}
\begin{proof}
	Define $\*{\tilde b}_t^2=\*b_t^2+\varepsilon_t^2\*1_d,\*{\tilde g}_t^2=\*g_t^2+\zeta_t^2\*1_d$ and note that FTL run on the modified functions $\tilde\ell_t'(\*x)=\left\|\frac{\*{\tilde g}_t^2\odot\*x^2}2-\*{\tilde b}_t^2\odot\log(\*x)\right\|_1$ plays the exact same actions $\*x_t^2=\frac{\*{\tilde b}_{0:t-1}^2}{\*{\tilde g}_{0:t-1}^2}$ as FTL run $\tilde\ell_t$.
	Since both sequences of loss functions are separable across coordinates, we consider $d$ per-coordinate problems, with loss functions of form $\tilde\ell_t(x)=\frac{\tilde b_t^2}x+\tilde g_t^2x$ and $\tilde\ell_t'(x)=\frac{\tilde g_t^2x^2}2-\tilde b_t^2\log x$.
	We have that
	$$|\nabla_t|=\left|\tilde g_t^2-\frac{\tilde b_t^2}{x_t^2}\right|=\frac{|\tilde g_t^2x_t^2-\tilde b_t^2|}{x_t^2}
	\quad|\nabla_t'|=\left|\tilde g_t^2x_t-\frac{\tilde b_t^2}{x_t}\right|=\frac{|\tilde g_t^2x_t^2-\tilde b_t^2|}{x_t}
	\quad\partial_{xx}\tilde\ell_t'=\tilde g_t^2+\frac{\tilde b_t^2}{x^2}\ge\tilde g_t^2$$
	so by Theorem~\ref{thm:coupling} and substituting the action $x_t^2=\frac{\tilde b_{0:t-1}^2}{\tilde g_{0:t-1}^2}$ we have per-coordinate regret
	\begin{align*}
	\sum_{t=1}^T\tilde\ell_t(x_t)-\tilde\ell_t(x^\ast)
	\le2\sum_{t=1}^T\frac{|\nabla_t||\nabla_t'|}{\tilde g_{1:t}^2}
	&=2\sum_{t=1}^T\frac{|\tilde g_t^2x_t^2-\tilde b_t^2|^2}{x_t^3\tilde g_{1:t}^2}\\
	&\le2\sum_{t=1}^T\frac{\tilde g_t^4x_t}{\tilde g_{1:t}^2}+\frac{\tilde b_t^4}{x_t^3\tilde g_{1:t}^2}\\
	&=2\sum_{t=1}^T\frac{\tilde g_t^4\sqrt{\tilde b_{0:t-1}^2}}{\tilde g_{1:t}^2\sqrt{\tilde g_{0:t-1}^2}}+\frac{\tilde b_t^4}{\tilde g_{1:t}^2\left(\frac{\tilde b_{0:t-1}^2}{\tilde g_{0:t-1}^2}\right)^\frac32}\\
	&\le2\sum_{t=1}^T\frac{\tilde g_t^4\sqrt{\tilde b_{0:t-1}^2}}{\tilde g_{1:t}^2\sqrt{\tilde g_{0:t-1}^2}}+\frac{\tilde b_t^4\sqrt{2\tilde g_{1:t}^2}}{(\tilde b_{0:t-1}^2)^\frac32}+\frac{\tilde b_t^4\tilde g_0^3\sqrt2}{\tilde g_{1:t}^2(\tilde b_{0:t-1}^2)^\frac32}
	\end{align*}
	Taking the summation over the coordinates yields
	\begin{align*}
	\sum_{t=1}^T&\tilde\ell_t(\*x_t)-\tilde\ell_t(\*x^\ast)\\
	&\le4\sum_{t=1}^T\left(\frac{(D+\varepsilon)(\|\*g_t^2\|_2^2+\zeta_t^4d)}{\zeta_{1:t}^2\sqrt{2\zeta_{0:t-1}^2}}+\frac{(G+\zeta)(\|\*b_t^2\|_2^2+\varepsilon_t^4d)}{(\varepsilon_{0:t-1}^2)^\frac32}+\frac{(\|\*b_t^2\|_2^2+\varepsilon_t^4d)\zeta^3}{\tilde\zeta_{0:t-1}^2(\tilde\varepsilon_{0:t-1}^2)^\frac32}\right)\sqrt{2t}\\
	&\le4\sum_{t=1}^T\left(\frac{(D+\varepsilon)(G^4+\zeta_t^4d)}{(\underbar c_p\zeta^2t^{1-p})^\frac32\sqrt2}+\frac{(G+\zeta)(D^4+\varepsilon_t^4d)}{(\underbar c_p\varepsilon^2t^{1-p})^\frac32}+\frac{(D^4+\varepsilon_t^4d)\zeta}{\varepsilon^3(\underbar c_pt^{1-p})^\frac52}\right)\sqrt{2t}\\
	&\le4\sqrt2\frac{1+\frac1{\underbar c_p}}{\underbar c_p^\frac32}\sum_{t=1}^T\left(\frac{D+\varepsilon}{\zeta^3}G^4+\frac{G+\zeta}{\varepsilon^3}D^4\right)t^{\frac32p-1}+\frac{D\zeta+G\varepsilon+2\varepsilon\zeta}{t^{1+\frac p2}}d\\
	&\le C_{p,1}\left(\frac{D+\varepsilon}{\zeta^3}G^4+\frac{G+\zeta}{\varepsilon^3}D^4\right)T^{\frac32p}+C_{p,2}(D\zeta+G\varepsilon+2\varepsilon\zeta)d
	\end{align*}
	for $C_{p,1}=4\overline c_{1-\frac32p}\sqrt2\left(1+\frac1{\underbar c_p}\right)/\underbar c_p^{3/2}$ and $C_{p,2}=4\sqrt2\left(1+\frac1{\underbar c_p}\right)\sum_{t=1}^\infty\frac1{t^{1+\frac p2}}/\underbar c_p^{3/2}$.
	Thus we have
	\begin{align*}
	\sum_{t=1}^T\ell_t(\*x_t)
	&\le\sum_{t=1}^T\tilde\ell_t(\*x_t)\\
	&\le \min_{\*x^\ast>0}C_{p,1}\left(\frac{D+\varepsilon}{\zeta^3}G^4+\frac{G+\zeta}{\varepsilon^3}D^4\right)T^{\frac32p}+C_{p,2}(D\zeta+G\varepsilon+2\varepsilon\zeta)d+\sum_{t=1}^T\tilde\ell_t(\*x^\ast)\\
	&=C_{p,1}\left(\frac{D+\varepsilon}{\zeta^3}G^4+\frac{G+\zeta}{\varepsilon^3}D^4\right)T^{\frac32p}+C_{p,2}(D\zeta+G\varepsilon+2\varepsilon\zeta)d\\
	&\quad+\min_{\*x^\ast>0}\sum_{t=1}^T\left\|\frac{\*b_t^2+\varepsilon_t^2\*1_d}{\*x^\ast}+(\*g_t^2+\zeta_t^2\*1_d)\odot\*x^\ast\right\|_1\\
	&\le C_{p,1}\left(\frac{D+\varepsilon}{\zeta^3}G^4+\frac{G+\zeta}{\varepsilon^3}D^4\right)T^{\frac32p}+C_{p,2}(D\zeta+G\varepsilon+2\varepsilon\zeta)d\\
	&\quad\min_{\*x^\ast>0}\overline c_pT^{1-p}\sum_{j=1}^d\frac{\varepsilon^2}{\*x_j^\ast}+\zeta^2\*x_j^\ast+\sum_{t=1}^T\ell_t(\*x^\ast)
	\end{align*}
	Separating again per-coordinate we have that 
	$$\sum_{t=1}^T\frac{\tilde b_t^2}{x^\ast}+\tilde g_t^2x^\ast\le\overline c_pT^{1-p}\frac{\varepsilon^2}{x^\ast}+\zeta^2x^\ast+\sum_{t=1}^T\ell_t(x^\ast)$$
	However, substituting $x^\ast=\sqrt{\frac{\tilde b_{1:T}^2}{\tilde g_{1:T}}}$ also yields
	\begin{align*}
	\min_{x^\ast>0}\sum_{t=1}^T\frac{\tilde b_t^2}{x^\ast}+\tilde g_t^2x^\ast
	&\le2\sqrt{\tilde b_{1:T}^2\tilde g_{1:T}^2}\\
	&\le2\sqrt{\overline c_p\left(\zeta^2b_{1:T}^2+\varepsilon^2g_{1:T}^2\right)}T^\frac{1-p}2+2\overline c_p\varepsilon\zeta T^{1-p}+\min_{x^\ast>0}\sum_{t=1}^T\ell_t(x^\ast)
	\end{align*}
	completing the proof.
\end{proof}

\newpage
\begin{Thm}
	Let $\Theta$ be a bounded convex subset of $\mathbb R^d$, let $\mathcal D\subset\mathbb R^{d\times d}$ be the set of positive definite diagonal matrices, and let each task $t\in[T]$ consist of a sequence of $m$ convex Lipschitz loss functions $\ell_{t,i}:\Theta\mapsto\mathbb R$.
	Suppose for each task $t$ we run the iteration in Equation~\ref{eq:matrix} setting $\*\phi=\frac1{t-1}\*\theta_{1:t-1}^\ast$ and setting $\*H=\Diag(\*\eta_t)$ via Equation~\ref{eq:percoord} for $\varepsilon=1,\zeta=\sqrt m$, and $p=\frac25$. 
	Then we achieve
	$$\TAR_T
	\le\RUB_T
	\hspace{-1mm}=\hspace{-1mm}\min_{\begin{smallmatrix}\*\phi\in\Theta\\\*H\in\mathcal D\end{smallmatrix}}\hspace{-1mm}\tilde{\mathcal O}\left(\sum_{j=1}^d\min\left\{\frac{\frac1{\*H_{jj}}+\*H_{jj}}{T^\frac25},\frac1{\sqrt[5]T}\right\}\right)\sr m+\frac1T\sum_{t=1}^T\frac{\|\*\theta_t^\ast-\*\phi\|_{\*H^{-1}}^2}2+\sum_{i=1}^m\|\*\nabla_{t,i}\|_{\*H}^2$$
\end{Thm}
\begin{proof}
	Define $\*b_t^2=\frac12(\*\theta_t^\ast-\*\phi_t)^2$ and $\*g_t^2=\*\nabla_{1:m}^2$.
	Then applying Proposition~\ref{prp:diagftl} yields
	\begin{align*}
	\RUB_TT
	&=\sum_{t=1}^T\frac{\|\*\theta_t^\ast-\*\phi_t\|_{\*\eta_t^{-1}}^2}2+\sum_{i=1}^m\|\*\nabla_{t,i}\|_{\*\eta_t}^2\\
	&=\sum_{t=1}^T\left\|\frac{(\*\theta_t^\ast-\*\phi_t)^2}{2\*\eta_t}+\*\eta_t\odot\*\nabla_{t,1:m}^2\right\|_1\\
	&\le\min_{\*\eta>0}\sum_{t=1}^T\left\|\frac{(\*\theta_t^\ast-\*\phi_t)^2}{2\*\eta}+\*\eta\odot\*\nabla_{t,1:m}^2\right\|_1\\
	&\qquad+C_p\sum_{j=1}^d\min\left\{\left(\frac{\varepsilon^2}{\*\eta_j}+\zeta^2\*\eta_j\right)T^{1-p},\sqrt{\zeta^2\*b_{j,1:T}^2+\varepsilon^2\*g_{j,1:T}^2}T^\frac{1-p}2+2\varepsilon\zeta T^{1-p}\right\}\\
	&\qquad+C_p\left(\frac{D+\varepsilon}{\zeta^3}G^4m^2+\frac{G\sqrt m+\zeta}{\varepsilon^3}D^4\right)T^{\frac32p}+C_p(D\zeta+G\sqrt m\varepsilon+\varepsilon\zeta)d\\
	&\le\min_{\begin{smallmatrix}\*\phi\in\Theta\\\*\eta>0\end{smallmatrix}}\sum_{t=1}^T\frac{\|\*\theta_t^\ast-\*\phi\|_{\*\eta^{-1}}^2}2+\sum_{i=1}^{m_t}\|\*\nabla_{t,i}\|_{\*\eta}^2+\frac{D_\infty^2}2\|\*\eta^{-1}\|_1(1+\log T)\\
	&\qquad+C_p\sum_{j=1}^d\min\left\{\left(\frac{\varepsilon^2}{\*\eta_j}+\zeta^2\*\eta_j\right)T^{1-p},\sqrt{\zeta^2\*b_{j,1:T}^2+\varepsilon^2\*g_{j,1:T}^2}T^\frac{1-p}2+2\varepsilon\zeta T^{1-p}\right\}\\
	&\qquad+C_p\left(\frac{D+\varepsilon}{\zeta^3}G^4m^2+\frac{G\sqrt m+\zeta}{\varepsilon^3}D^4\right)T^{\frac32p}+C_p(D\zeta+G\sqrt m\varepsilon+\varepsilon\zeta)d
	\end{align*}
	Substituting $\*\eta+\frac{\*1_d}{\sqrt{mT}}$ for the optimum and the values of $\varepsilon,\zeta,p$ completes the proof.
\end{proof}

\newpage

\begin{Prp}\label{prp:matrixftl}
	Let $\{\ell_t:\mathbb R_+\mapsto\mathbb R\}_{t\ge1}$ be of form $\ell_t(\*X)=\Tr(\*X^{-1}\*B_t^2)+\Tr(\*X\*G_t^2)$ for adversarially chosen $\*B_t,\*G_t$ satisfying $\|\*B_t\|_2\le\sigma_B,\|\*G_t\|_2\le\sigma_G\sqrt m$ for $m\ge1$.
	Then the $(\varepsilon,\zeta)$-FTL algorithm, which for $\varepsilon,\zeta>0$ uses the actions of FTL on the alternate function sequence $\tilde\ell_t(\*X)=\Tr((\*B^2+\varepsilon^2\*I_d)\*X^{-1})+\Tr((\*G^2+\zeta^2\*I_d)\*X)$, achieves regret 
	$$\R_T\le\frac{C_\sigma m^2}{\varepsilon^4\zeta^3}(1+\log T)+((1+\sigma_G^2)\varepsilon\sqrt m+(1+\sigma_B^2)\zeta)T$$
	for constant $C_\sigma$ depending only on $\sigma_B,\sigma_G$.
\end{Prp}
\begin{proof}
	Define $\tilde{\*B}_t^2=\*B_t^2+\varepsilon^2\*I_d,\tilde{\*G}_t^2=\*G_t^2+\zeta^2\*I_d$ and note that FTL run on modified functions $\tilde\ell_t'(\*X)=\frac12\Tr(\tilde{\*B}_t^{-2}\*X\tilde{\*G}_t^2\*X)-\log\det\*X$ has the same solution $\tilde{\*B}_{1:T}^2=\*X\tilde{\*G}_{1:T}^2\*X$.
	$$\|\nabla_{\*X}\tilde\ell_t(\*X)\|_2
		=\|\tilde{\*G}_t^2-\*X^{-1}\tilde{\*B}_t^2\*X^{-1}\|_2
		\le\|\tilde{\*G}_t\|_2^2+\|\*X^{-1}\|_2^2\|\tilde{\*B}_t\|_2^2
		\le\frac{\sigma_B^2}{\varepsilon^2}+m\sigma_G^2+\zeta^2$$
	\begin{align*}
		\|\nabla_{\*X}\tilde\ell_t'(\*X)\|_2
		=\|\tilde{\*G}_t^2\*X\tilde{\*B}_t^{-2}-\*X^{-1}\|_2
		&\le\|\tilde{\*G}_t\|_2^2\|\*X\|_2\|\tilde{\*B}_t^{-1}\|_2^2+\|\*X^{-1}\|_2\\
		&\le\frac{(m\sigma_G^2+\zeta^2)\sqrt{\sigma_B^2+\varepsilon^2}}{\varepsilon^2\zeta}+\frac{\sqrt{m\sigma_G^2+\zeta^2}}{\zeta}
	\end{align*}
	Since by Claim~\ref{clm:logdetsc} $-\log\det|\*X|$ is $\frac{\zeta^2}{\sigma_B^2+\varepsilon^2}$-strongly-convex we have by Theorem~\ref{thm:coupling} that
	$$\sum_{t=1}^T\tilde\ell_t(\*X_t)-\tilde\ell_t(\*X^\ast)
	\le\frac{C_\sigma m^2}{\varepsilon^4\zeta^3}(1+\log T)$$
	for some $C_\sigma$ depending on $\sigma_B^2,\sigma_G^2$.
	Therefore
	\begin{align*}
		\sum_{t=1}^T\ell_t(\*X)
		&\le\sum_{t=1}^T\tilde\ell_t(\*X)\\
		&\le\frac{C_\sigma m^2}{\varepsilon^4\zeta^3}(1+\log T)+\min_{\*X\succ0}\sum_{t=1}^T\tilde\ell_t(\*X)\\
		&\le\frac{C_\sigma m^2}{\varepsilon^4\zeta^3}(1+\log T)+\min_{\*X\succ0}\varepsilon^2T\Tr(\*X^{-1})+\zeta^2T\Tr(\*X)+\sum_{t=1}^T\ell_t(\*X)\\
		&\le\frac{C_\sigma m^2}{\varepsilon^4\zeta^3}(1+\log T)+(1+\sigma_G^2)\varepsilon T\sqrt m+\min_{\*X\succ0}\zeta^2T\Tr(\*X)+\sum_{t=1}^T\ell_t(\*X)\\
		&\le\frac{C_\sigma m^2}{\varepsilon^4\zeta^3}(1+\log T)+((1+\sigma_G^2)\varepsilon\sqrt m+(1+\sigma_B^2)\zeta) T+\min_{\*X\succ0}\sum_{t=1}^T\ell_t(\*X)
	\end{align*}
\end{proof}

\newpage
\begin{Thm}
	Let $\Theta$ be a bounded convex subset of $\mathbb R^d$ and let each task $t\in[T]$ consist of a sequence of $m$ convex Lipschitz loss functions $\ell_{t,i}:\Theta\mapsto\mathbb R$.
	Suppose for each task $t$ we run the iteration in Equation~\ref{eq:matrix} with $\*\phi=\frac1{t-1}\*\theta_{1:t-1}^\ast$ and $\*H$ the unique positive definite solution of $\*B_t^2=\*H\*G_t^2\*H$ for
	$$\*B_t^2=t\varepsilon^2\*I_d+\sum_{s<t}(\*\theta_s^\ast-\*\phi_s)(\*\theta_s^\ast-\*\phi_s)^T\qquad\textrm{and}\qquad\*G_t^2=t\varepsilon^2\*I_d+\sum_{s<t}\sum_{i=1}^m\*\nabla_{s,i}\*\nabla_{s,i}^T$$
	for $\varepsilon=1/\sqrt[8]T$ and $\zeta=\sqrt m/\sqrt[8]T$.
	Then we achieve
	$$\TAR_T
	\le\RUB_T
	=\tilde{\mathcal O}\left(\frac1{\sqrt[8]T}\right)\sqrt m+\min_{\begin{smallmatrix}\*\phi\in\Theta\\\*H\succ0\end{smallmatrix}}\frac{2\lambda_1^2(\*H)}{\lambda_d(\*H)}\frac{1+\log T}T+\sum_{t=1}^T\frac{\|\*\theta_t^\ast-\*\phi^\ast\|_{\*H^{-1}}^2}2+\sum_{i=1}^m\|\*\nabla_{t,i}\|_{\*H}^2$$
\end{Thm}
\begin{proof}
	Let $D$ and $G$ be the diameter of $\Theta$ and Lipschitz bound on the losses, respectively.
	Then applying Proposition~\ref{prp:matrixftl} yields
	\begin{align*}
	\RUB_TT
	&=\sum_{t=1}^T\frac{\|\*\theta_t^\ast-\*\phi_t\|_{\*H_t^{-1}}^2}2+\sum_{i=1}^m\|\*\nabla_{t,i}\|_{\*H_t}^2\\
	&=\sum_{t=1}^T\frac12\Tr\left(\*H_t^{-1}(\*\theta_t^\ast-\*\phi_t)(\*\theta_t^\ast-\*\phi_t)^T\right)+\Tr\left(\*H_t\sum_{i=1}^m\*\nabla_{t,i}\*\nabla_{t,i}^T\right)\\
	&\le\min_{\*H\succ0}\sum_{t=1}^T\frac12\Tr\left(\*H^{-1}(\*\theta_t^\ast-\*\phi_t)(\*\theta_t^\ast-\*\phi_t)^T\right)+\Tr\left(\*H\sum_{i=1}^m\*\nabla_{t,i}\*\nabla_{t,i}^T\right)\\
	&\qquad+\frac{C_\sigma m^2}{\varepsilon^4\zeta^3}(1+\log T)+((1+G^2)\varepsilon\sqrt m+(1+D^2)\zeta)T\\
	&=\min_{\*H\succ0}\sum_{t=1}^T\frac{\|\*\theta_t^\ast-\*\phi_t\|_{\*H^{-1}}^2}2+\Tr\left(\*H\sum_{i=1}^m\*\nabla_{t,i}\*\nabla_{t,i}^T\right)\\
	&\qquad+\frac{C_\sigma m^2}{\varepsilon^4\zeta^3}(1+\log T)+((1+G^2)\varepsilon\sqrt m+(1+D^2)\zeta)T\\
	&\le\min_{\begin{smallmatrix}\*\phi\in\Theta\\\*H\succ0\end{smallmatrix}}\frac{2\lambda_1^2(\*H)}{\lambda_d(\*H)}\sum_{t=1}^T\frac1t+\sum_{t=1}^T\frac{\|\*\theta_t^\ast-\*\phi^\ast\|_{\*H^{-1}}^2}2+\sum_{i=1}^m\|\*\nabla_{t,i}\|_{\*H}^2\\
	&\qquad+\frac{C_\sigma m^2}{\varepsilon^4\zeta^3}(1+\log T)+((1+G^2)\varepsilon\sqrt m+(1+D^2)\zeta)T\\
	&=\min_{\begin{smallmatrix}\*\phi\in\Theta\\\*H\succ0\end{smallmatrix}}\frac{2\lambda_1^2(\*H)}{\lambda_d(\*H)}\sum_{t=1}^T\frac1t+\sum_{t=1}^T\frac{\|\*\theta_t^\ast-\*\phi^\ast\|_{\*H^{-1}}^2}2+\sum_{i=1}^m\|\*\nabla_{t,i}\|_{\*H}^2\\
	&\qquad+\frac{C_\sigma m^2}{\varepsilon^4\zeta^3}(1+\log T)+((1+G^2)\varepsilon\sqrt m+(1+D^2)\zeta)T
	\end{align*}
\end{proof}

%% file: batch.tex

\newpage
\section{Online-to-Batch Conversion for Task-Averaged Regret}\label{sec:app:o2b}

\begin{Thm}\label{thm:o2bexp}
	Let $\Q$ be a distribution over distributions $\mathcal P$ over convex loss functions $\ell:\Theta\mapsto[0,1]$.
	A sequence of sequences of loss functions $\{\ell_{t,i}\}_{t\in[T],i\in[m]}$ is generated by drawing $m$ loss functions i.i.d. from each in a sequence of distributions $\{\mathcal P_t\}_{t\in[T]}$ themselves drawn i.i.d. from $\Q$.
	If such a sequence is given to an meta-learning algorithm with task-averaged regret bound $\TAR_T$ that has states $\{s_t\}_{t\in[T]}$ at the beginning of each task $t$ then we have w.p. $1-\delta$ for any $\theta^\ast\in\Theta$ that
	$$\E_{t\sim\mathcal U[T]}\E_{\mathcal P\sim\Q}\E_{\mathcal P^m}\E_{\ell\sim\mathcal P}\ell(\bar\theta)\le\E_{\mathcal P\sim\Q}\E_{\ell\sim\mathcal P}\ell(\theta^\ast)+\frac{\TAR_T}m+\sqrt{\frac8T\log\frac1\delta}$$
	where $\bar\theta=\frac1m\theta_{1:m}$ is generated by randomly sampling $t\in\mathcal U[T]$, running the online algorithm with state $s_t$, and averaging the actions $\{\theta_i\}_{i\in[m]}$.
	If on each task the meta-learning algorithm runs an online algorithm with regret upper bound $\U_m(s_t)$ a convex, nonnegative, and $B\sqrt m$-bounded function of the state $s_t\in\mathcal X$, where $\mathcal X$ is a convex Euclidean subset, and the total regret upper bound is $\RUB_T$, then we also have the bound
	$$\E_{\mathcal P\sim\Q}\E_{\mathcal P^m}\E_{\ell\sim\mathcal P}\ell(\bar\theta)\le\E_{\mathcal P\sim\Q}\E_{\ell\sim\mathcal P}\ell(\theta^\ast)+\frac{\RUB_T}m+B\sqrt{\frac8{mT}\log\frac1\delta}$$
	where $\bar\theta=\frac1m\theta_{1:m}$ is generated by running the online algorithm with state $\bar s=\frac1Ts_{1:T}$ and averaging the actions $\{\theta_i\}_{i\in[m]}$.
\end{Thm}
\begin{proof}
	For the second inequality, applying Proposition~\ref{prp:o2bexp}, Jensen's inequality, and Proposition~\ref{prp:o2b} yields
	\begin{align*}
		\E_{\mathcal P\sim\Q}\E_{\mathcal P^m}\E_{\ell\sim\mathcal P}\ell(\bar\theta)
		&\le\E_{\mathcal P\sim\Q}\left(\E_{\ell\sim\mathcal P}\ell(\theta^\ast)+\frac{\U_m(\bar s)}m\right)\\
		&\le\E_{\mathcal P\sim\Q}\E_{\ell\sim\mathcal P}\ell(\theta^\ast)+\frac1T\sum_{t=1}^T\E_{\mathcal P\sim\Q}\left(\frac{\U_m(s_t)}m\right)\\
		&=\E_{\mathcal P\sim\Q}\E_{\ell\sim\mathcal P}\ell(\theta^\ast)+\frac{2B}{T\sqrt m}\sum_{t=1}^T\E_{\mathcal P\sim\Q}\left(\frac{\U_m(s_t)}{2B\sqrt m}+\frac{\sqrt m}{2B}\right)-1\\
		&\le\E_{\mathcal P\sim\Q}\E_{\ell\sim\mathcal P}\ell(\theta^\ast)+\frac{\RUB_T}m+B\sqrt{\frac8{mT}\log\frac1\delta}
	\end{align*}
	The first inequality follows similarly except using $\R_m$ instead of $\U_m$, linearity of expectation instead of Jensen's inequality, 1 instead of $B$, and $\TAR_T$ instead of $\RUB_T$.
\end{proof}
Note that since regret-upper-bounds are nonnegative one can easily replace 8 by 2 in the second inequality by simply multiplying and dividing by $B\sqrt m$ in the third line of the above proof.

\begin{Clm}\label{clm:chernoff}
	In the setup of Theorem~\ref{thm:o2bexp}, let $\theta_t^\ast\in\argmin_{\theta\in\Theta}\sum_{i=1}^m\ell_{t,i}(\theta)$ and define the quantities $V_\Q^2=\argmin_{\phi\in\Theta}\E_{\mathcal P\sim\Q}\E_{\mathcal P^m}\|\theta^\ast-\phi\|_2^2$ and $D$ the $\ell_2$-radius of $\Theta$.
	Then w.p. $1-\delta$ we have $$V^2=\min_{\phi\in\Theta}\frac1T\sum_{t=1}^T\|\theta_t^\ast-\phi\|_2^2\le\mathcal O\left(V_\Q^2+\frac{D^2}T\log\frac1\delta\right)$$
\end{Clm}
\begin{proof}
	Define $\hat\phi=\argmin_{\phi\in\Theta}\sum_{t=1}^T\|\theta_t^\ast-\phi\|_2^2$ and $\phi^\ast=\argmin_{\phi\in\Theta}\E_{\mathcal P\sim\Q}\E_{\mathcal P^m}\|\theta^\ast-\phi\|_2^2$.
	Then by a multiplicative Chernoff's inequality w.p. at least  $1-\delta$ we have
	\begin{align*}
		TV^2
		=\sum_{t=1}^T\|\theta_t^\ast-\hat\phi\|_2^2
		\le\sum_{t=1}^T\|\theta_t^\ast-\phi^\ast\|_2^2
		&\le\left(1+\max\left\{1,\frac{3D^2}{V_\Q^2T}\log\frac1\delta\right\}\right)T\E_{\mathcal P\sim\Q}\E_{\mathcal P^m}\|\theta^\ast-\phi^\ast\|_2^2\\
		&\le2TV_Q^2+3D^2\log\frac1\delta
	\end{align*}
\end{proof}

\begin{Cor}\label{cor:batch}
	Under the assumptions of Theorems~\ref{thm:similarity} and~\ref{thm:statistical}, if the loss functions are Lipschitz and we use Algorithm~\ref{alg:general} with $\eta_t$ also learned, using $\varepsilon$-EWOO as in Theorem~\ref{thm:similarity} for $\varepsilon=1/\sqrt[4]{mT}+1/\sqrt m$, and set the initialization using $\phi_{t+1}=\frac1t\sum_{s\le t}\theta_s^\ast$, then w.p. $1-\delta$  we have
	$$\E_{\mathcal P\sim\Q}\E_{\mathcal P^m}\ell_{\mathcal P}(\bar\theta)\le\E_{\mathcal P\sim\Q}\ell_{\mathcal P}(\theta^\ast)+\tilde{\mathcal O}\left(\frac{V_\Q}{\sqrt m}+\min\left\{\frac{\frac1{\sqrt T}+\frac1{\sqrt m}}{V_\Q m},\frac1{\sqrt[4]{m^3T}}+\frac1m\right\}+\sqrt{\frac1T\log\frac1\delta}\right)$$
	where $V_\Q^2=\min_{\phi\in\Theta}\E_{\mathcal P\sim\Q}\E_{\mathcal P^m}\|\theta^\ast-\phi\|_2^2$.
\end{Cor}
\begin{proof}
	Substitute Corollary~\ref{cor:epsewoo} into Theorem~\ref{thm:o2bexp} using the fact the the regret-upper-bounds are $\mathcal O(\frac{\sqrt m}{\varepsilon})$-bounded.
	Conclude by applying Claim~\ref{clm:chernoff}.
\end{proof}

\begin{Thm}\label{thm:o2bwhp}
	Let $\Q$ be a distribution over distributions $\mathcal P$ over convex losses $\ell:\Theta\mapsto[0,1]$ such that the functions $\ell(\theta)-\ell(\theta^\ast)$ are $\rho$-self-bounded for some $\rho>0$ and $\theta^\ast\in\argmin_{\theta\in\Theta}\E_{\ell\sim\mathcal P}(\theta)$.
	A sequence of sequences of loss functions $\{\ell_{t,i}\}_{t\in[T],i\in[m]}$ is generated by drawing $m$ loss functions i.i.d. from each in a sequence of distributions $\{\mathcal P_t\}_{t\in[T]}$ themselves drawn i.i.d. from $\Q$.
	If such a sequence is given to an meta-learning algorithm with task-averaged regret bound $\TAR_T$ that has states $\{s_t\}_{t\in[T]}$ at the beginning of each task $t$ then we have w.p. $1-\delta$ for any $\theta^\ast\in\Theta$ that
	\begin{align*}
	\E_{t\sim\mathcal U[T]}\E_{\mathcal P\sim\Q}\E_{\ell\sim\mathcal P}\ell(\bar\theta)
	\le\E_{\mathcal P\sim\Q}\E_{\ell\sim\mathcal P}\ell(\theta^\ast)+\frac{\TAR_T}m
	&+\sqrt{\frac{2\rho}{m}\left(\frac{\TAR_T}m+\sqrt{\frac8T\log\frac2\delta}\right)\log\frac2\delta}\\
	&+\sqrt{\frac8T\log\frac2\delta}+\frac{3\rho+2}m\log\frac2\delta
	\end{align*}
	where $\bar\theta=\frac1m\theta_{1:m}$ is generated by randomly sampling $t\in\mathcal U[T]$, running the online algorithm with state $s_t$, and averaging the actions $\{\theta_i\}_{i\in[m]}$.
	If on each task the meta-learning algorithm runs an online algorithm with regret upper bound $\U_m(s_t)$ a convex, nonnegative, and $B\sqrt m$-bounded function of the state $s_t\in\mathcal X$, where $\mathcal X$ is a convex Euclidean subset, and the total regret upper bound is $\RUB_T$, then we also have the bound
	\begin{align*}
	\E_{\mathcal P\sim\Q}\E_{\ell\sim\mathcal P}\ell(\bar\theta)\le\E_{\mathcal P\sim\Q}\E_{\ell\sim\mathcal P}\ell(\theta^\ast)+\frac{\RUB_T}m
	&+\sqrt{\frac{2\rho}{m}\left(\frac{\RUB_T}m+B\sqrt{\frac8{mT}\log\frac2\delta}\right)\log\frac2\delta}\\
	&+B\sqrt{\frac8{mT}\log\frac2\delta}+\frac{3\rho+2}m\log\frac2\delta
	\end{align*}
	where $\bar\theta=\frac1m\theta_{1:m}$ is generated by running the online algorithm with state $\bar s=\frac1Ts_{1:T}$ and averaging the actions $\{\theta_i\}_{i\in[m]}$.
\end{Thm}
\begin{proof}
	By Corollary~\ref{cor:o2bsb} and Jensen's inequality we have w.p. $1-\frac\delta2$ that
	\begin{align*}
		\E_{\mathcal P\sim\Q}\E_{\ell\sim\mathcal P}\ell(\bar\theta)
		&\le\E_{\mathcal P\sim\Q}\left(\E_{\ell\sim\mathcal P}\ell(\theta^\ast)+\frac{\U_m(\bar s)}m+\frac1m\sqrt{2\rho\U_m(\bar s)\log\frac1\delta}+\frac{3\rho+2}m\log\frac1\delta\right)\\
		&\le\E_{\mathcal P\sim\Q}\E_{\ell\sim\mathcal P}\ell(\theta^\ast)+\frac1T\sum_{t=1}^T\E_{\mathcal P\sim\Q}\left(\frac{\U_m(s_t)}m\right)\\
		&\qquad+\sqrt{\frac{2\rho}{mT}\sum_{t=1}^T\E_{\mathcal P\sim\Q}\left(\frac{\U_m(s_t)}m\right)\log\frac2\delta}+\frac{3\rho+2}m\log\frac2\delta
	\end{align*}
	As in the proof of Theorem~\ref{thm:o2bexp}, by Proposition~\ref{prp:o2b} we further have w.p. $1-\frac\delta2$ that
	$$\frac1T\sum_{t=1}^T\E_{\mathcal P\sim\Q}\left(\frac{\U_m(s_t)}m\right)\le\frac{\RUB_T}m+B\sqrt{\frac8{mT}\log\frac2\delta}$$
	Substituting the second inequality into the first yields the second bound.
	The first bound follows similarly except using $\R_m$ instead of $\U_m$, linearity of expectation instead of Jensen's inequality, 1 instead of $B$, and $\TAR_T$ instead of $\RUB_T$.
\end{proof}

\newpage
\begin{Thm}\label{thm:o2b}
	Let $\Q$ be a distribution over distributions $\mathcal P$ over convex loss functions $\ell:\Theta\mapsto[0,1]$.
	A sequence of sequences of loss functions $\{\ell_{t,i}\}_{t\in[T],i\in[m]}$ is generated by drawing $m$ loss functions i.i.d. from each in a sequence of distributions $\{\mathcal P_t\}_{t\in[T]}$ themselves drawn i.i.d. from $\Q$.
	If such a sequence is given to an meta-learning algorithm that on each task runs an online algorithm with regret upper bound $\U_m(s_t)$ a nonnegative, $B\sqrt m$-bounded, $G$-Lipschitz w.r.t. $\|\cdot\|$, and $\alpha$-strongly-convex w.r.t. $\|\cdot\|$ function of the state $s_t\in\mathcal X$ at the beginning of each task $t$, where $\mathcal X$ is a convex Euclidean subset, and the total regret upper bound is $\RUB_T$, then we have w.p. $1-\delta$ for any $\theta^\ast\in\Theta$ that
	$$\E_{\mathcal P\sim\Q}\E_{\mathcal P^m}\E_{\ell\sim\mathcal P}\ell(\bar\theta)\le\E_{\mathcal P\sim\Q}\E_{\ell\sim\mathcal P}\ell(\theta^\ast)+\mathcal L_T$$
	for
	$$\mathcal L_T=\frac{\U^\ast+\RUB_T}m+\frac{4G}T\sqrt{\frac{\RUB_T}{\alpha m}\log\frac{8\log T}\delta}+\frac{\max\{16G^2,6\alpha B\sqrt m\}}{\alpha mT}\log\frac{8\log T}\delta$$
	where $\U^\ast=\E_{\mathcal P\sim\Q}\U_m(s^\ast)$ for any valid $s^\ast$ and $\bar\theta=\frac1m\theta_{1:m}$ is generated by running the online algorithm with state $\bar s=\frac1Ts_{1:T}$ and averaging the actions $\{\theta_i\}_{i\in[m]}$.
	If we further assume that the functions $\ell(\theta)-\ell(\theta^\ast)$ are $\rho$-self-bounded for some $\rho>0$ and $\theta^\ast\in\argmin_{\theta\in\Theta}\E_{\ell\sim\mathcal P}(\theta)$ for all $\mathcal P$ in the support of $\Q$ then we also have the bound
	$$\E_{\mathcal P\sim\Q}\E_{\ell\sim\mathcal P}\ell(\bar\theta)\le\E_{\mathcal P\sim\Q}\E_{\ell\sim\mathcal P}\ell(\theta^\ast)+\mathcal L_T+\sqrt{\frac{2\rho\mathcal L_T}{m}\log\frac2\delta}+\frac{3\rho+2}m\log\frac2\delta$$
\end{Thm}
\begin{proof}
	Applying Proposition~\ref{prp:o2bexp} and Theorem~\ref{thm:o2bsc} we have w.p. $1-\frac\delta2$ that
	\begin{align*}
		\E_{\mathcal P\sim\Q}\E_{\mathcal P^m}\E_{\ell\sim\mathcal P}\ell(\bar\theta)
		&\le\E_{\mathcal P\sim\Q}\left(\E_{\ell\sim\mathcal P}\ell(\theta^\ast)+\frac{\U_m(\bar s)}m\right)\\
		&\le\E_{\mathcal P\sim\Q}\E_{\ell\sim\mathcal P}\ell(\theta^\ast)+\frac1m\E_{\mathcal P\sim\Q}\U_m(s^\ast)+\frac{\RUB_T}m\\
		&\qquad+\frac{4G}T\sqrt{\frac{\RUB_T}{\alpha m}\log\frac{8\log T}\delta}+\frac{\max\{16G^2,6\alpha B\sqrt m\}}{\alpha mT}\log\frac{8\log T}\delta\\
		&\le\E_{\mathcal P\sim\Q}\E_{\ell\sim\mathcal P}\ell(\theta^\ast)+\mathcal L_T
	\end{align*}
	This yields the first bound since.
	The second bound follows similarly except for the application of Corollary~\ref{cor:o2bsb} in the second step w.p. $1-\frac\delta2$.
\end{proof}

\begin{Cor}\label{cor:scbatch}
	Under the assumptions of Theorem~\ref{thm:statistical} and boundedness of $\Theta$, if the loss functions are $G$-Lipschitz and we use Algorithm~\ref{alg:general} running OGD with fixed $\eta=\frac{V_\Q+1/\sqrt T}{G\sqrt m}$, where we have $V_\Q^2=\min_{\phi\in\Theta}\E_{\mathcal P\sim\Q}\E_{\mathcal P^m}\|\theta^\ast-\phi\|_2^2$, and set the initialization using $\phi_{t+1}=\frac1t\theta_{1:t}^\ast$, then w.p. $1-\delta$  we have
	$$\E_{\mathcal P\sim\Q}\E_{\mathcal P^m}\ell_{\mathcal P}(\bar\theta)\le\E_{\mathcal P\sim\Q}\ell_{\mathcal P}(\theta^\ast)+\tilde{\mathcal O}\left(\frac{V_\Q}{\sqrt m}+\left(\frac1T+\frac1{\sqrt{mT}}\right)\max\left\{\log\frac1\delta,\sqrt{\log\frac1\delta}\right\}\right)$$
\end{Cor}
\begin{proof}
	Apply Theorem~\ref{thm:app:general} with $V_\Phi=V_\Q+1/\sqrt T$, $\Rsim=0$ (because the learning rate is fixed), and $\Rinit=\tilde{\mathcal O}\left(\hat V\sqrt m+1/\sqrt T\right)$ (for $\hat V^2=\min_{\phi\in\Theta}\frac1T\sum_{t=1}^T\|\theta_t^\ast-\phi\|_2^2$). 
	Substitute the result into Theorem~\ref{thm:o2b} using the fact that $\U_m$ is $\mathcal O\left(\left(\frac1\varepsilon+\varepsilon\right)\sqrt m\right)$-bounded, $\mathcal O\left(\frac{\sqrt m}\varepsilon\right)$-Lipschitz, and $\Omega\left(\frac{\sqrt m}\varepsilon\right)$-strongly-convex.
	Conclude by applying Claim~\ref{clm:chernoff} to bound $\hat V$.
\end{proof}

%% file: similarity.tex

\clearpage
\section{Adapting to Task-Similarity under Parameter Growth }\label{sec:app:growth}

In this appendix we cast the problem of adaptively learning the task-similarity in the framework of \citet{khodak:19}.
We do this specifically to show that our basic results extend to approximate meta-updates under quadratic growth.
We first provide a generalized version of their Ephemeral method in Algorithm~\ref{alg:fmrl}.
We then state the relevant approximation assumptions and proceed to prove guarantees on the average regret-upper-bound for the case of a fixed task-similarity in Theorem~\ref{thm:fixed} and for adaptively learning it in Theorem~\ref{thm:learning}.
Then the quadratic-growth results of \citet{khodak:19}, specifically Propositions~B.1, B.2, and B.3, can be applied directly to show average regret-upper-bound guarantees of the same order as those in the main paper but with additional $o_m(1)$ terms inside the parentheses.
Note that our results, especially in the batch-within-online setting, will in general be stronger because we do not incur the $\Delta_{\max}$-error term that is needed to account for the doubling trick in \citet{khodak:19}.

\begin{algorithm}
	\DontPrintSemicolon
	\vspace{5pt}
	\KwData{
		\begin{itemize}
			\item action space $\Theta\subset\mathbb R^d$ with norm $\|\cdot\|$
			\item function $R:\Theta\mapsto\mathbb R$ that is 1-strongly-convex w.r.t. $\|\cdot\|$ and its corresponding Bregman divergence $\Breg_R$
			\item class of within-task algorithms $\{\TASK_{\eta,\phi}:\eta>0,\phi\in\Theta\}$
			\item meta-update algorithms $\DYN$ and $\SIM$
			\item sequence of loss functions $\{\ell_{t,i}:\Theta\mapsto\mathbb R\}_{t\in[T],i\in[m_t]}$ where $\ell_{t,i}$ is $G_{t,i}$-Lipschitz w.r.t. $\|\cdot\|$
		\end{itemize}
	}
	\vspace{5pt}
	\For{$t\in[T]$}{
		\vspace{10pt}
		\tcp{set learning rate and initialization using meta-update algorithms}
		$D_t=\SIM(\{\ell_{s,i}\}_{s<t,i\in[m_s]})$\\
		$G_t\gets\sqrt{\frac1{m_t}\sum_{i=1}^{m_t}G_{t,i}^2}$\\
		$\eta_t\gets\frac{D_t}{G_t\sqrt{m_t}}$\\
		$\phi_t=\DYN(\{\ell_{s,i}\}_{s<t,i\in[m_s]})$\\
		
		\vspace{10pt}
		\tcp{run within-task algorithm}
		\For{$i\in[m_t]$}{
			$\theta_{t,i}\gets\TASK_{\eta_t,\phi_t}(\ell_{t,1},\sd,\ell_{t,i-1})$\\
			suffer loss $\ell_{t,i}(\theta_{t,i})$\\
		}
		
		\vspace{10pt}
		\tcp{compute meta-update vector $\theta_t$ according to \Eph variant}
		\uCase{\em Optimal Action}{
			$\theta_t\gets\argmin_{\theta\in\Theta}\sum_{i=1}^{m_t}\ell_{t,i}(\theta)$
		}
		\Case{\em Last Iterate}{
			$\theta_t\gets\TASK_{\eta_t,\phi_t}(\ell_{t,1},\sd,\ell_{t,m_t})$
		}
		\Case{\em Average Iterate}{
			$\theta_t\gets\frac1{m_t}\sum_{i=1}^{m_t}\theta_{t,i}$
		}
	}
	\caption{\label{alg:fmrl}
		Follow-the-Meta-Regularized-Leader (\Eph) meta-algorithm for meta-learning \citep{khodak:19}.
		For the {\em Optimal Action} variant we assume $\argmin_{\theta\in\Theta}L(\theta)$ returns $\theta$ minimizing $\Breg_R(\theta|\phi_R)$ over the set of all minimizers of $L$ over $\Theta$, where $\phi_R$ is some appropriate element of $\Phi$ such as the origin in Euclidean space or the uniform distribution over the simplex.
	}
\end{algorithm}

%
%

\newpage
\begin{Asu}\label{asu:general}
	Assume the data given to Algorithm~\ref{alg:fmrl} and define the following quantities:
	\begin{itemize}
		\item convenience coefficients $\sigma_t=G_t\sqrt{m_t}$
		\item sequence of update parameters $\{\hat\theta_t\in\Theta\}_{t\in[T]}$ with average update $\hat\phi=\frac1{\sigma_{1:T}}\sum_{t=1}^T\sigma_t\hat\theta$
		\item a sequence of reference parameters $\{\theta_t'\in\Theta\}_{t\in[T]}$ with average reference parameter $\phi'=\frac1{\sigma_{1:T}}\sum_{t=1}^T\sigma_t\theta_t'$
		\item a sequence $\{\theta_t^\ast\in\Theta\}_{t\in[T]}$ of optimal parameters in hindsight
		\item we will say we are in the ``Exact" case if $\hat\theta_t=\theta_t'=\theta_t^\ast~\forall~t$ and the ``Approx" case otherwise
		\item $\kappa\ge1,\Delta_t^\ast\ge0$ s.t. $\sum_{t=1}^T\alpha_t\Breg_R(\theta_t^\ast||\phi_t)\le\Delta_{1:T}^\ast+\kappa\sum_{t=1}^T\alpha_t\Breg_R(\hat\theta_t||\phi_t)$ for some $\alpha_t\ge0$
		\item $\nu\ge1,\Delta'\ge0$ s.t. $\sum_{t=1}^T\sigma_t\Breg_R(\hat\theta_t||\hat\phi)\le\Delta'+\nu\sum_{t=1}^T\sigma_t\Breg_R(\theta_t'||\phi')$
		\item average deviation $V^2=\frac1{\sigma_{1:T}}\sum_{t=1}^T\sigma_t\Breg_R(\theta_t'||\phi')$ of the reference parameters
		\item action diameter $D^2=\max\{{D^\ast}^2,\max_{\theta\in\Theta}\Breg_R(\theta||\phi_1)\}$ in the Exact case or $\max_{\theta,\phi\in\Theta}\Breg_R(\theta||\phi)$ in the Approx case
		\item constant $C'$ s.t. $\|\theta\|\le C'\|\theta\|_2~\forall~\theta\in\Theta$ and $\ell_2$-diameter $D'=\max_{\theta,\phi}\|\theta-\phi\|_2$ of $\Theta$
		\item effective action space $\hat\Theta=\Conv(\{\hat\theta_t\}_{t\in[T]})$ if $\DYN$ is FTL or $\Theta$ if $\DYN$ is AOGD
		\item upper bound $G'$ on the Lipschitz constants of the functions $\{\Breg_R(\hat\theta_t||\cdot)\}_{t\in[T]}$ over $\hat\Theta$
		\item we will say we are in the ``Nice" case if $\Breg_R(\theta||\cdot)$ is 1-strongly-convex and $\beta$-strongly-smooth w.r.t. $\|\cdot\|~\forall~\theta\in\Theta$
		\item in the general case $\DYN$ is FTL; in the Nice case $\DYN$ may instead be AOGD
		\item convenience indicator $\iota=1_{\DYN=\FTL}$
		\item $\TASK_{\eta,\phi}=\FTRL_{\eta,\phi}^{(R)}$ or $\OMD_{\eta,\phi}^{(R)}$
	\end{itemize}
	We make the following assumptions:
	\begin{itemize}
		\item the loss functions $\ell_{t,i}$ are convex $\forall~t,i$
		\item at $t=1$ the update algorithm $\DYN$ plays $\phi_1\in\Theta$ satisfying $\max_{\theta\in\Theta}\Breg_R(\theta||\phi_1)<\infty$
		\item in the Approx case $R$ is $\beta$-strongly-smooth for some $\beta\ge1$
	\end{itemize}
\end{Asu}

\newpage
\subsection{Average Regret using Fixed Task Similarity}\label{subsec:fixed}
The following theorem does not appear in the main paper but is used in discussion.
It shows guarantees for the case when the task-similarity is known in advance and so $\SIM$ always returns a constant.

\begin{Thm}\label{thm:fixed}
	Make Assumption~\ref{asu:general} and suppose $\SIM$ always plays $D_t=\varepsilon$.
	Then Algorithm~\ref{alg:fmrl} has a regret upper-bound of
	$$\RUB_M\le\frac1T\left(\left(\frac{\kappa D^2}\varepsilon+\varepsilon\right)\iota\sigma_1+\frac{\kappa C}\varepsilon\sum_{t=1}^T\frac{\sigma_t^2}{\sigma_{1:t}}+\left(\frac{\kappa\nu V^2}\varepsilon+\varepsilon\right)\sigma_{1:T}+\frac{\Delta_{1:T}^\ast}\varepsilon+\frac{\kappa\Delta'}\varepsilon\right)$$
	for $C=\frac{{G'}^2}2$ in the Nice case or otherwise $C=2C'D'G'$.
\end{Thm}
\begin{proof}
	Let $\{\tilde\phi_t\}_{t\in[T]}$ be a ``cheating" sequences such that $\tilde\phi_t=\phi_t$ on all $t$ except if $\SIM$ is FTL and $t=1$, in which case $\tilde\phi_1=\hat\theta_1$.
	Note that by this definition all upper bounds of $\Breg_R(\hat\theta_t||\phi_t)$ also upper bound $\Breg_R(\hat\theta_t||\tilde\phi_t)$.
	We then use the fact that the actions of FTL at $t>1$ do not depend on the action at time $t=1$ to get
	\begin{align*}
		&\RUB_MT\\
		&=\sum_{t=1}^T\frac{\Breg_R(\theta_t^\ast||\phi_t)}{\eta_t}+\eta_tG_t^2m_t\\
		&=\frac{\Delta_{1:T}^\ast}\varepsilon+\sum_{t=1}^T\left(\frac{\kappa\Breg_R(\hat\theta_t||\phi_t)}\varepsilon+\varepsilon\right)\sigma_t\qquad\textrm{(substitute $\eta_t=\frac{D_t}{G_t\sqrt{m_t}}$ and $D_t=\varepsilon$)}\\
		&\le\left(\frac{\kappa D^2}\varepsilon+\varepsilon\right)\iota\sigma_1+\frac{\Delta_{1:T}^\ast}\varepsilon+\sum_{t=1}^T\left(\frac{\kappa\Breg_R(\hat\theta_t||\tilde\phi_t)}\varepsilon+\varepsilon\right)\sigma_t\qquad\textrm{(substitute cheating sequence)}\\
		&=\left(\frac{\kappa D^2}\varepsilon+\varepsilon\right)\iota\sigma_1+\frac{\Delta_{1:T}^\ast}\varepsilon+\frac\kappa\varepsilon\sum_{t=1}^T\left(\Breg_R(\hat\theta_t||\tilde\phi_t)-\Breg_R(\hat\theta_t||\hat\phi)\right)\sigma_t+\sum_{t=1}^T\left(\frac{\kappa\Breg_R(\hat\theta_t||\hat\phi)}\varepsilon+\varepsilon\right)\sigma_t\\
		&\le\left(\frac{\kappa D^2}\varepsilon+\varepsilon\right)\iota\sigma_1+\frac{\Delta_{1:T}^\ast}\varepsilon+\frac{\kappa C}\varepsilon\sum_{t=1}^T\frac{\sigma_t^2}{\sigma_{1:t}}+\frac{\kappa\Delta'}\varepsilon\\
		&\qquad+\sum_{t=1}^T\left(\frac{\kappa\nu\Breg_R(\theta_t'||\phi')}\varepsilon+\varepsilon\right)\sigma_t\qquad\textrm{(Thm.~\ref{thm:ftlaogd} and Prop.~\ref{prp:bregman})}\\
		&=\left(\frac{\kappa D^2}\varepsilon+\varepsilon\right)\iota\sigma_1+\frac{\Delta_{1:T}^\ast}\varepsilon+\frac{\kappa C}\varepsilon\sum_{t=1}^T\frac{\sigma_t^2}{\sigma_{1:t}}+\frac{\kappa\Delta'}\varepsilon+\left(\frac{\kappa\nu V^2}\varepsilon+\varepsilon\right)\sigma_{1:T}
	\end{align*}
\end{proof}


\newpage
\subsection{Average Regret when Learning Task Similarity}\label{subsec:learning}
\begin{Thm}\label{thm:learning}
	Make Assumption~\ref{asu:general} and let $\SIM$ be an algorithm running on the sequence of pairs $\{\Breg_R(\hat\theta_t||\phi_t),\sigma_t\}_{t\in[T]}$ and at each time $t$ having as output the action of an OCO algorithm on the function sequence $\{\ell_t(x)=(\Breg_R(\hat\theta_t||\phi_t)/x+x)\sigma_t\}_{t\in[T]}$.
	Let $\R_T$ be the associated regret of this algorithm and suppose it has a parameter $\varepsilon>0$ controlling the minimum action taken.
	For simplicity assume that at time $t=1$ $\SIM$ plays $D_1$ s.t. $\frac12(\max_{\theta\in\Theta}\sqrt{\Breg_R(\theta||\phi_1)}+\varepsilon)\le D_1\le\max_{\theta\in\Theta}\sqrt{\Breg_R(\theta||\phi_1)}+\varepsilon$ .
	Then Algorithm~\ref{alg:fmrl} has a regret upper-bound of
	$$\RUB_M\le\frac1T\left((2\kappa D+\varepsilon)\iota\sigma_1+\kappa\R_T+\frac{\kappa C}V\sum_{t=1}^T\frac{\sigma_t^2}{\sigma_{1:t}}+\kappa(\nu+1)V\sigma_{1:T}+\frac{\Delta_{1:T}^\ast}\varepsilon+\frac{\kappa\Delta'}V\right)$$
	for $C=\frac{{G'}^2}2$ in the Nice case or otherwise $C=2C'D'G'$.
\end{Thm}
\begin{proof}
	Let $\{\tilde\phi_t\}_{t\in[T]}$ be ``cheating" sequence such that $\tilde\phi_t=\phi_t$ on all $t$ except if $\SIM$ is FTL and $t=1$, in which case $\tilde\phi_1=\hat\theta_1$.
	Note that by this definition all upper bounds of $\Breg_R(\hat\theta_t||\phi_t)$ also upper bound $\Breg_R(\hat\theta_t||\tilde\phi_t)$.
	We then have
	\begin{align*}
		\RUB_MT
		&=\sum_{t=1}^T\frac{\Breg_R(\theta_t^\ast||\phi_t)}{\eta_t}+\eta_tG_t^2m_t\\
		&=\frac{\Delta_{1:T}^\ast}\varepsilon+\sum_{t=1}^T\left(\frac{\kappa\Breg_R(\hat\theta_t||\phi_t)}{D_t}+D_t\right)\sigma_t\qquad\textrm{(substitute $\eta_t=\frac{D_t}{G_t\sqrt{m_t}}$ and $D_t\ge\varepsilon$)}\\
		&\le\left(\frac{\kappa\Breg_R(\hat\theta_t||\phi_t)}{D_1}+D_1\right)\iota\sigma_1+\frac{\Delta_{1:T}^\ast}\varepsilon\\
		&\qquad+\sum_{t=1}^T\left(\frac{\kappa\Breg_R(\hat\theta_t||\tilde\phi_t)}{D_t}+D_t\right)\sigma_t\qquad\textrm{(substitute cheating sequences)}\\
		&\le((\kappa+1)D+\varepsilon)\iota\sigma_1+\frac{\Delta_{1:T}^\ast}\varepsilon+\kappa\R_T+\kappa\sum_{t=1}^T\left(\frac{\Breg_R(\hat\theta_t||\tilde\phi_t)}V+V\right)\sigma_t\\
		&\le(2\kappa D+\varepsilon)\iota\sigma_1+\frac{\Delta_{1:T}^\ast}\varepsilon+\kappa\R_T+\frac{\kappa C}V\sum_{t=1}^T\frac{\sigma_t^2}{\sigma_{1:t}}\\
		&\qquad+\kappa\sum_{t=1}^T\left(\frac{\Breg_R(\hat\theta_t||\hat\phi)}V+V\right)\sigma_t\qquad\textrm{(Thm.~\ref{thm:ftlaogd} and Prop.~\ref{prp:bregman})}\\
		&\le(2\kappa D+\varepsilon)\iota\sigma_1+\frac{\Delta_{1:T}^\ast}\varepsilon+\kappa\R_T+\frac{\kappa C}V\sum_{t=1}^T\frac{\sigma_t^2}{\sigma_{1:t}}+\frac{\kappa\Delta'}V+\kappa\sum_{t=1}^T\left(\frac{\nu\Breg_R(\theta_t'||\phi')}V+V\right)\sigma_t\\
		&\le(2\kappa D+\varepsilon)\iota\sigma_1+\frac{\Delta_{1:T}^\ast}\varepsilon+\kappa\R_T+\frac{\kappa C}V\sum_{t=1}^T\frac{\sigma_t^2}{\sigma_{1:t}}+\frac{\kappa\Delta'}V+\kappa(\nu+1)V\sigma_{1:T}
	\end{align*}
\end{proof}

\newpage
\subsection{Statistical Task-Similarity under Quadratic Growth}
In this section we relate our task-similarity measure to that of \citet{denevi:19} under $\alpha$-QG.

\begin{Prp}
	For some distribution $\mathcal P\sim\Q$ over losses $\ell:\Theta\mapsto\mathbb R_+$ let $\theta_{\mathcal P}^\ast=\argmin_{\theta\in\Theta}\ell_{\mathcal P}(\theta)$ and $\hat\theta_m=\argmin_{\theta\in\Theta}\sum_{i=1}^m\ell_i(\theta)$ for $m$ i.i.d. samples $\ell_i\sim\mathcal P$.
	Define task-similarity measures $V^2=\min_{\phi\in\Theta}\E_{\mathcal P\sim\Q}\|\theta_{\mathcal P}^\ast-\phi\|_2^2$ and $\hat V_m^2=\min_{\phi\in\Theta}\E_{\mathcal P\sim\Q}\E_{\mathcal P^m}\|\hat\theta_m-\phi\|_2^2$.
	If both $\ell_{\mathcal P}$ and $\frac1m\sum_{i=1}^m\ell_i$ are $G$-Lipschitz and $\alpha$-QG a.s. then we have
	$$V^2\le2\hat V_m^2+\frac{16G^2}{\alpha^2m}\qquad\textrm{and}\qquad\hat V_m^2\le2V^2+\frac{16G^2}{\alpha^2m}$$
\end{Prp}
\begin{proof}
	Following the argument of \citet[Theorem~2]{shalev-shwartz:10} but applying $\alpha$-QG instead of strong-convexity in Equation~8, which holds by definition of $\alpha$-QG, we obtain
	$$\E_{\mathcal P^m}(\ell_{\mathcal P}(\hat\theta_m)-\ell_{\mathcal P}(\theta_{\mathcal P}^\ast))\le\frac{4G^2}{\alpha m}$$
	Then for $\phi^\ast=\argmin\E_{\mathcal P\sim\Q}\|\theta_{\mathcal P}^\ast-\phi\|_2^2$ and $\phi_m^\ast=\argmin_{\phi\in\Theta}\E_{\mathcal P\sim\Q}\E_{\mathcal P^m}\|\hat\theta_m-\phi\|_2^2$ we have by these definitions, the triangle inequality, Jensen's inequality, $\alpha$-QG of $\frac1m\sum_{i=1}^m\ell_i$, and the above inequality we have
	\begin{align*}
	\hat V_m^2
	=\E_{\mathcal P\sim\Q}\E_{\mathcal P^m}\|\hat\theta_m-\phi_m^\ast\|_2^2
	&\le\E_{\mathcal P\sim\Q}\E_{\mathcal P^m}\|\hat\theta_m-\phi^\ast\|_2^2\\
	&\le2\E_{\mathcal P\sim\Q}\E_{\mathcal P^m}\left(\|\hat\theta_m-\theta_{\mathcal P}^\ast\|_2^2+\|\theta_{\mathcal P}^\ast-\phi^\ast\|_2^2\right)\\
	&\le\frac4\alpha\E_{\mathcal P\sim\Q}\E_{\mathcal P^m}(\ell_{\mathcal P}(\hat\theta_m)-\ell_{\mathcal P}(\theta_{\mathcal P}^\ast))+2V^2\\
	&\le\frac{16G^2}{\alpha^2m}+2V^2
	\end{align*}
	Similarly,
	\begin{align*}
	V^2
	=\E_{\mathcal P\sim\Q}\E_{\mathcal P^m}\|\theta_{\mathcal P}-\phi^\ast\|_2^2
	&\le\E_{\mathcal P\sim\Q}\E_{\mathcal P^m}\|\theta_{\mathcal P}-\phi_m^\ast\|_2^2\\
	&\le2\E_{\mathcal P\sim\Q}\E_{\mathcal P^m}\left(\|\hat\theta_m-\theta_{\mathcal P}^\ast\|_2^2+\|\hat\theta_m-\phi_m^\ast\|_2^2\right)\\
	&\le\frac4\alpha\E_{\mathcal P\sim\Q}\E_{\mathcal P^m}(\ell_{\mathcal P}(\hat\theta_m)-\ell_{\mathcal P}(\theta_{\mathcal P}^\ast))+2V^2\\
	&\le\frac{16G^2}{\alpha^2m}+2V^2
	\end{align*}
\end{proof}

%% file: details.tex

\newpage
\section{Experimental Details}\label{sec:details}

Code is available at \url{https://github.com/mkhodak/ARUBA}.

\subsection{Reptile}

For our Reptile experiments we use the code and default settings provided by \citet{nichol:18}, except we tune the learning rate, which for ARUBA corresponds to $\varepsilon/\zeta$, and the coefficient $c$ in ARUBA++.
In addition to the the parameters listed in the above tables, we set $\zeta=p=1.0$ for all experiments.
All evaluations are averages of three runs.

\begin{table}[H]
	\centering
	\scriptsize
	\begin{tabular}{ccccccccc}  
		\toprule
		{\bf Omniglot} & \multicolumn{4}{c}{\bf 1-shot} & \multicolumn{4}{c}{\bf 5-shot} \\
		\cmidrule(lr){2-5} \cmidrule(lr){6-9}
		{\bf 5-way}  & \multicolumn{2}{c}{evaluation setting} & \multicolumn{2}{c}{hyperparameters} & \multicolumn{2}{c}{evaluation setting} & \multicolumn{2}{c}{hyperparameters}\\
		\cmidrule(lr){2-3} \cmidrule(lr){4-5} \cmidrule(lr){6-7} \cmidrule(lr){8-9}
		& regular & transductive & $\eta=\frac\varepsilon\zeta$ & $c$ & regular & transductive & $\eta=\frac\varepsilon\zeta$ & $c$ \\
		\midrule
		\hspace{-2mm}MAML (1) \cite{finn:17} && $98.3\pm0.5$ &&&& $99.2\pm0.2$ \\
		\hspace{-2mm}Reptile \citep{nichol:18} & $95.39\pm0.09$ & $97.68\pm0.04$ & $1e-3$ && $98.90\pm0.10$ & $99.48\pm0.06$ & $1e-3$ \\
		\hspace{-2mm}ARUBA & $94.57\pm1.04$ & $97.44\pm0.32$ & $1e-1$ && $98.64\pm0.04$ & $99.29\pm0.07$ & $1e-2$ \\
		\hspace{-2mm}ARUBA++ & $94.80\pm1.10$ & $97.58\pm0.13$ & $1e-1$ & $10^3$ & $98.93\pm0.13$ & $99.46\pm0.02$ & $1e-2$ & $10^3$ \\
		\midrule
		\hspace{-2mm}MAML (2) && $98.7\pm0.4$ &&&& $99.9\pm0.1$ \\
		\hspace{-2mm}Meta-SGD \cite{li:17} && $99.53\pm0.26$ &&&& $99.93\pm0.09$ \\
		\bottomrule
	\end{tabular}
	\vspace{1mm}
\end{table}
\begin{figure}[h]
	\centering
	\includegraphics[width=0.777\linewidth]{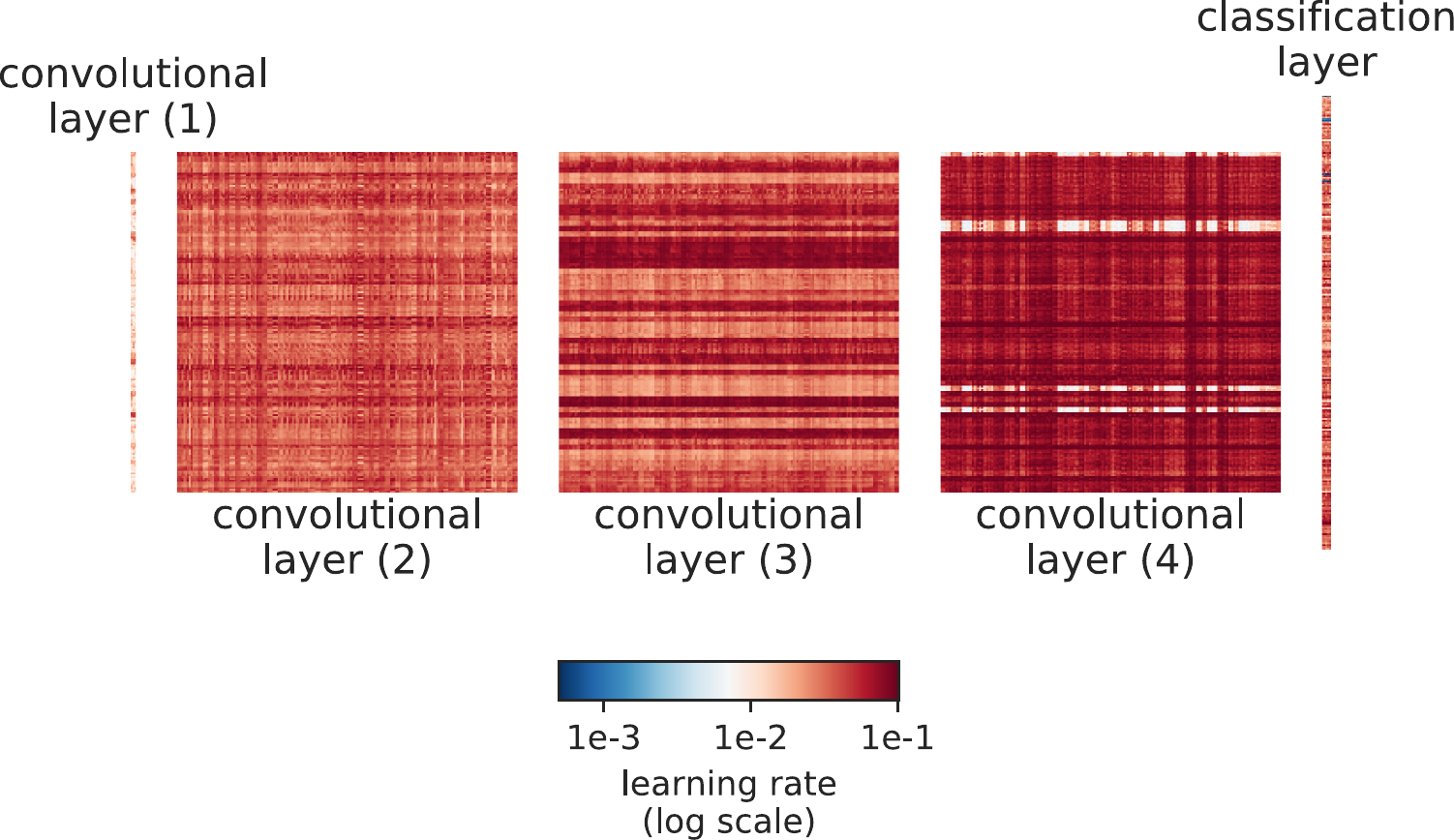}
	\includegraphics[width=0.777\linewidth]{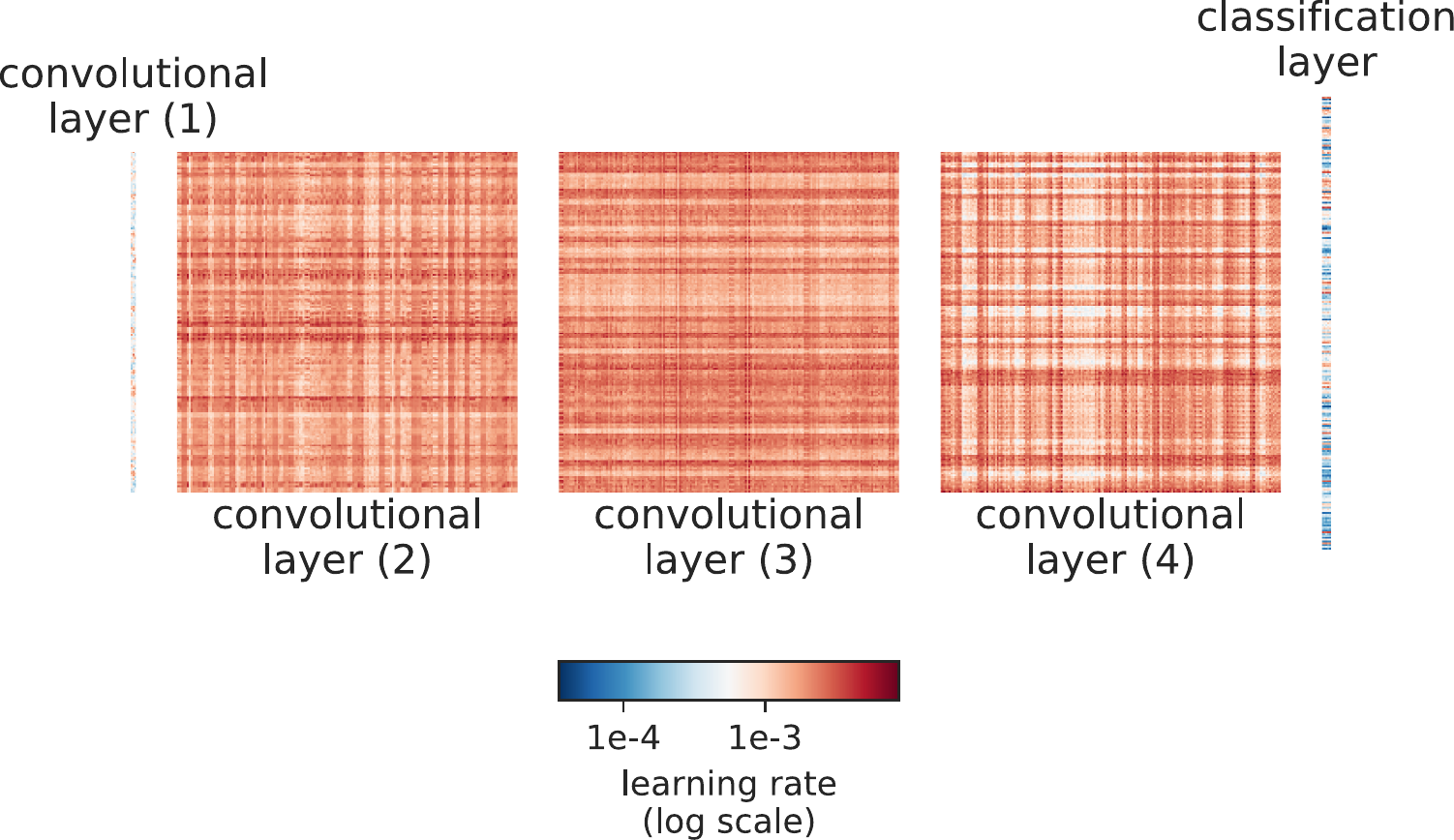}
	\caption{
		Final learning rate $\eta_T$ across the layers of a convolutional network trained on 1-shot 5-way Omniglot (top) and 5-shot 5-way Omniglot (bottom) using Algorithm~\ref{alg:aruba} applied to Reptile.
	}
\end{figure}

\begin{table}[H]
	\centering
	\scriptsize
	\begin{tabular}{ccccccccc}  
		\toprule
		{\bf Omniglot} & \multicolumn{4}{c}{\bf 1-shot} & \multicolumn{4}{c}{\bf 5-shot} \\
		\cmidrule(lr){2-5} \cmidrule(lr){6-9}
		{\bf 20-way}  & \multicolumn{2}{c}{evaluation setting} & \multicolumn{2}{c}{hyperparameters} & \multicolumn{2}{c}{evaluation setting} & \multicolumn{2}{c}{hyperparameters}\\
		\cmidrule(lr){2-3} \cmidrule(lr){4-5} \cmidrule(lr){6-7} \cmidrule(lr){8-9}
		& regular & transductive & $\eta=\frac\varepsilon\zeta$ & $c$ & regular & transductive & $\eta=\frac\varepsilon\zeta$ & $c$ \\
		\midrule
		\hspace{-2mm}MAML (1) \cite{finn:17} && $95.8\pm0.3$ &&&& $98.9\pm0.2$ \\
		\hspace{-2mm}Reptile \citep{nichol:18} & $88.14\pm0.15$ & $89.43\pm0.14$ & $5e-4$ && $96.65\pm0.33$ & $97.12\pm0.32$ & $5e-4$ \\
		\hspace{-2mm}ARUBA & $85.61\pm0.25$ & $86.67\pm0.17$ & $5e-3$ && $96.02\pm0.12$ & $96.61\pm0.13$ & $5e-3$ \\
		\hspace{-2mm}ARUBA++ & $88.38\pm0.24$ & $89.66\pm0.3$ & $5e-3$ & $10^3$ & $96.99\pm0.35$ & $97.49\pm0.28$ & $5e-3$ & $10$ \\
		\midrule
		\hspace{-2mm}MAML (2) && $95.8\pm0.3$ &&&& $98.9\pm0.2$ \\
		\hspace{-2mm}Meta-SGD \cite{li:17} && $95.93\pm0.38$ &&&& $98.97\pm0.19$ \\
		\bottomrule
	\end{tabular}
	\vspace{1mm}
\end{table}
\begin{figure}[h]
	\centering
	\includegraphics[width=0.99\linewidth]{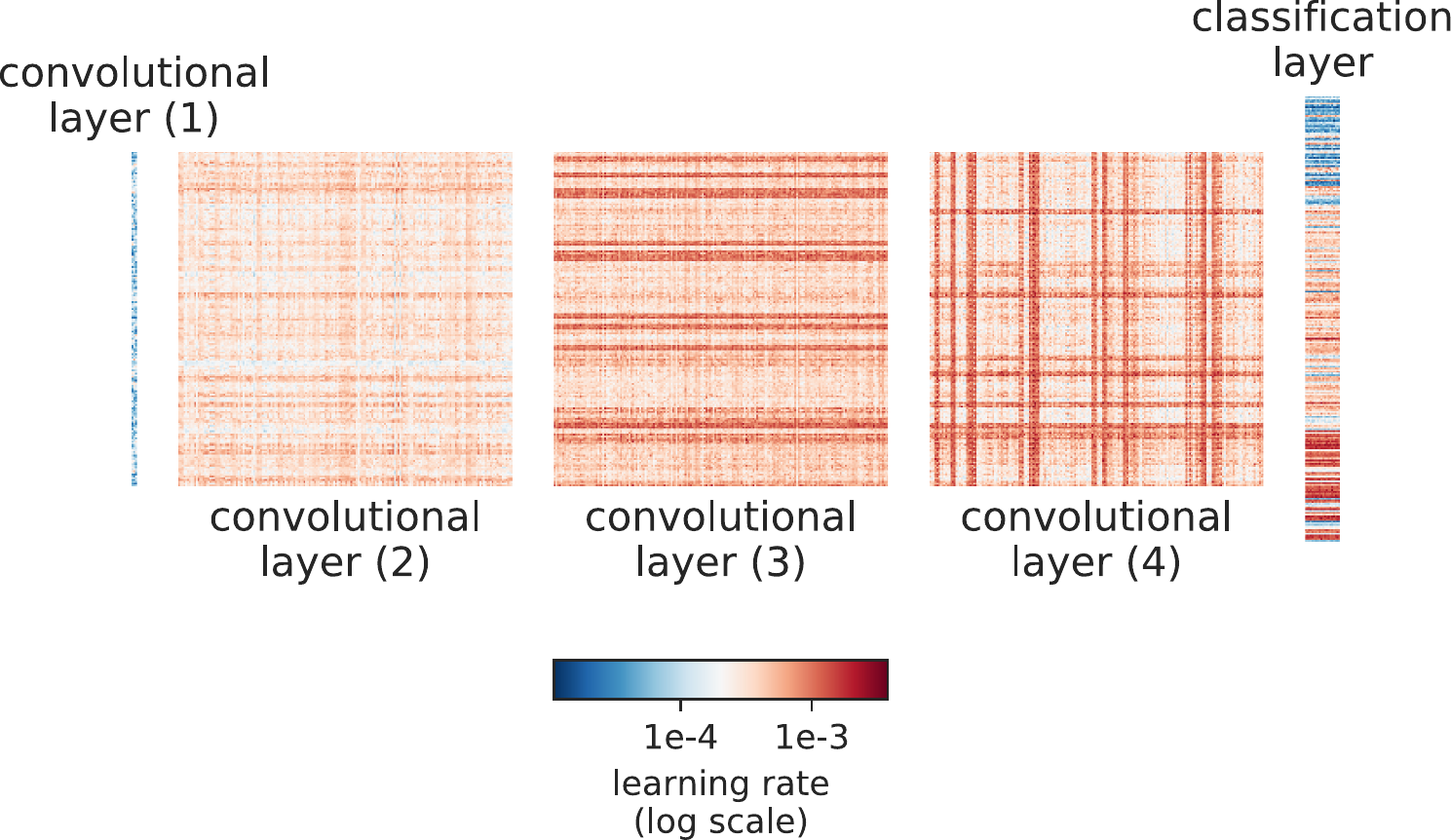}
	\includegraphics[width=0.99\linewidth]{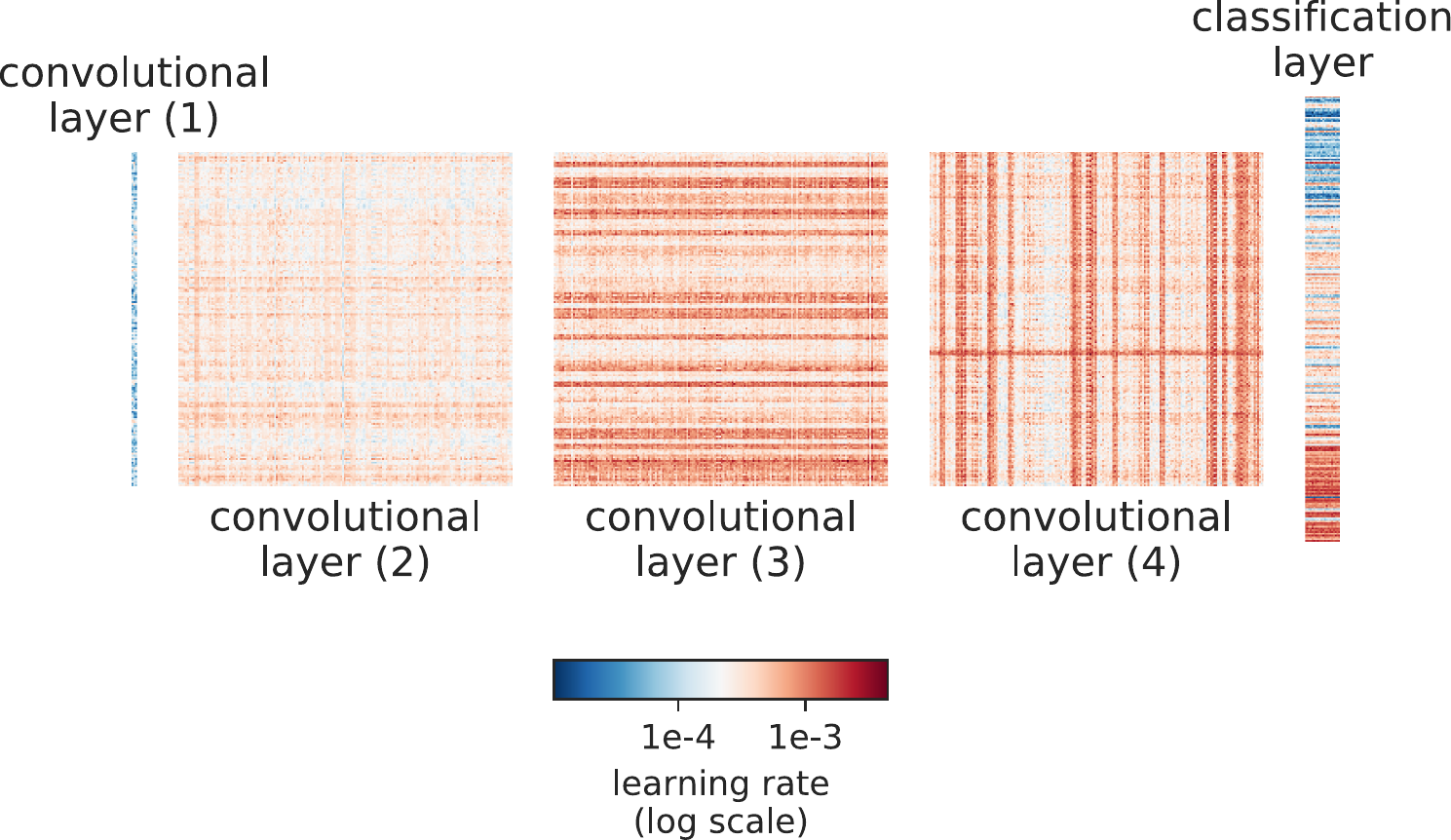}
	\caption{
		Final learning rate $\eta_T$ across the layers of a convolutional network trained on 1-shot 20-way Omniglot (top) and 5-shot 20-way Omniglot (bottom) using Algorithm~\ref{alg:aruba} applied to Reptile.\vspace{12mm}
	}
\end{figure}

\begin{table}[H]
	\centering
	\scriptsize
	\begin{tabular}{ccccccccc}  
		\toprule
		\hspace{-2mm}{\bf Mini-ImageNet} & \multicolumn{4}{c}{\bf 1-shot} & \multicolumn{4}{c}{\bf 5-shot} \\
		\cmidrule(lr){2-5} \cmidrule(lr){6-9}
		{\bf 5-way}  & \multicolumn{2}{c}{evaluation setting} & \multicolumn{2}{c}{hyperparameters} & \multicolumn{2}{c}{evaluation setting} & \multicolumn{2}{c}{hyperparameters}\\
		\cmidrule(lr){2-3} \cmidrule(lr){4-5} \cmidrule(lr){6-7} \cmidrule(lr){8-9}
		& regular & transductive & $\eta=\frac\varepsilon\zeta$ & $c$ & regular & transductive & $\eta=\frac\varepsilon\zeta$ & $c$ \\
		\midrule
		\hspace{-2mm}MAML (1) \cite{finn:17} && $48.07\pm1.75$ &&&& $63.15\pm0.91$ \\
		\hspace{-2mm}Reptile \citep{nichol:18} & $47.07\pm0.26$ & $49.97\pm0.32$ & $1e-3$ && $62.74\pm0.37$ & $65.99\pm0.58$ & $1e-3$ \\
		\hspace{-2mm}ARUBA & $47.01\pm0.37$ & $50.73\pm0.32$ & $5e-3$ && $62.35\pm0.25$ & $65.69\pm0.61$ & $5e-3$ \\
		\hspace{-2mm}ARUBA++ & $47.25\pm0.61$ & $50.35\pm0.74$ & $5e-3$ & $10$ & $62.69\pm0.57$ & $65.89\pm0.34$ & $5e-3$ & $10^{-1}$ \\
		\midrule
		\hspace{-2mm}MAML (2) && $48.70\pm1.84$ &&&& $63.11\pm0.92$ \\
		\hspace{-2mm}Meta-SGD \cite{li:17} && $50.47\pm1.87$ &&&& $64.03\pm0.94$ \\
		\bottomrule
	\end{tabular}
	\vspace{1mm}
\end{table}
\begin{figure}[H]
	\centering
	\includegraphics[width=0.85\linewidth]{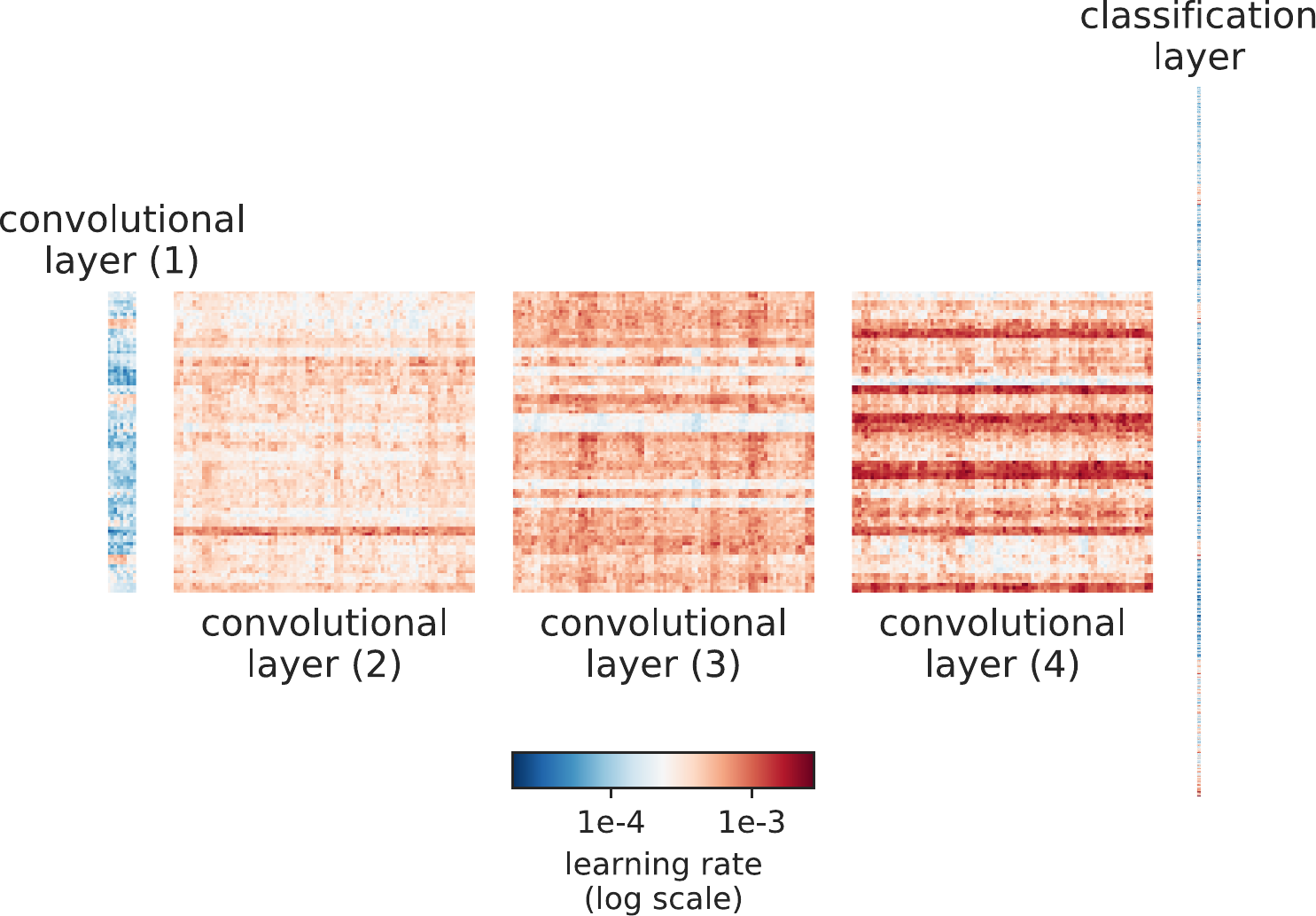}
	\includegraphics[width=0.85\linewidth]{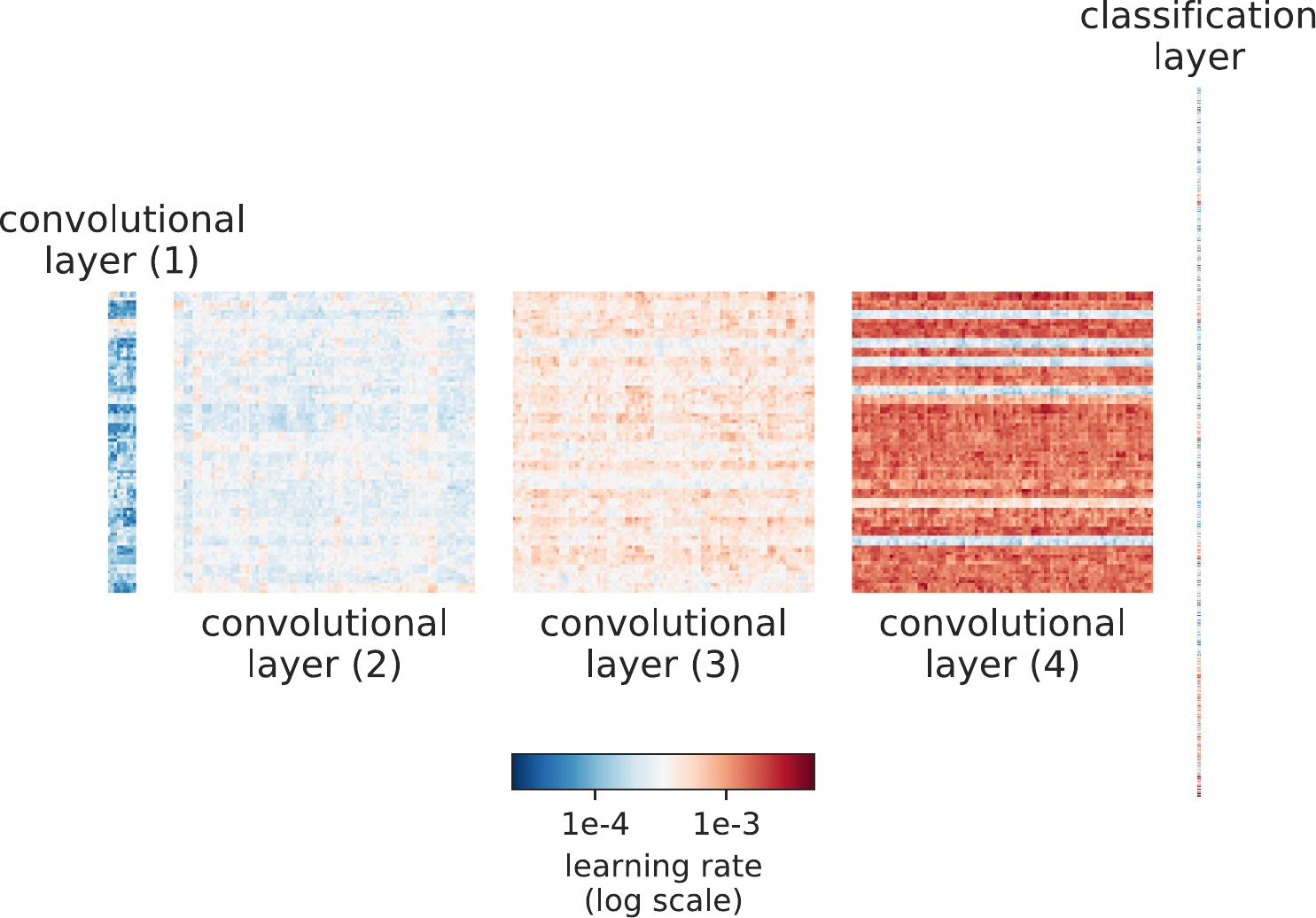}
	\caption{
		Final learning rate $\eta_T$ across the layers of a convolutional network trained on 1-shot 5-way Mini-ImageNet (top) and 5-shot 5-way Mini-ImageNet (bottom) using Algorithm~\ref{alg:aruba} applied to Reptile.
	}
\end{figure}

\newpage
\subsection{FedAvg}

%

For FedAvg we train a 2-layer stacked LSTM model with 256 hidden units, 8-dimensional trained character embeddings, with a maximum input string size of 80 characters;
these settings are used to match those of \citet{mcmahan:17}.
Similarly, we take their approach of only removing those actors from the Shakespeare dataset with fewer than two lines and split each user temporally into train/test sets with a training fraction of 0.8.
Unlike \citet{mcmahan:17}, we also split the users into meta-training and meta-testing sets, also with a fraction of 0.8, in order to evaluate meta-test performance.
We run both algorithms for 500 rounds with a batch of 10 users per round and a within-task batch-size of 10, as in \citet{caldas:18}.
For unmodified FedAvg we found that an initial learning rate of $\eta=1.0$ worked well -- this is similar to those reported in \citet{mcmahan:17} and \citet{caldas:18} -- and for the tuned variant we found that a multiplicative decay of $0.99$.
At meta-test-time we tuned the refinement learning rate over $\{10^{-3},10^{-2},10^{-1}\}$.
For ARUBA and its isotropic variant we set $\varepsilon=\zeta=0.05$ and $p=1.0$, so that $\eta=\varepsilon/\zeta=1.0$ in our setting as well.

\begin{figure}[H]
	\centering
	\includegraphics[width=0.95\linewidth]{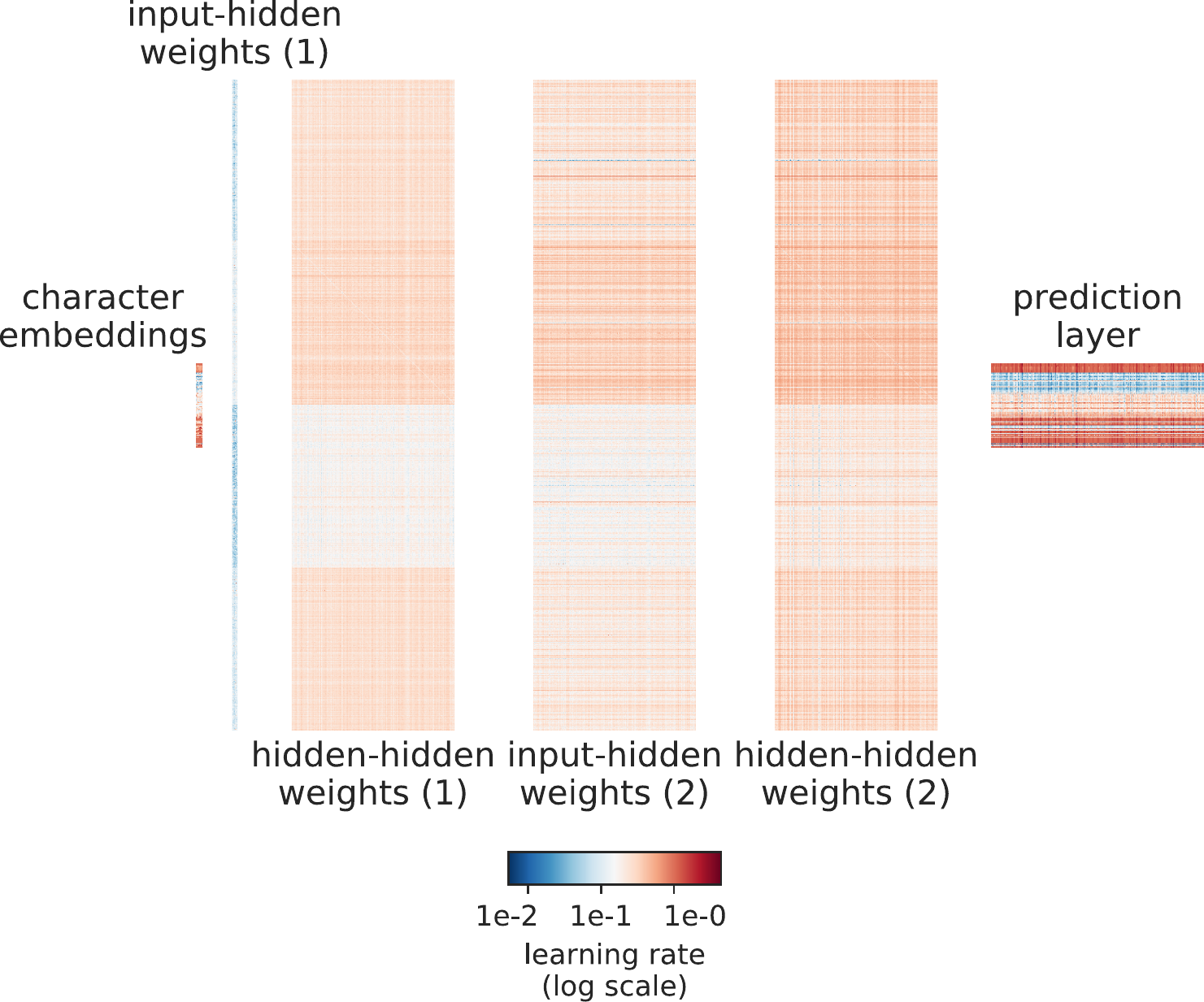}
	\caption{
		Final learning rate $\eta_T$ across the layers of an LSTM trained for next-character prediction on the Shakespeare dataset using Algorithm~\ref{alg:aruba} applied to FedAvg.
	}
\end{figure}
